\documentclass{article}
\usepackage{selectp}
\usepackage[accepted]{aistats2020}

\usepackage{natbib}
\usepackage[utf8]{inputenc} 
\usepackage[T1]{fontenc}    
\usepackage{url}            
\usepackage{booktabs}       
\usepackage{amsfonts}       
\usepackage{nicefrac}     
\usepackage{breqn}
\usepackage{microtype}      
\usepackage{dsfont}
\usepackage{graphicx}

\usepackage{mathenv}
\usepackage{algorithmicx}
\usepackage{algorithm,algpseudocode}

\usepackage{amsmath}
\usepackage{amssymb}
\usepackage{amsthm}
\usepackage{amsfonts} 
\usepackage{bm}
\usepackage{setspace}
\usepackage{xspace}
\usepackage{xcolor} 
\usepackage{wrapfig}
\usepackage{array} 
\usepackage{multirow}
\usepackage{enumitem}
\setitemize{noitemsep,topsep=0pt,parsep=0pt,partopsep=0pt}
\usepackage{mdframed}

\usepackage{thm-restate}
\usepackage{caption}
\usepackage{subcaption}

\usepackage{selectp}

\usepackage[colorinlistoftodos, textwidth=20mm, disable]{todonotes}

\definecolor{blued}{RGB}{70,197,221}
\definecolor{pearOne}{HTML}{2C3E50}
\definecolor{pearTwo}{HTML}{A9CF54}
\definecolor{pearTwoT}{HTML}{C2895B}
\definecolor{pearThree}{HTML}{E74C3C}
\colorlet{titleTh}{pearOne}
\colorlet{bull}{pearTwo}
\definecolor{pearcomp}{HTML}{B97E29}
\definecolor{pearFour}{HTML}{588F27}
\definecolor{pearFith}{HTML}{ECF0F1}
\definecolor{pearDark}{HTML}{2980B9}
\definecolor{pearDarker}{HTML}{1D2DEC}
\usepackage{hyperref}       
\hypersetup{
	colorlinks,
	citecolor=pearDark,
	linkcolor=pearThree,
	breaklinks=true,
	urlcolor=pearDarker}

\definecolor{citrine}{rgb}{0.89, 0.82, 0.04}

\DeclareMathOperator*{\argmax}{arg\,max}
\DeclareMathOperator*{\argmin}{arg\,min}

\newcommand{\ceil}[1]{\left\lceil#1\right\rceil}
\newcommand{\floor}[1]{\left\lfloor#1\right\rfloor}

\newcommand{\R}{\mathbb{R}}

\newcommand{\NN}{{\mathbb N}}

\newcommand{\E}{\mathbb{E}}
\newcommand{\EE}[1]{\mathbb{E}\left[#1\right]}

\newcommand{\PP}[1]{\mathbb{P}\left[#1\right]}

\newcommand{\EEempty}{\mathbb{E}}
\newcommand{\PPempty}{\mathbb{P}}
\newcommand{\pa}[1]{\left(#1\right)}

\newcommand{\cO}{\mathcal{O}}
\newcommand{\tcO}{\widetilde{\cO}}

\renewcommand{\epsilon}{\varepsilon}
\renewcommand{\hat}{\widehat}

\renewcommand{\bar}{\overline}

\newcommand{\nothere}[1]{}

\usepackage{xspace}

\newcommand{\UCB}{\texttt{UCB}\xspace}

\newcommand{\klucb}{\texttt{KL-UCB}\xspace}

\newcommand{\EXP}{\texttt{Exp3}\xspace}

\newcommand{\hmu}{\hat{\mu}}
\newcommand{\bmu}{\bar{\mu}}

\newcommand{\hepsilon}{\hat{\epsilon}}

\newcommand{\rounds}{{\textcolor[rgb]{0.3,0.0,0.8}{T}}}

\newcommand{\regret}{R_\rounds}

\newcommand{\noise}{\epsilon}

\newcommand{\CommaBin}{\mathbin{\raisebox{0.5ex}{,}}}

\newcommand{\currentTime}{t}

\newcommand{\subgaussian}{\sigma}
\newcommand{\arm}{i}

\newcommand{\policy}{\pi}
\newcommand{\reward}{\mu}

\newcommand{\possibleArms}{\mathcal{K}}
\newcommand{\arms}{\mathcal{K}}
\newcommand{\window}{h}

\newcommand{\historyt}{\mathcal{H}_\currentTime}
\newcommand{\history}{\mathbf{H}}

\newcommand{\underpullSet}{\textsc{up}}
\newcommand{\overpullSet}{\textsc{op}}
\newcommand{\HPevent}{\xi^\alpha_t}
\newcommand{\HPt}[1]{\xi^\alpha_{{#1}}}
\newcommand{\HPeff}{\xi^\alpha_{t,\, m}}
\newcommand{\HPtwo}{\xi^\alpha_{t,\, 2}}

\newcommand{\myAlgorithm}{\normalfont \texttt{FEWA}\xspace}

\newcommand{\FEWA}{\normalfont \texttt{FEWA}\xspace}

\newcommand{\piF}{\policy_{\rm F}}

\newcommand{\SWA}{\normalfont \texttt{SWA}\xspace}

\newcommand{\DUCB}{\normalfont \texttt{D-UCB}\xspace}
\newcommand{\SWUCB}{\normalfont \texttt{SW-UCB}\xspace}
\newcommand{\UCBone}{\normalfont \texttt{UCB1}\xspace}
\newcommand{\EFF}{\normalfont {\texttt{EFF\_UPDATE}}\xspace}

\newcommand{\FILTER}{\normalfont \texttt{FILTER}\xspace}
\newcommand{\UPDATE}{\normalfont \texttt{UPDATE}\xspace}
\newcommand{\EFFFEWA}{\normalfont \texttt{EFF-FEWA}\xspace}
\newcommand{\RAWUCB}{\normalfont \texttt{RAW-UCB}\xspace}
\newcommand{\RUCB}{\normalfont \texttt{RAW-UCB}\xspace}

\newcommand{\piR}{\policy_{\rm R}}

\newcommand{\EFFRAW}{\normalfont \texttt{EFF-RAW-UCB}\xspace}

\newcommand{\piER}{\normalfont \policy_{\rm ER}\xspace}
\newcommand{\piEF}{\normalfont \policy_{\rm EF}\xspace}

\newcommand{\EXPS}{\normalfont \texttt{Exp3.S}\xspace}

\newcommand{\ADSWITCH}{\normalfont \texttt{AdSwitch}\xspace}
\newcommand{\GLRKLUCB}{\normalfont \texttt{GLR-klUCB}\xspace}
\newcommand{\GLRUCB}{\normalfont \texttt{GLR-UCB}\xspace}
\newcommand{\MUCB}{\normalfont \texttt{M-UCB}\xspace}
\newcommand{\CUSUMUCB}{\normalfont \texttt{CUSUM-UCB}\xspace}
\newcommand{\ADAILTCBplus}{\normalfont \texttt{ADA-ILTCB+}\xspace}

\newcommand{\Nit}{N_{i,\,t}}
\newcommand{\Nitmone}{N_{i,\,t-1}}
\newcommand{\Nisttmone}{N_{\ist,\,t-1}}
\newcommand{\NiT}{N_{i,\,T}}

\newcommand{\hiT}{h_{i,\,T}}

\newcommand{\PPv}{\mathbb{P}}
\newcommand{\ie}{\emph{i.e.}\xspace} 

\newcommand{\Him}{H_{i,\,m}}
\newcommand{\Hitwo}{H_{i,\,2}}

\newcommand{\hmueff}{\hmu_{i,\texttt{\footnotesize{eff}}}^{h_j}}
\newcommand{\hmuiteff}{\hmu_{i_t,\texttt{\footnotesize{eff}}}^{h_j}}
\newcommand{\bmueff}{\bmu_{i,\texttt{\footnotesize{eff}}}^{h_j}}
\newcommand{\bmuiteff}{\bmu_{i_t,\texttt{\footnotesize{eff}}}^{h_j}}
\newcommand{\neff}{n_i^{h_j}}
\newcommand{\peff}{p_i^{h_j}}
\newcommand{\ist}{i^\star_t}
\newcommand{\tteff}{\scriptsize{\normalfont{\texttt{eff}}}}

\newtheorem{corollary}{Corollary}
\newtheorem{lemma}{Lemma}
\newtheorem{proposition}{Proposition}

\newtheorem{assumption}{Assumption}

\newtheorem{remark}{Remark}
\usepackage{authblk}

\AfterEndEnvironment{assumption}{\noindent\ignorespaces}
\AfterEndEnvironment{proposition}{\noindent\ignorespaces}
\AfterEndEnvironment{theorem}{\noindent\ignorespaces}
\AfterEndEnvironment{restatable}{\noindent\ignorespaces}
\AfterEndEnvironment{lemma}{\noindent\ignorespaces}
\AfterEndEnvironment{remark}{\noindent\ignorespaces}

\begin{document}
\twocolumn[
\aistatstitle{A single algorithm for both restless and rested rotting bandits}

\aistatsauthor{ Julien Seznec \And Pierre Menard \And  Alessandro Lazaric \And Michal Valko }
\aistatsaddress{ Lelivrescolaire.fr\\ SCOOL, Inria
Lille\And  SCOOL, Inria Lille \And FAIR Paris \And DeepMind Paris} 
]

\begin{abstract}
    In many application domains (e.g., recommender systems, intelligent tutoring systems), the rewards associated to the actions tend to decrease over time. This decay is either caused by the actions executed in the past (e.g., a user may get bored when songs of the same genre are recommended over and over) or by an external factor (e.g., content becomes outdated). These two situations can be modeled as specific instances of the rested and restless bandit settings, where arms are \textit{rotting} (i.e., their value decrease over time). These problems were thought to be significantly different, since \citet{levine2017rotting} showed that state-of-the-art algorithms for restless bandit perform poorly in the rested rotting setting. In this paper, we introduce a novel algorithm, Rotting Adaptive Window UCB (\RAWUCB), that achieves near-optimal regret in both rotting rested and restless bandit, without any prior knowledge of the setting (rested or restless) and the type of non-stationarity (e.g., piece-wise constant, bounded variation). This is in striking contrast with previous negative results showing that no algorithm can achieve similar results as soon as rewards are allowed to increase. We confirm our theoretical findings on a number of synthetic and dataset-based experiments.

\end{abstract}

 \section{Introduction}
When we design sequential learner,
we would like them to be as adaptive to 
environment as possible. 
This becomes a challenge
when the environment only provides limited feedback, 
as in the \emph{bandit} setting~ \citep{lai1985asymptotically,lattimore2020banditbook}, where the learner receives only the feedback associated to the action it executed.
Since the early stages of the research in bandits
\citep{thompson1933likelihood,whittle1980multi},
one of the most desirable properties for a learners would be to adapt to actions whose \textit{value changes over time}
\citep{whittle1988restless}, as it happens in non-stationary environments. 
In fact, from applications 
in medical trials (where the patient can become more
resistant to antibiotics) to a 
 modern applications
in recommender systems \citep{chapelle2011empirical,traca2015regulating},
assuming that the environment is \textit{stationary is very limiting}. 

However, modeling and managing non-stationary environments 
is obviously way more difficult \citep{lattimore2020banditbook}.
That is why \citet{auer2002nonstochastic}
went as far as to consider the worst-case scenario, referred to as the \textit{adversarial bandit} setting,
where the learner should try to shield from the worst possible variation 
in rewards. Nonetheless, real-world environments are
rarely adversarial and algorithms
for adversarial bandits turn out to be too conservative
for practical use. On the one hand, in order to manage such general family of environments, the performance of a learner is compared to the best \textit{fixed} action in \textit{hindsight}. This is arguably a weaker objective w.r.t.\ competing against the optimal strategy, as it is the case in stationary bandits. On the other hand, state-of-the-art adversarial algorithms \citep{audibert2009minimax}, which are proved to recover near-optimal regret rates on stationary problems, still under-perform in practice against optimal stationary algorithm~\citep{zimmert2019optimal}. In order to address these issues, prior work identified specific types of non-stationary environments, for which specifically designed algorithms can be used. 


There are two main classes of non-stationary environments, depending on whether the change of rewards is triggered by the actions of the learner, the \textit{rested bandits}, or it happens over time independently from the learner, the \textit{restless bandits}. 
In this paper, we consider the specific case where the changes in the rewards are arbitrary \textit{non-increasing} functions of time and/or number of pulls (in contrast with typical restless bandit models, where the evolution of rewards was regulated by Markov chain processes).
For instance, \citet{warlop2018fighting} model boredom effects in recommender systems as a rested bandit problem, but need to resort to a more general reinforcement learning framework to address the fact that rewards are  decreasing while an action is repeatedly selected but may increase back if \textit{enough time} has passed since the last time is chosen. \citet{Immorlica2018} and \citet{pikeburke2019recovering} have recently modeled these recharging effects as a bandits problem. In the restless setting, \citet{louedec2016algorithme} models obsolescence of appearing arms (e.g. piece of news) with a known exponential rate. \cite{komiyama2014time-decaying} study a parametric decay in restless bandits where rewards are linear combination of known decaying function. 
In the following, we briefly review the most relevant results available for restless bandit (where no rotting assumption has been studied before) and the rested rotting bandit settings.


\paragraph{Restless stochastic bandits}
\cite{garivier2011upper-confidence-bound} study the restless bandits case, where rewards are piece-wise stationary. If the number of stationary pieces $\Upsilon_T$ at the horizon $T$ is known, the optimal strategy is included in a set of $T^{\Upsilon_T}$ switching experts. Hence one can use \EXPS, an adversarial algorithm designed for this specific set of experts \citep{auer2002nonstochastic}. Moreover, \cite{garivier2011upper-confidence-bound} show that two upper-confidence bound index algorithms with passive forgetting parameters, \SWUCB and \DUCB, are also able to reach nearly-minimax performance when they know in advance $\Upsilon_T$ and $T$.  Recent research \citep{cao2019nearly, liu2018change-detection, besson2019generalized} has focused on integrating change-detection algorithms with standard bandit learners (\emph{e.g.} \UCB) to actively forget past rewards whenever a significant variation in the reward distribution is detected. Among them, we mention \GLRKLUCB \citep{besson2019generalized} which uses a parameter-free change-point detector.   These algorithms actively explore sub-optimal actions to track potential increase in their value. Yet, their analysis assume that change-points are always big enough to be detectable with high-probability. 
\cite{auer2019adaptively} introduce \ADSWITCH, a filtering algorithm with a planned active exploration scheme for sub-optimal actions. \ADSWITCH achieves the minimax rate while being agnostic to $\Upsilon_T$ without any extra assumption. 

\cite{besbes2014stochastic} introduced a restless bandits framework where the environment has a variation budget of $V_T$ to change the rewards' values. In this setup, the  best arm can change at each round and thus the optimal strategy is not necessary included in a "small" set of switching experts. Yet, they show that the best strategy with $\cO\pa{T^{1/3}}$ switches suffers low regret compared to the optimal strategy. Hence, \EXPS matches the minimax rate $\cO\pa{T^{2/3}}$ with the knowledge of $V_T$.  \citet{cheung2019new} and \citet{russac2019weighted} extended \SWUCB and \DUCB to show that they also match the minimax rate of the variation budget setting even in the more general linear bandits framework. \citet{chen2019new} show that \ADSWITCH also matches the minimax rate without the knowledge of $V_T$. They also analyse \ADAILTCBplus , an algorithm which achieves similar guarantee in the more general linear setting. \citet{wei2016tracking} extended these results to a non-stationary environment where both the means and the variances of the rewards may change.

\paragraph{Rested rotting bandits}
Finally, \citet{heidari2016tight,levine2017rotting} and \citet{seznec2019rotting}
studied \textit{rested rotting bandits}, when the reward of an action decreases every time it is pulled.
\citet{seznec2019rotting} recently proposed a nearly-optimal algorithm for this setting. Interestingly, the algorithm does not execute an \textit{index policy} (defined later) which is a prevalent choice in bandit. Actually, a previous attempt of using an index policy by~
\citet{levine2017rotting} resulted in a sub-optimal performance. 

Our contribution is threefold:
\begin{itemize}
    \item We show that no learning strategy can achieve $o(T)$ worst case rate when we allow for both rested \underline{and} restless decay (Section~\ref{sec:general_decreasing_MAB_framework}).
    \item We introduce a novel index policy \RAWUCB (Section \ref{sec:algo}) and prove that it achieves minimax rate regret for either restless (Section~\ref{sec:restless}) \underline{or} rested (Section~\ref{sec:rested}) settings without any prior knowledge of the type of decay, the amount of change, or the horizon.
    \item \RAWUCB also recovers problem-dependent $\cO\pa{\log{T}}$ bounds in both setups. In the restless case\footnote{In the rested case, \citet{heidari2016tight} shows that increasing reward is a much harder problem, even in the absence of noise.}, such bounds cannot be achieved when the reward can increase. Hence, it shows that the decreasing assumption do help the learner compared to the well-studied general case. 
\end{itemize}

Also, we provide a rested simulated (Appendix~\ref{app:rested-sim}) and restless real-world (Section~\ref{sec:yahoo}) benchmarks on which  \RAWUCB gives the most consistent results in both setups.

\section{Decreasing multi-armed bandits}
\label{sec:general_decreasing_MAB_framework}
 At each round $t$, an agent chooses an arm $i_t \in \possibleArms \triangleq \left\{ 1, ... , K\right\} $ and receives a noisy reward $o_{t}$. The sample associated to each arm $i$ is a $\sigma^2$-sub-Gaussian r.v.\,with expected value of $\mu_{i}(t,n)$ which depends on the number of times $n$ it was pulled before and on the time $t$. 
 
 Let $\history_t \triangleq \left\{ \left\{ i(s), o_{s} \right\}, \forall s \leq t \right\}$  be the sequence of arms pulled and rewards observed until round $t$, then 
 \vspace{-4pt}
\begin{equation*}
     o_{t} \triangleq \mu_{i_t}(t,N_{i_t,t-1}) + \noise_t\,,\\
\end{equation*}
with  $\E \left[  \noise_t | \history_{t-1} \right] = 0 \; \text{and} \; \forall \lambda \in \R, \; \E\left[ e^{\lambda\noise_t}\right] \leq e^{\frac{\sigma\lambda^2}{2}}$,
where $N_{i,t} \triangleq \sum_{s=1}^t \mathds{1} \pa{i_s = i} $ is the number of pulls of arm $i$ at time $t$. We call $\mu \triangleq \left\{\mu_i\right\}_{i \in \possibleArms}$ the set of reward functions. 
\paragraph{Decreasing rewards} Throughout all the paper, we consider the following assumption.
\begin{assumption}\label{assum:general}
For each arm $i$, any number of pulls $n$,  and time $t$, the functions $\mu_i(t,\cdot)$ and $\mu_i(\cdot,n)$ are non-increasing.
\end{assumption}
We will use interchangeably the terms \textit{decreasing}, \textit{decaying} and \textit{rotting} to refer to this Assumption. If $\mu_i(t, N_{i,t}) = \mu_i(N_{i,t})$, then $i$ is called a rested arm. If $\mu_i(t, N_{i,t}) = \mu_i(t)$, then $i$ is called a restless arm. 
\paragraph{Learning problem} A (deterministic) learning policy~$\pi$ is a function that maps history of observations to arms, i.e., $\pi(\mathcal{H}_t) \in \mathcal{K}$. In the following, we often use $\pi(t) \triangleq \pi(\mathcal{H}_{t-1})$ to denote the arm pulled at time $t$. The performance of a policy $\pi$ is measured by the (expected) rewards accumulated over time, \vspace{-4pt}
\begin{equation*}
J_T(\pi, \mu) \triangleq \sum_{t=1}^T  \mu_{\pi(t)}\pa{t,N_{\pi(t),t-1}}\,.
\end{equation*}
A (deterministic) oracle policy is a function which maps the set of reward functions and a round to an arm, i.e., $\pi(t, \mu) \in \mathcal{K}$. Thus, these oracles have access to the true (without noise) value of the rewards, including future value. Notice that at the horizon $T$, there are $K^T$ distinct deterministic policies. Therefore, we call an optimal (oracle) policy, one which, at a given horizon $T$, maximizes the reward 
\vspace{-4pt}
\[
 \pi^*_T(t, \mu) \in \argmax_{\pi \in \possibleArms^T} J_T(\pi, \mu)\,.
\]
We define the regret as
\vspace{-4pt}
\begin{equation*}
R_T(\pi, \mu) \triangleq J_T(\pi^\star_T, \mu) - J_T(\pi, \mu).
\end{equation*}
Notice that this definition is more challenging than the regret w.r.t.\ the best fixed-arm policy commonly used as comparator in adversarial bandits.
In the following, we often use shorter notation $\pi^*_T(t)$, $J_T(\pi)$, $R_T(\pi)$ where the considered problem $\mu$ is implicit.

\paragraph{Greedy oracle policy} It is still unclear if 1) we can compute $\pi_T^\star$ in a tractable way; 2) if a learning policy can suffer low regret compared to this policy. We call $\pi_{O}$ the oracle policy  which selects greedily at each round $t$ the largest available reward $i_t \in \argmax_{i \in \possibleArms} \mu_{i}(t,N_{i,t-1})$.\footnote{We break the ties arbitrarily, for instance by selecting the smallest index in $\argmax_{i \in \possibleArms} \mu_{i}\pa{t,\historyt}$} We notice that this policy is optimal at any time in any restless non-stationary bandit problem $\mu(t)$. \citet{heidari2016tight} show that it is also optimal in the rested rotting bandits problem. Thus, $\pi_O$ answers positively to the first question for either rested or restless decay. In the next proposition, we show that the greedy oracle suffers linear worst-case regret when we allow for both restless and rested decay at the same time. Worse, we show that no learning policy can approach the performance of the optimal oracle at a $o{\pa{T}}$ rate

\begin{restatable}{proposition}{restaunlearnableprop}
\label{prop:unlearnable}
In the no noise setting ($\sigma = 0$), there exists a rotting 2-arms bandits problem (satisfying Assumption~\ref{assum:general}) with reward value in $\left[0,1\right]$, with one rested arm and one restless arm, and with at most one change-point before $T$ each, such that the greedy oracle strategy $\pi_O$ suffers a regret 
\vspace{-4pt}
\[R_T\pa{\pi_O} \geq \floor{\frac{T}{4}}.\]
Moreover, for any learning strategy  $\pi_S$, there exists a rotting 2-arms bandits problem (satisfying Assumption~\ref{assum:general}) with reward value in $\left[0,1\right]$, with one rested arm and one restless arm, and with at most one change-point before $T$ each, such that 
\vspace{-4pt}
\[R_T\pa{\pi_S} \geq \floor{\frac{T}{8}}.\]
\end{restatable}

Notice that the two reward functions of the constructed difficult problems are simple: either rested or restless, bounded and with at most one break-point. If we consider a 2-arm setup with one rested arm and one restless arm, a good strategy may be to select the restless arm even when its current value is the worst. Indeed, this value is only available now, while the good value of the rested arm will still be available in the future. Whether the restless rewards are interesting to the learner depends on the future behavior of the (currently best) rested arm. On the first hand, if it decays below the current value of the restless arm before the horizon $T$, then the learner should profit from the restless reward available right now. On the other hand, if the rested arm stays optimal until the end of the game then the learner should ignore the restless arm and follows the greedy oracle strategy. However, the learner does not know in advance if (and how much) an arm will decay and any anticipation she makes will turn to be bad in the worst case. We formalize these ideas in the proof in Appendix~\ref{app:unlearnable} and show that any strategy suffers linear regret in the worst case. 

While learning with rested and restless rotting reward is impossible, we show in the next sections that a single policy reaches near-optimal guarantee in both separated setups.

\section{The {\RAWUCB} algorithm}
\label{sec:algo}
\paragraph{Notation}  For policy $\pi$, we define 
the average of the last $h$ observations of arm $i$ at time $t$ as
\vspace{-4pt}
\begin{equation}
\label{eq:hmu}
    \widehat{\mu}_i^h(t,\pi) \triangleq \frac{1}{h}\sum_{s=1}^{t-1} \mathds{1} \!\pa{\pi\!\pa{\!s\!}\! =\! i \land N_{i,s}\!>\! N_{i,t\!-\!1}\! -\! h } o_{s}
\end{equation}
and the average of the associated means as
\vspace{-4pt}
\[\bar{\mu}_i^h(t,\!\pi) \!\triangleq\! \frac{1}{h}\!\sum_{s=1}^{t-1} \mathds{1}\!\!\pa{\!\pi\!\pa{\!s\!}\! =\! i \!\land\! N_{i,s}\!>\! N_{i,t\!-\!1}\! -\! h \!} \mu_i\!\pa{s,\! N_{i,s\!-\!1}}.\] 

\paragraph{A favorable event} 
We use a similar high probability analysis than \UCBone. We design a favorable event and we show in Prop.~\ref{prop:prb_favorable_event} that it holds with high probability.
\vspace{-4pt}
\begin{restatable}{proposition}{restachernoff}
\label{prop:prb_favorable_event}
For any round $t$ and confidence $\delta_{t} \triangleq 2t^{-\alpha}$, let 
\vspace{-4pt}
\begin{equation}
\begin{split}
      \HPevent \triangleq \Big\{ & \forall i\in\!\arms,\ \forall n \!\leq\! t\!-\!1 ,\ \forall h \!\leq\! n, \nonumber \\ 
    & \big| \hmu^h_i(t, \pi) - \bmu^h_i(t,\pi) \big| \leq c(h, \delta_{t}) \Big\}  
\end{split}
\label{eq:def_favorable_event}
\end{equation}
 be the event under which the estimates at round $t$  are all accurate up to $c(h,\delta_{t}) \triangleq \sqrt{2 \subgaussian^2\log(2/\delta_t)/h}$. Then, for a policy $\pi$ which pulls each arms once at the beginning, and for all $t>K$,
 \vspace{-4pt}
\[
\PPempty\Big[\bar{\HPevent}\Big] \leq \frac{Kt^2\delta_{t}}{2}=Kt^{2-\alpha}\,\cdot
\]
\end{restatable} 

\paragraph{Rotting Adaptive Window Upper Confidence Bound ({\RAWUCB} or $\piR$).}
At each round, \RAWUCB selects the arm with the largest following index,
\vspace{-4pt}
\begin{align}
\label{eq:xindex}
\operatorname{ind}(i,t, \delta_{t}) \triangleq \min_{h\leq N_{i,t\!-\!1}} {\widehat{\mu}_i^h(t,\piR) + c(h,\delta_{t})},
\end{align}
with $\delta_{t} \triangleq \frac{2}{t^\alpha}$. There is  a bias-variance trade-off for the window choice: more variance for smaller size of the window $h$ and more bias for larger $h$. The goal of \RAWUCB is to adaptively select the right window to compute the tightest UCB. \RAWUCB uses the indexes of \UCBone computed on all the slices of each arm's history which include the last pull. When the rewards are rotting, all these indexes are upper confidence bounds on the \textit{next value}.  Thus, \RAWUCB simply selects the tightest (minimum) one as index of the arm: it is a pure UCB-index algorithm. By contrast, when reward can increase, the learner can only derive upper-confidence bound on past values which are loosely related to the next value. Hence, all the UCB-index algorithms in the restless non-stationary literature need to add change-detection sub-routine, active random exploration or passive forgetting mechanism. In Lemma~\ref{lem:core-RAWUCB}, we show a guarantee of \RAWUCB on the favorable event. 

\begin{algorithm}[t]
\caption{\RUCB}
\begin{algorithmic}[1]
\label{alg:RAWUCB}
\Require $\arms$,  $\subgaussian$, $\alpha$
\For{$t \gets 1, 2, \dots, K \do $}{\footnotesize \Comment{\emph{Pull each arm once}}}
	\State \textsc{Pull}  $i_t \gets t$; \textsc{Receive} $o_{t}$ ; $N_{i_t} \gets 1$
	\State $\left\{\hmu_{i_t}^h\right\}_h \gets \UPDATE(\left\{\hmu_{i_t}^h\right\}_h, o_t)$ \Comment{\emph{cf.\,\eqref{eq:hmu}}} \label{algline:raw-update1}
\EndFor
\For{$t \gets K+1, K+2, \dots \do $}
	\State \textsc{Pull}  $ i_t \!\in\! \argmax_i \min_{h \leq N_{i}}\!\hmu_i^h \!+\! c(h,\! \delta_t) $ {\footnotesize \Comment{\emph{cf.\,\eqref{eq:xindex}}}}
	\State \textsc{Receive} $o_{t}$ \label{algline:raw-pull};  $N_{i_t} \gets N_{i_t} +1$
	\State $\left\{\hmu_{i_t}^h\right\}_h \gets \UPDATE(\left\{\hmu_{i_t}^h\right\}_h, o_t)$\Comment{\emph{cf.\,\eqref{eq:hmu}}}\label{algline:raw-update2}
\EndFor
\end{algorithmic}
\end{algorithm}

\begin{restatable}{lemma}{restafundamentallemma}
\label{lem:core-RAWUCB}
At round $t$ on favorable event $\HPevent$, if arm~$i_{t}$ is selected, for any $h \leq N_{i,t-1}$,  the average of its $\window$ last pulls cannot deviate significantly from the best available arm at that round, i.e.,
\vspace{-4pt}
\begin{equation*}
\bar{\mu}^{h}_{i_{t}}(t,\pi) \geq \max_{i \in \possibleArms}\mu_i(t,N_{i,t-1}) - 2 c(h, \delta_{t}).
\end{equation*}
\end{restatable}

\citet{seznec2019rotting} show a slightly worse guarantee about the algorithm {\FEWA ($\piF$)} for the rested rotting bandits. In Appendix~\ref{app:proof_fundamental_lemma} (see Lemma~\ref{lem:core-FEWA}), we restate their result using only Assumption~\ref{assum:general}.  \FEWA uses the same statistics than \RAWUCB but in a rather complex expanding filtering mechanism which leads to a guarantee of only 4 confidence bounds. Lemma~\ref{lem:core-RAWUCB} is the only characterization we need for our analysis. Therefore, all our upper bounds will hold for both \FEWA and \RAWUCB with their associated constant, 
\vspace{-4pt}
\begin{equation}
    C_\piR \triangleq 2 \sqrt{2\alpha} \quad   
    C_\piF \triangleq 4 \sqrt{2\alpha}.
\end{equation}

\paragraph{Algorithmic complexity} \FEWA and \RAWUCB have $\cO(Kt)$ per round time and space complexity. In Appendix~\ref{app:efficient_alg}, we describe \EFFRAW($\piER$) and \EFFFEWA($\piEF$), two algorithms which reduces the complexities to $\cO\pa{K\log_m(t)}$. It is a refinement of the trick of \citet{seznec2019rotting} where we add a parameter $m > 1$ to trade-off between complexity and efficiency\footnote{When $m < 1+ \frac{1}{T}$, \EFFRAW behave as \RAWUCB.}. For $m=2$, we prove Lemma~\ref{lem:core-eff} and Prop.~\ref{prop:prb_favorable_event_eff}, which are comparable with Lemma~\ref{lem:core-RAWUCB} and Prop.~\ref{prop:prb_favorable_event}. Therefore, our analysis also holds for these algorithms with, 
\vspace{-4pt}
\begin{equation}
    C_\piER \triangleq \frac{4 \sqrt{\alpha}}{\sqrt{2}-1} \quad   
    C_\piEF \triangleq \frac{8 \sqrt{\alpha}}{\sqrt{2}-1}\cdot
\end{equation}
The efficient algorithms use less statistics than the original ones. Thus, the probability of the unfavorable event is bounded by $\cO\pa{t^{1-\alpha}}$ (see Prop.~\ref{prop:prb_favorable_event_eff}) which is smaller than $\cO\pa{t^{2-\alpha}}$ in Prop.~\ref{prop:prb_favorable_event}. Hence, our theory holds for a wider range of $\alpha$ for the efficient algorithms. 
\section{Restless rotting bandits}
\label{sec:restless}
In this section, the reward decreases independently of the user actions. Hence, we have that $\mu_i\pa{t,n} = \mu_i\pa{t}$.

\subsection*{Variation budget bandits}
\paragraph{Setup.} \cite{besbes2014stochastic} introduce the limited variation budget bandits, a restless setting where at each round Nature can modify the reward value of any arm but with a limited total variation budget $V_T$ at round T. We combine this assumption with Assumption~\ref{assum:general},

\begin{assumption}
\label{assum:variation}
$\mu_i : \NN^\star \rightarrow [- V_T, 0]$ are decreasing functions of the time $t$ with $V_T$ a positive constant. Moreover, we have that, 
\vspace{-4pt}
\begin{equation}
\label{eq:defbudget}
    \sum_{t=1}^{T-1} \sup_{i \in \possibleArms} \pa{\mu_i(t) - \mu_i(t+1) } \leq V_T\,.
\end{equation}
\end{assumption}
\begin{remark}
\label{rem:budget}
In the rotting scenario, the budget assumption is very similar to the bounded assumption. Indeed, any set of decreasing functions $\mu_i : \NN^\star \rightarrow [- V, 0]$ satisfies Equation~\ref{eq:defbudget} with $V_T = KV$. Reciprocally, any set of functions satisfying Equation~\ref{eq:defbudget} with $\mu_i(1) \in [- V_T, 0]$ are bounded in $[- 2V_T, 0]$. 
\end{remark}
\paragraph{Lower bound.} We show that our additional decreasing assumption does not change the minimax rate for budget bandits. This is an adaptation of the proof of \citet{besbes2014stochastic} where we only use rotting function.

\begin{restatable}{proposition}{restavariationlb}
\label{prop:variation_lb}
For any strategy $\pi$, there exists a \underline{rotting} variation budget bandit scenario with means $\{\mu_i(t)\}_{i,t}$ \underline{satisfying Assumption~\ref{assum:variation}} with a budget $V_T \geq \sigma \sqrt{\frac{K}{8T}}$ such that,
\vspace{-4pt}
\[
    \mathbb{E}\left[R_T(\pi)\right] \geq \frac{1}{16\sqrt{2}} \pa{\sigma^2 V_T KT^2}^{\nicefrac{1}{3}}.
\]
\end{restatable}

\paragraph{Upper bound.}\RAWUCB matches this lower bound up to poly-logarithmic factors without any knowledge of the horizon $T$ nor the budget $V_T$.

\begin{restatable}{theorem}{restabudgettheorem}
\label{th:variation-minimax}
Let $\pi \in \left\{ \piF, \piR\right\}$ tuned with $\alpha \geq 4$ or $\pi \in \left\{ \piEF, \piER\right\}$ tuned with $\alpha \geq 3$ and $m=2$. For any variation budget bandit scenario with means $\{\mu_i(t)\}_{i,t}$ satisfying Assumptions~\ref{assum:variation} with variation budget $V_T$, $\pi$ suffers an expected regret,
\vspace{-4pt}
\[
\mathbb{E}\left[R_T(\pi)\right] \! \leq 4\!\pa{C_\pi^2 \sigma^2 V_T\! K T^2\log{\!T}}^{\!\nicefrac{1}{3}\!} \!+ \tcO\!\pa{\!\pa{ \sigma V_T^2\!  K^2  T}^{\!\nicefrac{1}{3}}\!}\!.
\]
\end{restatable}
The remaining terms are of second order when $KV_T \leq \cO{\pa{T}}$, which is a necessary condition for the problem to be learnable (see Proposition~\ref{prop:variation_lb}).

\subsection*{Piece-wise stationary bandits.} 
\paragraph{Setup.} In this section, we also consider bounded functions. Hence, they also satisfy Assumption~\ref{assum:variation} (see Remark~\ref{rem:budget}). However, we further restrained them to be piece-wise stationary,
 \begin{assumption}\label{assum:piecewise}
Let $V$ be a positive constant and $\Upsilon_T$ a positive integer.  $\mu_i : \NN^\star \rightarrow [- V, 0]$ are piece-wise stationary non-increasing functions of the time $t$ with at most $\Upsilon_T-1$ breakpoints.
\end{assumption}
Formally, $\sum_{t=1}^{T-1} \mathds{1}\left(\exists i\!\in\! \possibleArms, \mu_i(t) \!\neq\! \mu_i(t\!+\!1)\right)\leq \Upsilon_T\!-\!1$. We call $\left\{t_k\right\}_{k \leq \Upsilon-1}$ the set of breakpoints with $t_0 = 0$, $\mu_i^k$ the value of $\mu_i(t)$ for $t \in \left\{t_k+1 , \dots, t_{k+1}\right\}$. We call $i^\star_k \in \argmax_{i \in \possibleArms}{\mu_i^k}$ (one of) the best arm in batch $k$ and $\Delta_{i,k} \triangleq \mu_{i^\star_k}^k - \mu_i^k$ the gap to the best arm for arm $i$ during batch $k$. Note that we relax all the assumptions related to the distance between consecutive breakpoints (e.g. \citet{besson2019generalized} and their Assumption~4 and~7; \citet{liu2018change-detection} and their Assumption~1 and~2; \citet{cao2019nearly} and their Assumption~1). 
 

\paragraph{Lower bound.} We show that our additional Assumption~\ref{assum:general} does not decrease the minimax rate of \citet{garivier2011upper-confidence-bound}. 
\begin{restatable}{proposition}{restapiecewiselb}
\label{prop:piecewise_lb}
For any strategy $\pi$, there exists a \underline{rotting} piece-wise stationary bandit scenario with means $\{\mu_i(t)\}_{i,t}$ \underline{satisfying Assumption~\ref{assum:piecewise}} with $\Upsilon_T \leq \pa{\frac{32V^2T }{K\sigma^2}}^{\!\nicefrac{1}{3}\!}\!$ such that,
\vspace{-4pt}
\[
    \mathbb{E}\left[R_T(\pi)\right] \geq \frac{\sigma}{32}\sqrt{ \Upsilon_T KT}\,.
\]
\end{restatable}

The condition on $\Upsilon_T$ in Proposition~\ref{prop:piecewise_lb} follows from Remark~\ref{rem:budget}: if $V$ is too small compared to $\Upsilon_T$, then we have a budget constraint (with associated lower bound in Proposition~\ref{prop:variation_lb}) rather than a break-point constraint.

\paragraph{Upper bound.} \RAWUCB matches the lower bound from Proposition~\ref{prop:piecewise_lb} up to poly-logarithmic factors without any knowledge of the horizon $T$ nor the number of breakpoints $\Upsilon_T -1$.

\begin{restatable}{theorem}{restapiecewisetheorem}
\label{th:piecewise-minimax}
Let $\pi \in \left\{ \piF, \piR\right\}$ tuned with $\alpha \geq 4$ or $\pi \in \left\{ \piEF, \piER\right\}$ tuned with $\alpha \geq 3$ and $m=2$. For any piece-wise stationary bandit scenario with means $\{\mu_i(t)\}_{i,t}$ satisfying Assumption~\ref{assum:piecewise}  with $\Upsilon_T-1$ change-points, $\pi$  suffers an expected regret,
\[
\EE{R_T(\pi)} \leq C_\pi \sigma \sqrt{\log{T}} \pa{ \sqrt{\Upsilon_T KT} + \Upsilon_T K} + 6KV.
\]
\end{restatable}
\paragraph{Are rotting restless bandits easier?} Learning at the minimax rate without knowing $\Upsilon_T$ or $V_T$ was achieved in the non-rotting setup by significantly more complex algorithms. For instance, \cite{auer2019adaptively} use a combination of filtering on the set of potentially good arms, forced exploration planning on identified bad arms, and full restart of the algorithm when a change is detected. This algorithmic complexity has a performance cost, as \ADSWITCH is guaranteed to achieve 56 times the leading term in Theorem~\ref{th:piecewise-minimax}. Moreover, these algorithms rely on doubling trick when the horizon is unknown, which also has a regret cost compared to intrinsically anytime algorithms \citep{besson2018doubling}.

Yet, Proposition~\ref{prop:variation_lb} and ~\ref{prop:piecewise_lb} show that the rotting assumption do not improve the minimax rate for the two considered setups. Interestingly both these lower bounds are matched by (tuned) \EXPS \citep{auer2002nonstochastic}, an algorithm originally designed for switching best arm in adversarial sequence of rewards. This is comparable to the fixed best arm world:  adversarial and stochastic bandits share the same minimax rate which is matched in both setups by \EXP. The main interest of the stochastic assumption is to allow for \textit{problem dependent analysis}. For the stochastic stationary bandits, it leads to a stronger $\cO{\pa{\log\pa{T}}}$ bounds. In the piece-wise stationary setting, \citet{garivier2011upper-confidence-bound} show that such bounds cannot be achieved without sacrificing the minimax optimality. 

\begin{proposition}[Theorem~31.2, \citet{lattimore2020banditbook}]
\label{prop:lattimore}
If a policy $\pi$ performs a regret $R_T(\pi, \mu)$ on a 2-arm stationary instance $\mu$, one can find a piece-wise stationary instance $\mu'$ with only two breakpoints such that, for a sufficiently long horizon $T$, the regret is lower bounded by 
\vspace{-4pt}
\[\EE{R_T(\pi, \mu')} \geq \frac{T}{22R_T(\pi, \mu)}\cdot\]  
\end{proposition}
\vspace{-10pt}
\begin{corollary}
Let $\pi$ a minimax policy on the (non-rotting) piece-wise stationary setups. Then, for a sufficiently large horizon $T$, there exists a universal constant $C$ such that for all the 2-arm stationary problems $\mu$, 
\vspace{-4pt}
\[
\EE{R_T(\pi,\mu)} \geq C\sqrt{T}.
\]
\end{corollary}

The proof of Proposition~\ref{prop:lattimore} is instructive. It builds a problem $\mu'$ on which the reward function equals the reward of the stationary problem $\mu$ except on a time span $\tau$. During  this time span, the best arm of $\mu$ keeps its value while the worst arm \textit{increases} to become optimal. The size of $\tau$ is chosen inversely proportional to the average pulling rate of the bad arm in $\mu$. Indeed, the lower the pulling rate of the bad arm, the longer the adversary can increase its value in $\mu'$ without being noticeable by the learner. Since the pulling rate of the bad arm in $\mu$ is proportional to $R_T(\mu)$, we get a lower bound proportional to $\tau \sim \frac{T}{R_T(\mu)}$.

The decreasing Assumption~\ref{assum:general} excludes this $\mu'$ from the set of possible problems. Theorem~\ref{th:piecewise_pd} shows that not only \RAWUCB is able to recover the $\cO\pa{\log\pa{T}}$ on stationary problems but also recovers the same rate on each batch of a rotting piece-wise stationary problem. 
\begin{restatable}{theorem}{restapiecewisetheorempd}
\label{th:piecewise_pd}
Let $\pi \in \left\{ \piF, \piR\right\}$ tuned with $\alpha \geq 4$ or $\pi \in \left\{ \piEF, \piER\right\}$ tuned with $\alpha \geq 3$ and $m=2$. For any piece-wise stationary bandit scenario with means $\{\mu_i(t)\}_{i,t}$ satisfying Assumption~\ref{assum:piecewise} with $\Upsilon_T-1$ change-points, $\pi$ suffers an expected regret\,
\[
    \mathbb{E}\left[R_T(\pi)\right] \leq\! \sum_{k=0}^{\Upsilon_T-1} \! \sum_{i\in\arms} \frac{C_\pi^2 \sigma^2\log{T}}{\Delta_{i,k}} +  \cO\pa{ \sigma \Upsilon_T K \sqrt{ \log{T}}}\!.
\]
\end{restatable}

Like in \UCBone’s analysis, Proposition~\ref{prop:prb_favorable_event} uses a union-bound with Hoeffding inequality. This technique leads to conservative theoretical tuning of confidence levels and to a suboptimal constant factor $C_\pi^2/2$. One can get the asymptotic optimal tuning for \UCB on stationary gaussian bandits with a refined analysis which uses a specific concentration result on the deviation of the index (e.g. Lemma 8.2, \cite{lattimore2020banditbook}). Yet, extending this result to our more complex meta-index and to our several setups is not straightforward and we leave it as future work. Interestingly, the experimental tuning $\alpha = 1.4$ is very close to the asymptotic tuning of \UCB (see Section~\ref{sec:yahoo}). It suggests that, besides our union bound considers more events than \UCB in the theory, we do not have to be significantly more conservative on the confidence levels in practice.

Notice that \citet{mukherjee2019distribution} use a different assumption to get a similar problem-dependent bound. Indeed, they assume that all the arms change at the same time which also excludes $\mu'$ from the set of possible problems.

\subsection*{Proofs sketch (full proofs in Appendix~\ref{app:restless})}

\paragraph{Lower bounds.} Our proof technique make a strong connection between Proposition~\ref{prop:variation_lb} and ~\ref{prop:piecewise_lb}. Yet, we adapt existing proofs to the decreasing case \citep{garivier2011upper-confidence-bound, besbes2014stochastic}. Hence, we defer the full proof and its sketch to Appendix~\ref{app:restless}. 

\paragraph{Upper bounds.} We start by separating the regret on the bad events $\bar{\HPevent}$ from the good events $\HPevent$. According to Proposition~\ref{prop:prb_favorable_event}, the bad events $\bar{\HPevent}$ have low probability for appropriate $\alpha$. For $\alpha = 4$, they weigh at most $\cO{\pa{KV}}$ in the expected regret.  On the good events, we write:
\vspace{-4pt}
\begin{equation}
\label{eq:restless-regret-decompo}
R_T(\pi)= \sum_{t=1}^T \mu_{i_t^\star}(t) - \bar{\mu}_{i_t}^{h_t}(t, \pi) + \bar{\mu}_{i_t}^{h_t}(t, \pi) - \mu_{i_t}(t)\,.   
\end{equation}

Notice that Lemma~\ref{lem:core-RAWUCB} can bound the first difference for any $h_t$. When the reward is piece-wise stationary, we can select $h_t$ such that we include all the pulls of arm $i_t$ from the current stationary batch. If there is none, then it is the first pull of arm $i_t$ in this batch. We handle these $\cO{\pa{K\Upsilon_T}}$ rounds separately (see Lemma~\ref{lem:FP} in Appendix~\ref{app:restless}). In the other cases, we note that the second difference is null because $\bar{\mu}_{i_t}^{h_t}(t, \pi) = \mu_{i_t}(t) = \mu_i^k$ by the piece-wise stationary assumption. The remaining of the proofs of Theorem~\ref{th:piecewise-minimax} and~\ref{th:piecewise_pd} are then very similar to the analysis of \cite{auer2002finite} on each stationary batch. Indeed, the two confidence bounds guarantee of Lemma~\ref{lem:core-RAWUCB} is similar to \UCBone's guarantee.

In the variation budget setting, there is no stationary batches. Hence, we cannot choose an $h_t$ which cancels the second difference in Equation~\ref{eq:restless-regret-decompo}. Yet, we still decompose the rounds in $\Upsilon$ batches of equal length for the analysis. We choose $h_t$ such that we include all the pulls of arm $i_t$ from the current batch. For the sum of the first differences in Equation~\ref{eq:restless-regret-decompo}, there is no difference with the piece-wise stationary case and we can bound
\vspace{-4pt}
\begin{equation}
\label{eq:variance_bound}
    \sum_{t=1}^T \mu_{i_t^\star}(t) - \bar{\mu}_{i_t}^{h_t}(t, \pi)\leq \tcO{\pa{\sqrt{K\Upsilon T}}}\,.
\end{equation}
We call $\Delta_i^k \triangleq \mu_i(t_k) - \mu_i(t_{k+1})$, the total variation of arm $i$ in batch $k$. The sum of second differences in Equation~\ref{eq:restless-regret-decompo} can be bounded as follows: on each batch of $T\Upsilon^{-1}$ rounds, each second difference is bounded by $\max_{i\in \possibleArms} \Delta_i^k$. When we sum over the batches, we get
\vspace{-4pt}
\begin{equation}
\label{eq:bias_bound}
  \sum_{t=1}^T  \bar{\mu}_{i_t}^{h_t}(t, \pi) - \mu_{i_t}(t)\leq \frac{T}{\Upsilon}\sum_{k=0}^{\Upsilon-1}\max_{i \in \possibleArms}\Delta_i^k  \leq \frac{TV_T}{\Upsilon}\, .  
\end{equation}
Indeed, in the middle term, we have a maximum on the summed variation of arm $i$ in batch $k$. On the right-hand side, we have $V_T$ which bounds the sum over the rounds of maximal variation of the arms (see Equation~\ref{eq:defbudget}). Thus, the right-hand side is larger because the maximum of sums is smaller than the sum of maximums. We can then choose $\Upsilon = \tcO{\pa{T^{1/3}V_T^{2/3}K^{-1/3}}}$ to minimise the sum of Equation~\ref{eq:variance_bound} and ~\ref{eq:bias_bound}. It leads to the leading term of our Theorem~\ref{th:variation-minimax}. Notice that we still have to handle the first pull of each arm in each batch. If we bound roughly each first pull by $V_T$, we would get $K\Upsilon V_T \sim \tcO{\pa{V_T^{5/3}}}$ which would be the leading term for large $V_T$. Our Lemma~\ref{lem:FP} is more careful such that it leads to a second order term when $KV_T \leq o\pa{T}$.

 \section{Rested rotting bandits}
 \label{sec:rested}
 \paragraph{Setup} We use the rotting setup of \cite{seznec2019rotting}, which extends the one of \cite{levine2017rotting}. This setup is \emph{rested} non-stationary bandits: the change in arm's reward is triggered by the pulls. Hence, we have $\mu_{i}(t,n) = \mu_i(n)$. Thus, we note that $\bar{\mu}_i^h(t, \pi) = \bar{\mu}_i^h(\Nitmone) = \frac{1}{h} \sum_{s=0}^{h-1} \mu_i(\Nitmone-s)$. With a slight abuse of notations, we will also use $\hat{\mu}_i^h(\Nitmone) \triangleq \hat{\mu}_i^h(t, \pi)$\footnote{The average of the observations depends on the realization of the noise $\epsilon_t$ at time $t$. Yet, these $h$ samples of noise are i.i.d.\,and thus do not perturb the analysis (see Prop.~\ref{prop:prb_favorable_event}).}. Let 
 \vspace{-4pt}
 \begin{multline}
      L \triangleq \max_{i \in \possibleArms} \max_{n \in \left\{0,\dots, T-1\right\}} \mu_i(n) - \mu_i(n-1) ,\label{eq:LT}\\
      \text{ with } \mu_i(-1) \triangleq \max_{j \in \possibleArms} \mu_j(0).\nonumber 
 \end{multline}
Hence, $L$ bounds both the variation of $\mu_i$s between two consecutive pulls and the gaps between arms at the first pulls. This is an important quantity for the rested rotting analysis because the minimax rate for the noise-free case is $\cO\!\pa{KL}$ \citep{heidari2016tight}.
 

\paragraph{Theoretical guarantees}
The analysis of \RAWUCB is straightforward from the analysis of \FEWA due to their similarity. Thus, we recover the problem independent and dependent bounds (see \citet{seznec2019rotting} for a sketch of the proof, and App.~\ref{app:rested-theory} for a detailed analysis).
 
\begin{restatable}[gap-free bound]{proposition}{restaalgoindepub}
\label{prop:rested-PI}
Let $\pi \in \left\{ \piF, \piR\right\}$ tuned with $\alpha \geq 5$ or $\pi \in \left\{ \piEF, \piER\right\}$ tuned with $\alpha \geq 4$ and $m=2$. For any rotting bandit scenario with means $\{\mu_i\}_{i}$ satisfying Assumption~\ref{assum:general} with bounded decay $L$ and any time horizon $T$, $\pi$ suffers an expected regret,
\vspace{-4pt}
\begin{equation*}
\EE{\regret(\pi)} \leq C_\pi \sigma \sqrt{\log\pa{T}}\pa{\sqrt{KT} +K} + 6KL.
\end{equation*}
\end{restatable}

\begin{restatable}[gap-dependent bound]{proposition}{restaalgoub}\label{prop:rested-PD}
$\pi \in \left\{ \piF, \piR\right\}$ tuned with $\alpha \geq 5$ (or $\pi \in \left\{ \piEF, \piER\right\}$ tuned with $\alpha \geq 4$ and $m=2$) suffers an expected regret,
\vspace{-4pt}
\begin{align*}
\EE{\regret(\pi)} \leq \sum_{i\in \arms} \pa{\!\frac{C_\pi^2\sigma^2\log\pa{T}}{\Delta_{i,\hiT^+-1}} + C_\pi \sigma \sqrt{ \log\pa{T}} +6L\! }
\end{align*}
\text{with $\!h_{i,T}^+\! \triangleq\! \max \left\{\! \window \!\leq\! 
1 \!+\! \frac{C_\pi^2 \!\subgaussian^2 \!\log T}{\Delta_{i,h-1}^2} \!\right\}\!\CommaBin$} and the pseudo-gap
\[\Delta_{i,h} \triangleq \min_{j \in \possibleArms} \reward_j\pa{N_{j,T}^\star -1} - \bar{\reward}_i^h\left( N_{i,T}^\star+h \right).\]
\end{restatable}

\RAWUCB matches the minimax rate (Prop.~\ref{prop:rested-PI}) up to poly-logarithmic factors.  \RAWUCB improves over \FEWA's problem-dependent guarantee by a factor $4$ (Prop.~\ref{prop:rested-PD}). Following Remark~1 of \cite{seznec2019rotting}, one can identify $\Delta_{i,h} = \Delta_i$ in the stationary setting. It gives almost the same guarantee than in Theorem~\ref{th:piecewise_pd} when $\Upsilon_T = 1$ (stationary case). The difference comes from the increased $\alpha$ for the rested case. Indeed, in the rested case, the regret at each round $t$ can be as bad as $Lt$. Hence, we  reduce the probability of the bad event $\bar{\HPevent}$ (see Prop.~\ref{prop:prb_favorable_event}). When the reward means are bounded (e.g. for Bernoullis), we can decrease the lower bound on $\alpha$ by one in Propositions~\ref{prop:rested-PI} and~\ref{prop:rested-PD}.

\section{Real-word data experiment on Yahoo! Front Page}
\label{sec:yahoo}
\paragraph{R6A - Yahoo! Front page today module user click log dataset} 
This dataset was used for
the Exploration and Exploitation Challenge\footnote{\url{http://explochallenge.inria.fr/}}
at ICML 2012 and inspired new algorithms. Among them 
we mention the work of \cite{traca2015regulating}
who noticed the non-stationary trend and took advantage of it. Since then the dataset continues to be a benchmark\footnote{As 
it allows for offline evaluations as the actions
were samples uniformly.}
for non-stationary bandits \citep{liu2018change-detection,cao2019nearly}. 
It contains the history of clicks on news articles of 45 millions users in the first ten days of May 2009.
 We use three features in this dataset: \textit{timestamp} (rounded every 5 minutes), \textit{article$\!\_\!$id}, and \textit{click}. 
 
\paragraph{A real decaying scenario} Every day, between 6pm and 6am EST (12 hours), we notice a decreasing trend in click probability. It suggests that people in the US read less and less news during the evening and night. For every day, we keep all the articles which have been recommended at every timestamp during the 12 hours. For these articles, we use a rolling average window of 30000 in order to estimate the probability of click for each article at each timestamp \footnote{For each timestamp, we average the values given by rolling average. These values are close to each other because the number of click opportunity per article in the same timestamp is small compared to 30000.}. We use the \underline{real} total traffic for each timestamp. We highlight that \textit{we do not enforce any of our assumptions} to create reward functions to be aligned with our setup. In particular, we do not enforce them to be piecewise constant nor to be decreasing. At each round, the learner receives 10 reward samples in order to reduce the cost of computation.

\paragraph{Algorithms and Parameters.} We compare 
\RAWUCB, \FEWA, \EXPS and \GLRUCB. We refer to Appendix~\ref{app:experiments} for a discussion about missing algorithms and tuning. Note that our goal is to compare algorithms with the same tuning in the rested and restless benchmark. 

\paragraph{Results} We display the results for two different days. On day~2, there are several switches of optimal arms with many near-optimal ones: tracking the best arm is an "hard" problem. On day~7, one arm consistently dominates the others by far. Hence, it is an "easy" case where good algorithms should have a logarithmic regret rate. We show the six other days and running time in App.~\ref{app:experiments-appendix}.

\begin{figure}[!ht]
   \centering
{\includegraphics[width=.22\textwidth]{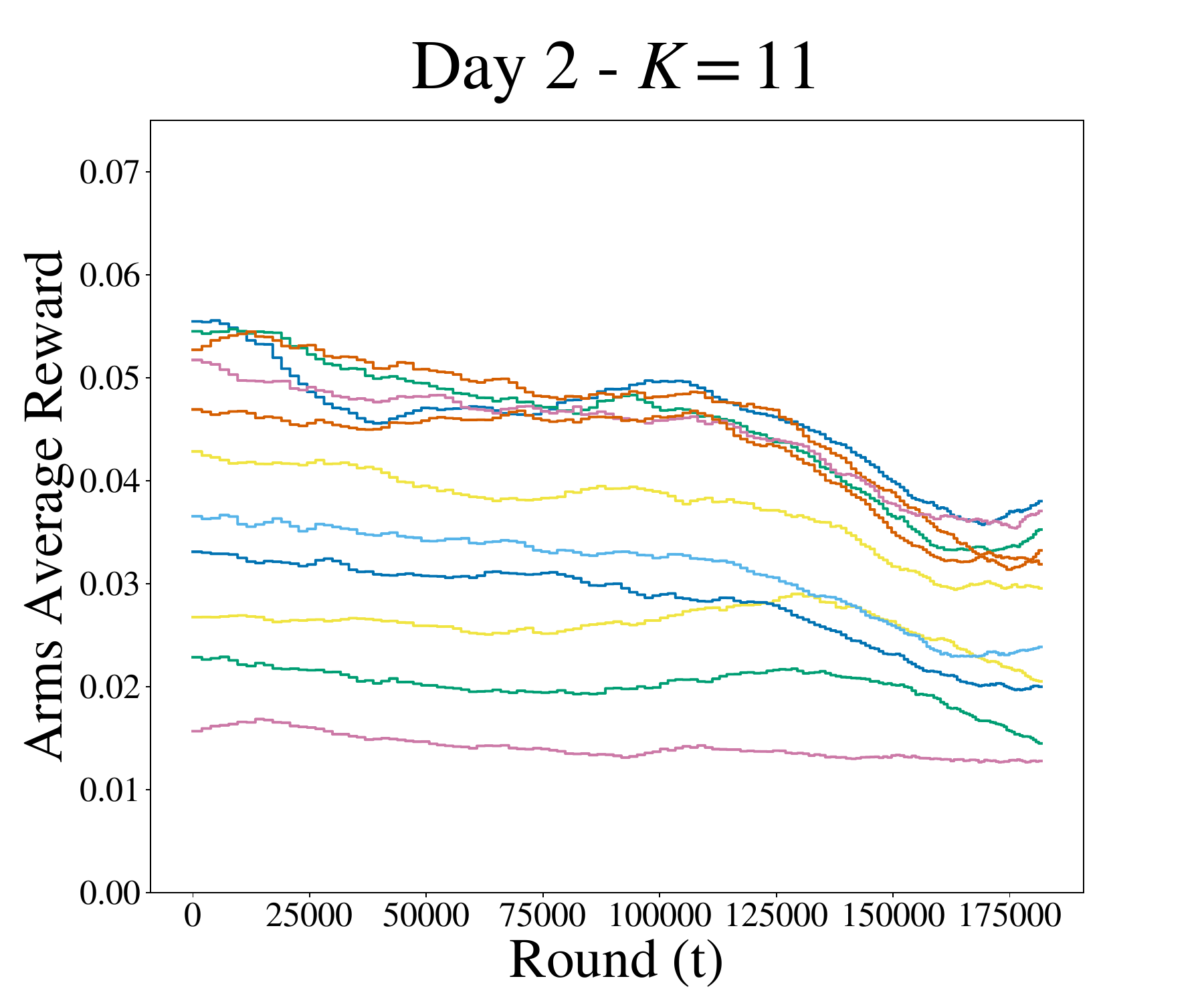}}\quad
{\includegraphics[width=.22\textwidth]{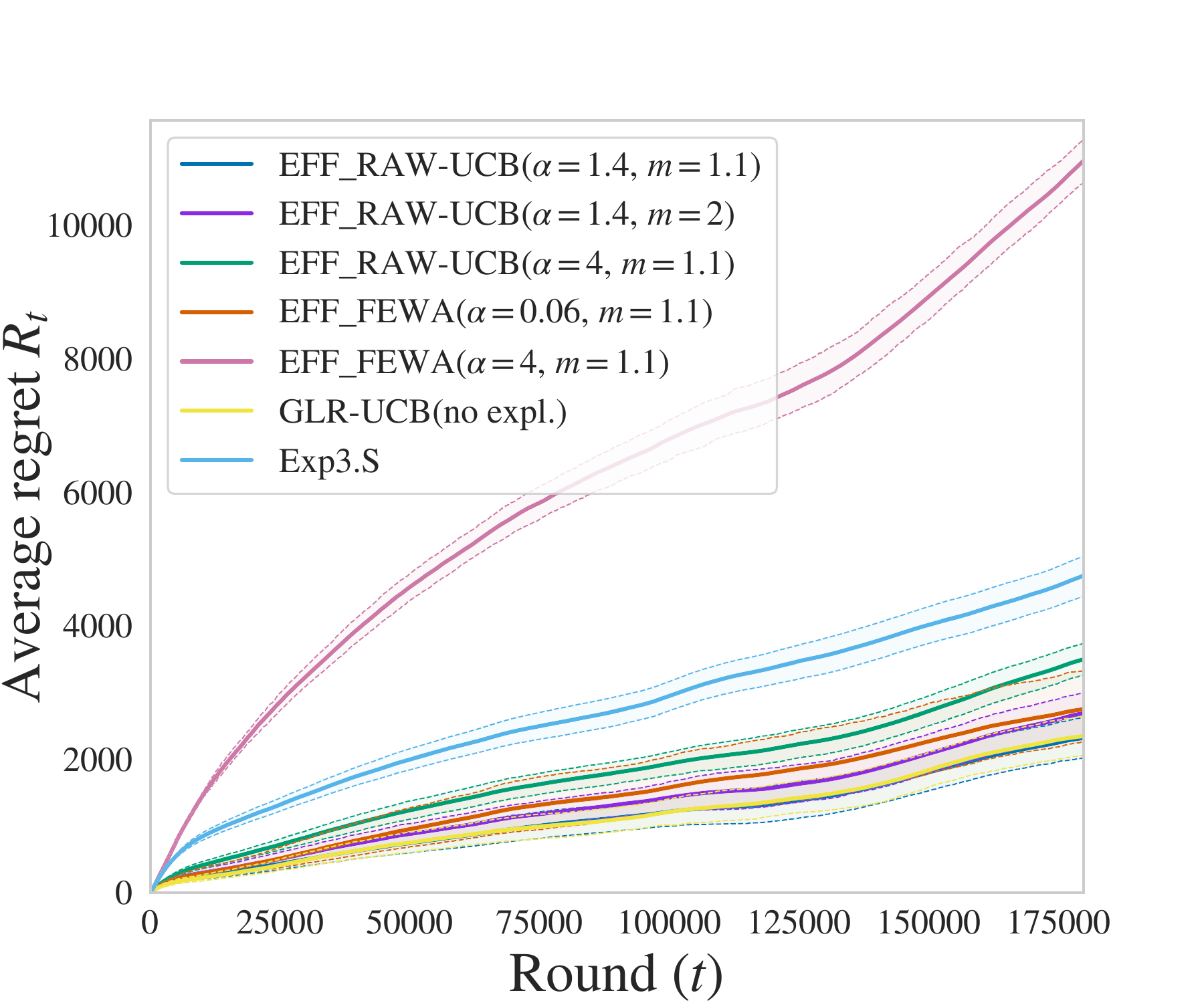}}\\
{\includegraphics[width=.22\textwidth]{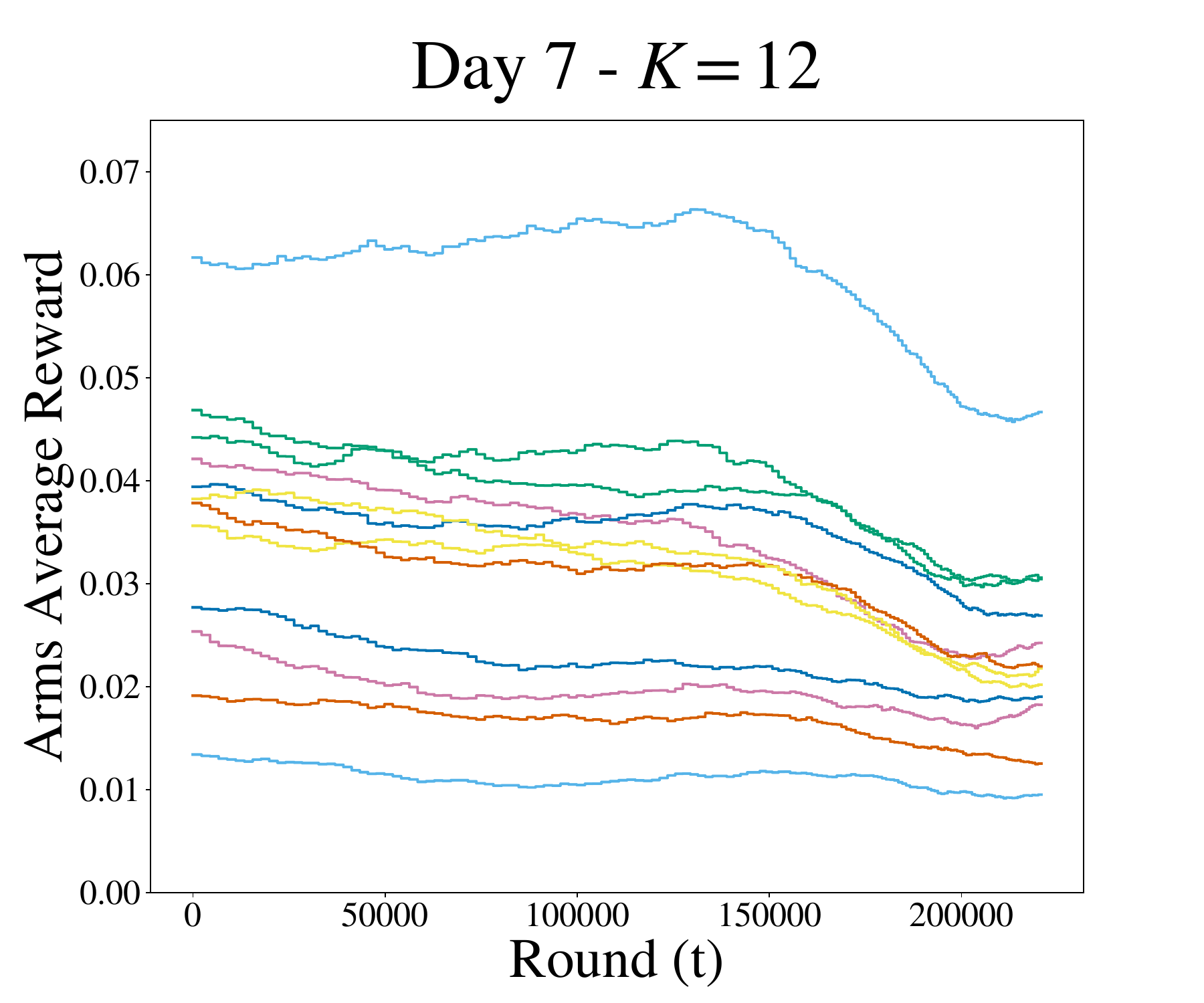}}\quad
{\includegraphics[width=.22\textwidth]{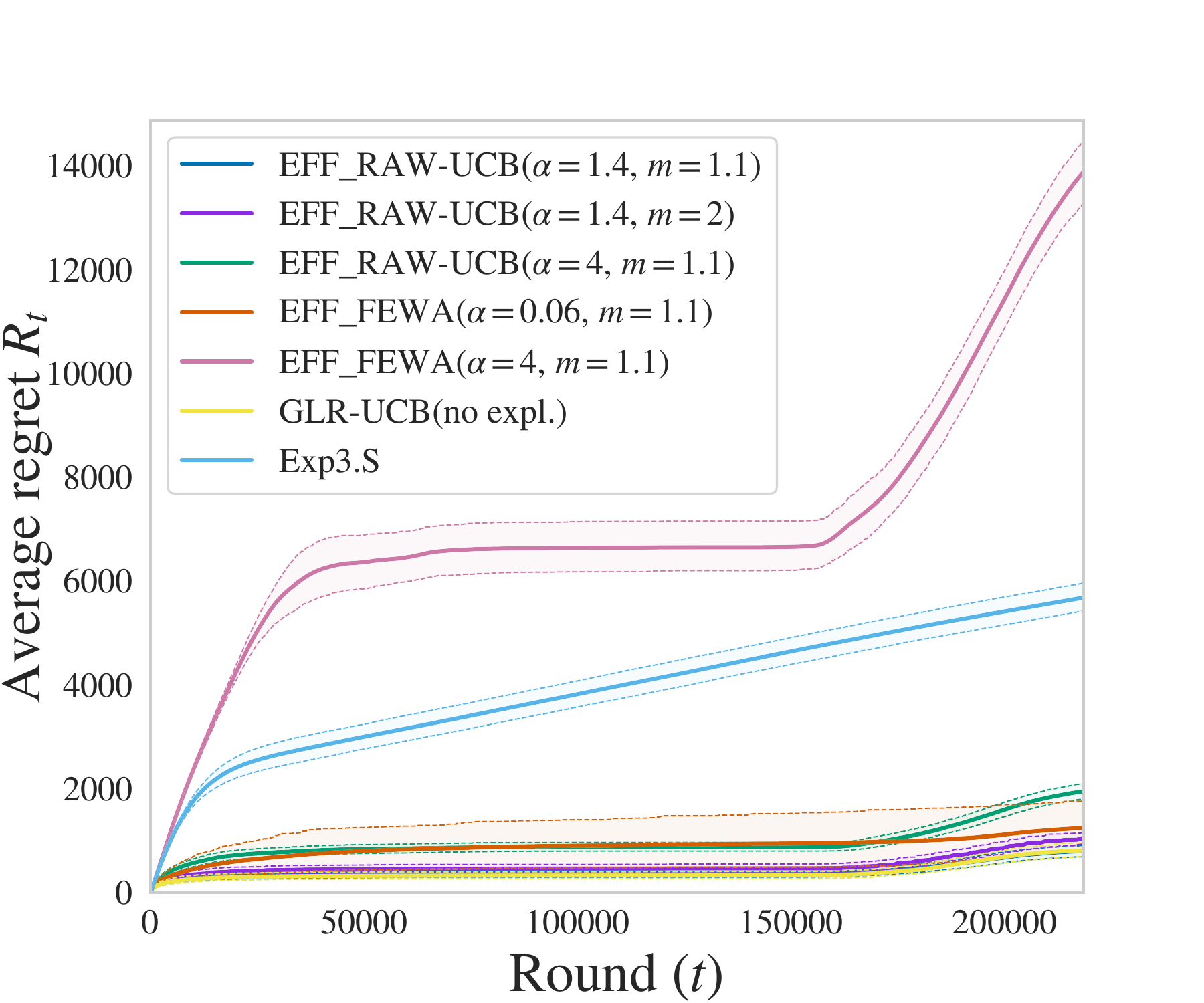}}

\caption{\textit{Left:} rewards from the Yahoo! dataset for two days. \textit{Right:} average regret over 500 runs.}
\label{Fig-regret-yahoo}
\end{figure}
\paragraph{{\RAWUCB} vs {\FEWA}.} The two algorithms compute the same statistics and share most of their analysis. Yet, {\RAWUCB} consistently outperforms {\FEWA} on the full (rested and restless) benchmark. The difference between the two is even more significant in the restless case. Moreover, {\RAWUCB} is also simpler to implement and faster to run. Its theoretical tuning $\alpha = 4$ gets reasonable result, while theoretical {\FEWA} is impractical. Finally, its empirical tuning $\alpha_{\mathrm{R}} =1.4$ is similar to the asymptotic optimal tuning of {\UCB} and shows good performance on both rested and restless problems. By contrast, {\FEWA} with $\alpha_{\mathrm{F}} = 0.06$ shows worse performance with larger deviation on the restless benchmark. 
\paragraph{{\RAWUCB} vs {\EXPS}.} In Appendix~\ref{app:rested-sim}, we show that random exploration of {\EXPS} leads to high regret rate in rested rotting bandits. Unsurprisingly, {\EXPS} recover more reasonable performance on the restless benchmark, on which it has theoretical guarantees. Yet, it is consistently outperformed by {\RAWUCB} when we tune the confidence bounds. It is particularly true on easy instance, e.g. on day 7. Indeed, on these cases, we expect logarithmic regret rate for {\RAWUCB}.
\paragraph{{\RAWUCB} vs {\GLRUCB} (no active exploration).} {\GLRUCB} shows good results on the rested benchmark though it is less consistent than {\RAWUCB}. On the restless benchmark, {\GLRUCB} shows similar result than {\RAWUCB}. Yet, we highlight that 1) {\GLRUCB} needs the knowledge of the horizon to tune its change-detector; 2) we use an efficient version of {\RAWUCB} which runs $\sim 10$ times faster than {\GLRUCB}. In fact, the two algorithms are similar: they are UCB index policies, they recover logarithmic rate on easy restless rotting bandits problems and hence they would both suffer near-linear worst case regret rate in the general restless setting (when active exploration is turned off for {\GLRUCB}). The main difference is that {\RAWUCB} scans its history to select its rotting UCB's window, while {\GLRUCB} scans its history to detect significant changes and restart. 

\paragraph{Acknowledgements} 
We thank Lilian Besson for hosting and reviewing our numerical experiments code on his bandits package in python \citep{SMPyBandits}. The computational experiments were conducted using the Grid'5000 experimental testbed \citep{grid5000}.

\bibliography{PhD_Thesis_Julien_Seznec}

\begin{thebibliography}{}

\bibitem[Audibert and Bubeck, 2009]{audibert2009minimax}
Audibert, J.-Y. and Bubeck, S. (2009).
\newblock {Minimax policies for adversarial and stochastic bandits}.
\newblock In {\em Proceedings of the Conference on Learning Theory (COLT),
  2009}, pages 217--226.

\bibitem[Auer et~al., 2002]{auer2002finite}
Auer, P., Cesa-Bianchi, N., and Fischer, P. (2002).
\newblock {Finite-time analysis of the multiarmed bandit problem}.
\newblock {\em Machine Learning}, 47(2-3):235--256.

\bibitem[Auer et~al., 2003]{auer2002nonstochastic}
Auer, P., Cesa-Bianchi, N., Freund, Y., and Schapire, R.~E. (2003).
\newblock {The nonstochastic multiarmed bandit problem}.
\newblock {\em SIAM Journal on Computing}, 32(1):48--77.

\bibitem[Auer et~al., 2019]{auer2019adaptively}
Auer, P., Gajane, P., and Ortner, R. (2019).
\newblock {Adaptively Tracking the Best Bandit Arm with an Unknown Number of
  Distribution Changes}.
\newblock In {\em Proceedings of the Conference on Learning Theory (COLT)},
  pages 138--158.

\bibitem[Balouek et~al., 2013]{grid5000}
Balouek, D., Amarie, A.~C., Charrier, G., Desprez, F., Jeannot, E., Jeanvoine,
  E., L{\`{e}}bre, A., Margery, D., Niclausse, N., Nussbaum, L., Richard, O.,
  Perez, C., Quesnel, F., Rohr, C., and Sarzyniec, L. (2013).
\newblock {Adding Virtualization Capabilities to the Grid'5000 Testbed}.
\newblock In {\em Communications in Computer and Information Science}, volume
  367 CCIS, pages 3--20. Springer Verlag.

\bibitem[Besbes et~al., 2014]{besbes2014stochastic}
Besbes, O., Gur, Y., and Zeevi, A. (2014).
\newblock {Stochastic Multi-Armed-Bandit Problem with Non-stationary Rewards}.
\newblock In {\em Advances in Neural Information Processing Systems (NIPS)},
  pages 199--207.

\bibitem[Besson, 2018]{SMPyBandits}
Besson, L. (2018).
\newblock {SMPyBandits: an Open-Source Research Framework for Single and
  Multi-Players Multi-Arms Bandits (MAB) Algorithms in Python}.
\newblock Online at: \url{https://GitHub.com/SMPyBandits/SMPyBandits}.

\bibitem[Besson and Kaufmann, 2018]{besson2018doubling}
Besson, L. and Kaufmann, E. (2018).
\newblock {What Doubling Tricks Can and Can't Do for Multi-Armed Bandits}.
\newblock arXiv preprint arXiv:1803.06971.

\bibitem[Besson and Kaufmann, 2019]{besson2019generalized}
Besson, L. and Kaufmann, E. (2019).
\newblock {The Generalized Likelihood Ratio Test meets klUCB: an Improved
  Algorithm for Piece-Wise Non-Stationary Bandits}.
\newblock arXiv preprint arXiv:1902.01575.

\bibitem[Bifet and Gavald{\`{a}}, 2007]{bifet2007learning}
Bifet, A. and Gavald{\`{a}}, R. (2007).
\newblock {Learning from time-changing data with adaptive windowing}.
\newblock In {\em Proceedings of the 7th SIAM International Conference on Data
  Mining}, pages 443--448.

\bibitem[Cao et~al., 2019]{cao2019nearly}
Cao, Y., Wen, Z., Kveton, B., and Xie, Y. (2019).
\newblock {Nearly Optimal Adaptive Procedure with Change Detection for
  Piecewise-Stationary Bandit}.
\newblock In {\em Proceedings of the International Conference on Artificial
  Intelligence and Statistics (AISTATS)}, pages 418--427.

\bibitem[Chapelle and Li, 2011]{chapelle2011empirical}
Chapelle, O. and Li, L. (2011).
\newblock {An Empirical Evaluation of Thompson Sampling}.
\newblock In {\em Advances in Neural Information Processing Systems (NIPS)},
  pages 2249--2257.

\bibitem[Chen et~al., 2019]{chen2019new}
Chen, Y., Lee, C.-W., Luo, H., and Wei, C.-Y. (2019).
\newblock {A New Algorithm for Non-stationary Contextual Bandits: Efficient,
  Optimal and Parameter-free}.
\newblock In {\em Proceedings of the Conference on Learning Theory (COLT)},
  pages 696--726.

\bibitem[Cheung et~al., 2019]{cheung2019new}
Cheung, W.~C., Simchi-Levi, D., and Zhu, R. (2019).
\newblock {Learning to Optimize under Non-Stationarity}.
\newblock In {\em Proceedings of the International Conference on Artificial
  Intelligence and Statistics (AISTATS)}, pages 1079--1087.

\bibitem[Chow and Teicher, 1997]{chow1997probability}
Chow, Y.~S. and Teicher, H. (1997).
\newblock {\em {Probability theory : independence, interchangeability,
  martingales}}.
\newblock Springer.

\bibitem[Garivier et~al., 2019]{garivier2018explore}
Garivier, A., M{\'{e}}nard, P., and Stoltz, G. (2019).
\newblock {Explore first, exploit next: The true shape of regret in bandit
  problems}.
\newblock {\em Mathematics of Operations Research}, 44(2):377--399.

\bibitem[Garivier and Moulines, 2011]{garivier2011upper-confidence-bound}
Garivier, A. and Moulines, E. (2011).
\newblock {On upper-confidence bound policies for switching bandit problems}.
\newblock In {\em Proceedings of the 22nd International Conference on
  Algorithmic Learning Theory (ALT), 2011, Espoo, Finland.}, volume 6925 LNAI,
  pages 174--188. Springer, Berlin, Heidelberg.

\bibitem[Heidari et~al., 2016]{heidari2016tight}
Heidari, H., Kearns, M., and Roth, A. (2016).
\newblock {Tight Policy Regret Bounds for Improving and Decaying Bandits}.
\newblock {\em Proceedings of the International Joint Conference on Artificial
  Intelligence (IJCAI)}, pages 1562--1570.

\bibitem[Immorlica and Kleinberg, 2018]{Immorlica2018}
Immorlica, N. and Kleinberg, R. (2018).
\newblock {Recharging bandits}.
\newblock In {\em Proceedings - Annual IEEE Symposium on Foundations of
  Computer Science, FOCS}, volume 2018-Octob, pages 309--319. IEEE Computer
  Society.

\bibitem[Komiyama and Qin, 2014]{komiyama2014time-decaying}
Komiyama, J. and Qin, T. (2014).
\newblock {Time-Decaying Bandits for Non-stationary Systems}.
\newblock In Liu, T.-Y., Qi, Q., and Ye, Y., editors, {\em Web and Internet
  Economics (WINE)}, pages 460--466, Cham. Springer International Publishing.

\bibitem[Lai and Robbins, 1985]{lai1985asymptotically}
Lai, T.~L. and Robbins, H. (1985).
\newblock {Asymptotically efficient adaptive allocation rules}.
\newblock {\em Advances in Applied Mathematics}, 6(1):4--22.

\bibitem[Lattimore and Szepesv{\'{a}}ri, 2020]{lattimore2020banditbook}
Lattimore, T. and Szepesv{\'{a}}ri, C. (2020).
\newblock {\em {Bandit Algorithms}}.
\newblock Cambridge University Press UK.

\bibitem[Levine et~al., 2017]{levine2017rotting}
Levine, N., Crammer, K., and Mannor, S. (2017).
\newblock {Rotting Bandits}.
\newblock In {\em Advances in Neural Information Processing Systems (NIPS)},
  pages 3074--3083.

\bibitem[Liu et~al., 2018]{liu2018change-detection}
Liu, F., Lee, J., and Shroff, N.~B. (2018).
\newblock A change-detection based framework for piecewise-stationary
  multi-armed bandit problem.
\newblock In {\em Proceedings of the AAAI Conference on Artificial Intelligence
  (AAAI)}, pages 3651--3658.

\bibitem[Lou{\"{e}}dec et~al., 2016]{louedec2016algorithme}
Lou{\"{e}}dec, J., Rossi, L., Chevalier, M., Garivier, A., and Mothe, J.
  (2016).
\newblock {Algorithme de bandit et obsolescence : un mod{\`{e}}le pour la
  recommandation (regular paper)}.
\newblock In {\em Conf{\'{e}}rence francophone sur l'Apprentissage Automatique,
  Marseille, 05/07/2016-07/07/2016}, page (en ligne),
  http://www.lif.univ-mrs.fr. Laboratoire d'Informatique Fondamentale de
  Marseille.

\bibitem[Mukherjee and Maillard, 2019]{mukherjee2019distribution}
Mukherjee, S. and Maillard, O.-A. (2019).
\newblock {Distribution-dependent and Time-uniform Bounds for Piecewise i.i.d
  Bandits}.
\newblock arXiv preprint arXiv:1905.13159.

\bibitem[Pike-Burke and Grunewalder, 2019]{pikeburke2019recovering}
Pike-Burke, C. and Grunewalder, S. (2019).
\newblock {Recovering Bandits}.
\newblock In {\em Advances in Neural Information Processing Systems (NeurIPS)},
  pages 14122--14131.

\bibitem[Russac et~al., 2019]{russac2019weighted}
Russac, Y., Vernade, C., and Capp{\'{e}}, O. (2019).
\newblock {Weighted Linear Bandits for Non-Stationary Environments}.
\newblock In {\em Advances in Neural Information Processing Systems (NeurIPS)},
  pages 12040--12049.

\bibitem[Seznec et~al., 2019]{seznec2019rotting}
Seznec, J., Locatelli, A., Carpentier, A., Lazaric, A., and Valko, M. (2019).
\newblock {Rotting bandits are no harder than stochastic ones}.
\newblock In {\em Proceedings of the International Conference on Artificial
  Intelligence and Statistics (AISTATS)}, pages 2564--2572.

\bibitem[Thompson, 1933]{thompson1933likelihood}
Thompson, W.~R. (1933).
\newblock {On the Likelihood that One Unknown Probability Exceeds Another in
  View of the Evidence of Two Samples}.
\newblock {\em Biometrika}, 25(3/4):285--294.

\bibitem[Trac{\`{a}} and Rudin, 2015]{traca2015regulating}
Trac{\`{a}}, S. and Rudin, C. (2015).
\newblock {Regulating Greed Over Time}.
\newblock arXiv preprint arXiv:1505.05629.

\bibitem[Warlop et~al., 2018]{warlop2018fighting}
Warlop, R., Lazaric, A., and Mary, J. (2018).
\newblock {Fighting Boredom in Recommender Systems with Linear Reinforcement
  Learning}.
\newblock In {\em Advances in Neural Information Processing Systems (NeurIPS)},
  pages 1757--1768.

\bibitem[Wei et~al., 2016]{wei2016tracking}
Wei, C.-Y., Hong, Y.-T., and Lu, C.-J. (2016).
\newblock {Tracking the Best Expert in Non-stationary Stochastic Environments}.
\newblock In {\em Advances in Neural Information Processing Systems (NIPS)},
  pages 3972--3980.

\bibitem[Whittle, 1980]{whittle1980multi}
Whittle, P. (1980).
\newblock {Multi-Armed Bandits and the Gittins Index}.
\newblock {\em Journal of the Royal Statistical Society. Series B
  (Methodological)}, 42:143--149.

\bibitem[Whittle, 1988]{whittle1988restless}
Whittle, P. (1988).
\newblock {Restless bandits: activity allocation in a changing world}.
\newblock {\em Journal of Applied Probability}, 25(A):287--298.

\bibitem[Zimmert and Seldin, 2019]{zimmert2019optimal}
Zimmert, J. and Seldin, Y. (2019).
\newblock {An Optimal Algorithm for Stochastic and Adversarial Bandits}.
\newblock In {\em Proceedings of the International Conference on Artificial
  Intelligence and Statistics (AISTATS)}, pages 467--475.

\end{thebibliography}
\bibliographystyle{apalike}

\newpage
\appendix
\onecolumn
\section{Outline}\label{app:outline}

The appendix of this paper is organized as follow:
\begin{itemize}[label=$\square$]
    \item Appendix~\ref{app:unlearnable} is dedicated to the unlearnability of the general decreasing setup.
    \item Appendix~\ref{app:proof_fundamental_lemma} is dedicated to Proposition~\ref{prop:prb_favorable_event} and Lemma~\ref{lem:core-RAWUCB}. For completion, we also restate a similar Lemma about algorithm \FEWA \citep{seznec2019rotting} in the rested and restless rotting framework.
    \item Appendix~\ref{app:efficient_alg} is dedicated to an efficient version of \RAWUCB.
    \item Appendix~\ref{app:restless} provides the analysis of \RAWUCB for restless rotting bandits.
    \item Appendix~\ref{app:rested-theory} provides the analysis of \RAWUCB for rested rotting bandits.
    \item Appendix~\ref{app:experiments} provides all the experiments.
\end{itemize}
\section{The general decreasing setup is unlearnable}
\label{app:unlearnable}
\restaunlearnableprop*
\begin{proof}
Let $\mu^{0}$ and $\mu^{1}$, two decreasing 2-arms bandits problems such that:
\begin{align*}
 &\mu^{0}_1(t,n) = \mu_1(n) = 1 \text{ if } n<\frac{T}{2} \text{ else } 0\,,\\
 &\mu^{1}_1(t,n) = 1\,, \\
 &\mu^{0}_2(t,n) = \mu^{1}_2(t,n) = \mu_2(t) = 1/2 \text{ if } t<\frac{T}{2} \text{ else } 0.
\end{align*}
Problem $\mu^{1}$ only evolves according to time. Hence, the oracle greedy policy $\pi_O$ is optimal for this problem and collects
\begin{equation}
\label{eq:regret1-piO}
J_T\pa{\pi_O, \mu^1} = T.
\end{equation}
On $\mu^{0}$, $\pi_O$ selects arm $1$ during $\floor{\frac{T}{2}}$ rounds and then both arms yield $0$ reward. Thus, $\pi_O$ collects 
\[J_T\pa{\pi_O, \mu^0} = \floor{\frac{T}{2}}.\]
However, let $\pi_0$ the policy which selects arm 2 for $\floor{\frac{T}{2}}$ rounds and arm 1 afterwards. Thus, $\pi_0$ collects
\begin{equation}
\label{eq:regret0-pi0}
  J_T\pa{\pi_0, \mu^0} = \pa{3/2} \floor{\frac{T}{2}}.  
\end{equation}
Hence, we conclude the first part of our proposition, 
\[R_T\pa{\pi_O, \mu^0} = J_T\pa{\pi^\star_T, \mu^0} - J_T\pa{\pi_O, \mu^0} \geq J_T\pa{\pi_0, \mu^0} - J_T\pa{\pi_O, \mu^0} \geq  \floor{\frac{T}{4}}.\]
Now, we consider any learning policy $\pi_S$ and we call $\EEempty_j\big[N_{i,t}(\pi_S)\big]$ the (expected, if the policy is random) number of pulls of arm $i$ at round $t$ by $\pi_S$ on problem $j$. Note that the leaner will receive the same rewards for both problems until at least $\floor{\frac{T}{2}}$. Therefore, we have that 

\[ \forall t \leq \floor{\frac{T}{2}}, \pi\big(\mathcal{H}_t\pa{\mu^0}\big) = \pi\big(\mathcal{H}_t\pa{\mu^1}\big) \implies \EEempty_0\Big[N_{2,\floor{\frac{T}{2}}}(\pi_S)\Big] = \EEempty_1\Big[N_{2,\floor{\frac{T}{2}}}(\pi_S)\Big] \triangleq n_2.\]

On problem $\mu^{1}$, $\pi_S$ collects a reward of at most,
\begin{equation}
\label{eq:regret1-piS}
    J_T\pa{\pi_S, \mu^1} = \EEempty_1[N_{1,T}(\pi_S)] + \frac{n_2}{2} = T - \EEempty_1[N_{2,T}(\pi_S)] + \frac{n_2}{2} \leq T - \frac{n_2}{2}\CommaBin
\end{equation}
because $n_2 = \EEempty_1\Big[N_{2,\floor{\frac{T}{2}}}(\pi_S)\Big] \leq \EEempty_1[N_{2,T}(\pi_S)]$. Using Equations~\ref{eq:regret1-piO} and~\ref{eq:regret1-piS}, we can lower bound the regret of $\pi_S$, 
\[ R_T\pa{\pi_S, \mu^1} = J_T\pa{\pi_O, \mu^1} - J_T\pa{\pi_S, \mu^1} \geq  \frac{n_2}{2}\cdot \]

On problem $\mu^{0}$, $\pi_S$ collects a reward of at most,
\begin{equation}
\label{eq:regret0-piS}
    J_T\pa{\pi_S, \mu^0} = \min\pa{\EEempty_1[N_{1,T}(\pi_S)],\floor{\frac{T}{2}}} + \frac{n_2}{2} \leq \floor{\frac{T}{2}} + \frac{n_2}{2}\cdot
\end{equation}
Using Equations~\ref{eq:regret0-pi0} and~\ref{eq:regret0-piS}, we can lower bound the regret of $\pi_S$, 
\[ R_T\pa{\pi_S, \mu^0} = J_T\pa{\pi_O, \mu^0} - J_T\pa{\pi_S, \mu^0} \geq  \frac{\floor{T/2} - n_2}{2}\cdot \]

Hence, the worst case regret on the two setups is bounded by 
\[R_T(\pi_S) \geq \max\pa{\frac{n_2}{2}, \frac{\floor{\frac{T}{2}} - n_2}{2}} \geq \floor{\frac{T}{8}}\!\cdot \]

\end{proof}

\section{Statistical guarantees: Proposition~\ref{prop:prb_favorable_event} and Lemma~\ref{lem:core-RAWUCB}}
\label{app:proof_fundamental_lemma}
\restachernoff*
\begin{proof}
We want to upper bound the probability
\[
\PP{\bar{\HPevent}} = \PP{\exists i \in K,\,\exists n \leq t-1, \exists h \leq\, n, \big|\hmu^h_i(t, \pi)-\bmu^h_i(t, \pi)\big|>c(h,\delta_t) }.
\]

By Doob's optional skipping (e.g. see \citet{chow1997probability}, Section 5.3) there exists a sequence of random independent variables $(\epsilon'_l)_{l\in\NN}$, $\sigma^2$ sub-Gaussian such that 
\begin{align*}
    \PPv\Big[\exists n \leq t-1, \exists h &\leq\, n, \big|\hmu^h_i(t, \pi)-\bmu^h_i(t, \pi)\big|>c(h,\delta_t) \Big]\\
    &= \PP{\,\exists n \leq t-1, \exists h \leq n,\, |\hepsilon^h_n|>c(h,\delta_t) }\\
    &\leq \sum_{n=1}^{t-1}\sum_{h=1}^n \PP{|\hepsilon^h_n|>c(h,\delta_t)} \\
    &\leq  \frac{t(t-1)}{2}\cdot \delta_t \,,
\end{align*}
where we used the Chernoff inequality in the last line. Thus, a union bound  over the arms allows us to conclude that
\[
\PPempty \Big[\bar{\HPevent}\Big]\leq \frac{K \delta_t t^2}{2}\cdot
\]
\end{proof}
\restafundamentallemma*
\begin{proof}
We denote by $
i^\star_t \in \argmax_{i\in \possibleArms}{\mu_i(t,N_{i,t-1})}
$, a best available arm at time $t$ and 
\[
h_{i,t}^{\min} \in \argmin_{h\leq N_{i,t-1}}{\hat{\mu}_i^h(t, \pi) + c(h,\delta_t)},
\]
a window which minimizes \RAWUCB index at time $t$ for arm $i$. Hence, because the reward functions are non-increasing, we know that 
\begin{equation*}
 \mu_{i^\star_t}(t, N_{i^\star_t,t-1}) \leq   \bar{\mu}_{i^\star_t}^1(t, \pi) \leq \cdots \leq  \bar{\mu}_{i^\star_t}^{h_{i^\star_t,t}^{\min}}(t, \pi)\cdot
\end{equation*}
On the high-probability event $\xi_t$, we know that the true average of the means cannot deviate significantly from the average of the observed quantity,
\begin{equation*}
\bar{\mu}_{i^\star_t}^{h_{i^\star_t,t}^{\min}}(t, \pi) \leq \hat{\mu}_{i^\star_t}^{h_{i^\star_t,t}^{\min}}(t, \pi) + c(h_{i^\star_t,t}^{\min},\delta_t).
\end{equation*}
We know that the selected arm $i_t$ at time $t$ has the largest index, hence, 
\[
\hat{\mu}_{i^\star_t}^{h^{\min}_{i^\star_t,t}}(t, \pi) + c(h^{\min}_{i^\star_t,t},\delta_t) \leq \hat{\mu}_{i_t}^{h^{\min}_{i_t,t}}(t, \pi) + c(h^{\min}_{i_t,t},\delta_t).
\]
From $h_{i,t}^{\min}$ definition, we know that this quantity is below any upper-confidence bound for any other window $h$
\[
\hat{\mu}_{i_t}^{h^{\min}_{i_t,t}}(t, \pi) + c(h^{\min}_{i_t,t},\delta_t) \leq \hat{\mu}_{i_t}^{h}(t, \pi) + c(h,\delta_t).
\]
Finally, using again the concentration of the average on the $\HPevent$, 
\[
\hat{\mu}_{i_t}^{h}(t, \pi) + c(h,\delta_t) \leq \bar{\mu}_{i_t}^{h}(t, \pi) + 2c(h,\delta_t).
\]
Hence, putting all the equations together, we can write
\begin{equation*}
\bar{\mu}_{i_t}^{h}(t, \pi) \geq \max_{i \in \possibleArms} \mu_{i}(t,N_{i,t-1}) - 2 c(\window, \delta_t).
\end{equation*}
\end{proof}

For completion, we also restate a similar Lemma about algorithm \FEWA \citep{seznec2019rotting} in the rested and restless rotting framework.
\begin{restatable}{lemma}{restalemmafewa}
\label{lem:core-FEWA}
For {\FEWA} tuned with $\alpha$, on the favorable event $\HPevent$, if an arm~$i$ passes through a filter of window $h$ at round $t$, i.e., $i\in\ \arms_h$, then the average of its $h$ last pulls satisfies
\begin{equation}\label{eq:fundamental-eq-FEWA}
\bmu^{h}_i(t, \piF) \geq  \max_{i \in \arms} \mu_i(t, \Nitmone) - 4 c(h, \delta_t).
\end{equation}
Therefore, at round $t$ on favorable event $\HPevent$, if arm~$i_{t}$ is selected by {\FEWA($\alpha$)}, for any $h \leq \Nitmone$,  the average of its $h$ last pulls cannot deviate significantly from the best available arm at that round, i.e.,
\begin{equation*}
\bmu^{h}_i(t, \piF) \geq \max_{i \in \arms} \mu_i(t, \Nitmone) - 4 c(h, \delta_{t}).
\end{equation*}
\end{restatable}

\begin{proof}
Let $i \in \arms_h$ be an arm that passed a filter of window $h$ at round $t$.
First, we use the confidence bound for the estimates and we pay the cost of keeping all the arms up to a distance $2c(h,  \delta_t)$ of $\hmu^{h}_{\max,\, t} \triangleq \max_{j \in \arms_h} \hmu_i^h(t,\piF)$,
\begin{equation}
\label{1}
\bmu_i^h(t,\piF)\geq \hmu_i^h(t,\piF)- c(h,  \delta_t) \geq \hmu^h_{\max,t} - 3c(h, \delta_t)
\geq \max_{j \in \arms_h}\bmu_j^h(t,\piF)  - 4 c(h, \delta_t),
\end{equation}
where in the last inequality, we used that for all $j \in \arms_h,$ \[\hmu^{h}_{\max,t} \geq  \hmu_j^h(t,\piF)  \geq \bmu_j^h(t,\piF)  - c(h, \delta_t).\]
Second, we call $t_{i,t} < t$ the last round at which arm $i$ was selected. Since the means of arms are decaying, we know that 
\begin{align}
\label{2}
 \mu^+_t(\piF) &\triangleq \max_{i \in \arms} \mu_{i}(t, \Nitmone) \nonumber\\
 &\leq  \mu_{i^\star_t, \, t_{i,t}} =  \bmu_i^1(t,\piF)  \nonumber\\
 &\leq \max_{j \in \arms}\bmu_j^1(t,\piF)  = \max_{j \in \arms_1} \bmu_j^1(t,\piF).
\end{align}
Third, we show that the largest average of the last $h'$ means of arms in $\arms_{h'}$ is increasing with~$h'$,
\begin{equation*}
\forall  h' \leq h ,  \max_{j \in \arms_{h'+1}}\bmu_j^{h'+1}(t,\piF)   \geq \max_{j \in \arms_{h'}}\bmu_j^{h'}(t,\piF). 
\end{equation*}
To show the above property, we remark that thanks to our selection rule, the arm that has the largest average of means, always passes the filter. Formally, we show that $\argmax_{j \in \arms_{h'}}\bmu_j^{h'}(t,\piF) \subseteq \arms_{h'+1}.$ 
Let $i^{h'}_{\max} \in \argmax_{j \in \arms_{h'}}\bmu_j^{h'}(t,\piF)$. Then, for such $i^{h'}_{\max}$, we have
\begin{equation*}
\hmu_{i^{h'}_{\max}}^{h'}(t,\piF) \geq \bmu_{i^{h'}_{\max}}^{h'}(t,\piF) - c(h', \delta_t) 
\geq \bmu^{h'}_{\max,t} - c(h', \delta_t) \geq \hmu^{h'}_{\max,t}- 2c(h', \delta_t),
\end{equation*}
where the first and the third inequality are due to concentration of the estimates on $\HPevent$, while the second one is due to the definition of $i^{h'}_{\max}$. 

Since the arms are decaying, the average of the last $h' +1$ mean values for a given arm is always greater than the average of the last $h'$ mean values
and therefore, 
\begin{equation}
\label{3}
 \max_{j \in \arms_{h'}}\bmu_j^{h'}(t,\piF) =   \bmu_{i^{h'}_{\max}}^{h'}(t,\piF) \leq \bmu_{i^{h'}_{\max}}^{h'+1}(t,\piF) \leq \max_{j \in \arms_{h'+1}}\bmu^{h' +1 }_{j}(t,\piF), 
\end{equation}
because $i^{h'}_{\rm max} \in \arms_{h'+1}$. Gathering Equations~\ref{1}, \ref{2}, and~\ref{3} leads to the first claim of the lemma,
\begin{align*}
\bmu^{h}_i(t,\piF)
&\stackrel{\eqref{1}}{\geq} \max_{j \in \arms_h}\bmu^{h}_{j}(t,\piF) - 4c(h, \delta_t)\\
&\stackrel{\eqref{3}}{\geq} \max_{j \in \arms_1}\bmu^{1}_{j}(t,\piF)- 4c(h,  \delta_t)\\
&\stackrel{\eqref{2}}{\geq}  \mu^+_t(\piF) - 4c(h, \delta_t).
\end{align*}
To conclude, we remark that if $i$ is pulled at round $t$, it means that $i$ passes through all the filters from $h=1$ up to $\Nitmone$. Therefore, for all $h\leq \Nitmone$,
\begin{equation}
\bmu^{h}_i(t,\piF) \geq  \mu^+_t(\piF) - 4 c(h, \delta_t).
\end{equation}
\end{proof}
\section{Efficient algorithms}
\label{app:efficient_alg}
\subsection{The numerical cost of adaptive windows}

\cite{seznec2019rotting} highlight that \FEWA was significantly improving over state-of-the-art algorithms on the rested rotting bandit problem but these improvements are computationally expensive. Indeed, at each round $t$, we store, update and compare $\cO\pa{t}$ statistics. \RAWUCB uses the same statistics than \FEWA (see Prop~\ref{prop:prb_favorable_event}), and thus has the same complexity.

Indeed, the full update of the statistics can be done at a worst case cost of $\cO\pa{t}$. Indeed, each statistics $\hmu_i^h$ can be refreshed with a $\cO\pa{1}$ operation : 
\[\hmu_i^{h+1}(n+1) = \frac{h}{h+1}\hmu_i^{h}(n) + \frac{1}{h+1}o_t \,. \]

The comparison part in both \FEWA and \RUCB is also a $\cO\pa{t}$ operations. In \FEWA , we do a scan based on $\hmu_i^{h}$ for all $i \in \arms_h$ with increasing $h$. Hence, the total number of unitary operation is in $\cO\pa{t}$ in the worst case, as it scales with the number of statistics. \RUCB computes one UCB for each of the $\cO\pa{t}$ statistics. For each arm, it selects the minimum UCB as index, which can be done with complexity $\cO\pa{t}$. Finally, finding the largest index is an $\cO\pa{K}$ operations. Therefore, we can conclude,

\begin{proposition}
\FEWA and \RUCB have a $\cO\pa{t}$ worst-case complexity per round $t$ in time and memory.
\end{proposition}

Hence, handling a large number of windows, which is the main strength of these algorithms to achieve a lower regret, is a significant drawback when it comes to design fast algorithms. In the following, we detail and refine the efficient trick of \citet{seznec2019rotting} by adding a parameter $m\leq 2$ which trades-off between regret and computational performance.  

\subsection{The efficient update trick}
We detail \EFF, an update scheme to handle efficiently statistics of different windows. A similar yet different approach has appeared independently in the context of streaming mining~\citep{bifet2007learning}. \EFF is built around two main ideas.

First, at any time $t$ we can avoid using $\left\{\hmu_i^h\right\}_{h}$ for all possible windows $h$ starting from 1 with an increment of 1. In fact, both statistics $\hmu_i^h$ and constructed confidence levels $c(h, \delta_t)$  have very close value for successive $h$ as $h$ becomes large : 
\begin{align*}
& \hmu_i^{h+1}(t , \pi) = \hmu_i^{h}(t, \pi) + \cO\pa{\frac{\sigma + L}{h}}\,,\\
& c(h+1, \delta_t) = c(h, \delta_t) + \cO\pa{\frac{\sigma }{h^{3/2}}} \,.
\end{align*}
Hence, in both \FEWA and \RUCB, we compute a lot of very similar quantities. Instead, we could use fewer statistics which are significantly different : $\left\{\hmu_i^h(\Nitmone)\right\}_{h\in \Him}$, where the window $h$ is dispatched on a geometric grid, 
 \[\Him\pa{\Nitmone} \triangleq \left\{ h_j \in  \left\{1, \dots , \Nitmone \right\} \;|\; h_{j+1} = \ceil{m \cdot h_j} \text{ and } h_1 = 1\right\}\quad \text{with } m > 1.\]

When there is no confusion, we drop the dependency in $\Nitmone$ and use $\Him$.  This modification alone is not enough to reduce both the time and space complexity. Indeed, updating $\hmu^h_{i}$ requires to replace the $h$-th last sample by the new one $o_t$. Hence, we need to store all the collected statistics to be able to update all the $\hmu^h_{i}$ for all $h$ with $\cO\pa{1}$ complexity. Therefore, in \EFF, we will use $\cO\pa{K\log\pa{t}}$ \emph{delayed} statistics that we can update with $\cO\pa{K\log\pa{t}}$ space and time complexity.

\EFF (Alg.~\ref{alg:effupdate}) takes as input the new observation $o_t$ that the learner gets at the $N_i$-th pull of arm $i$; the geometric window grid $\Him$ tuned with an hyperparameter $m>1$, and for each window $h_j$ in this grid, three different numbers $\hmueff,\; \peff, \; \neff$. $\left\{\hmueff\right\}_{i,h_j}$ represents the set of \emph{current} statistics of window size $h_j$ that will be used instead of $\left\{\hmu_i^h\right\}_{i,h}$ in our efficient algorithms. We also store a pending statistic $\peff$ and a count $\neff$  which are used in the sparse update procedure of $\hmueff$. \EFF outputs an updated set of statistics.  

\begin{algorithm}
\caption{{\small\sc Eff\_Update}}
\label{alg:effupdate}
\begin{algorithmic}[1]
\Require $o_t$, \small $\Him \gets \left\{h_j \! <\! \ceil{m \cdot N_i} \; | \;  h_{j+1} \!=\! \ceil{m \cdot h_j}  \text{with } h_0 \!=\! 1\right\} $\normalsize, $\left\{ \{ \hmueff,\, \peff, \, \neff \}\right\}_{h_j \in \Him}$
\If{$N_i = \max\pa{\Him}$}\label{algline:effu-new-condition} \Comment{Create a new triplet with window $h_j = \ceil{m \cdot N_i}$}  
\State $\Him \gets \Him \cup \left\{ \ceil{m \cdot N_i} \right\}$\label{algline:effu-new-h}
\State $p_i^{\ceil{m \cdot N_i} } = p_i^{N_i} $\label{algline:effu-new-p}
\State $n_i^{\ceil{m \cdot N_i} } \gets n_i^{N_i} $\label{algline:effu-new-n}
\State $\hmu_{i, \, \tteff}^{\ceil{m \cdot N_i} }\leftarrow \texttt{None}$\label{algline:effu-new-mu}
\EndIf\label{algline:effu-new-end} 
\State $p_i^{1} \gets o_t$ \label{algline:effu-update-first-p} \Comment{Update the first triplet with $o_t$}
\State $n_i^{1} \gets 1$\label{algline:effu-update-first-n}
\State $\hmu_{i, \, \tteff}^{1}\leftarrow o_t$ \label{algline:effu-update-first-hmu}
\For{$h_j \in  \Him \smallsetminus \left\{ 1\right\} $}\label{algline:effu-update-start} \Comment{Update the other pending statistics $\peff$ and $\neff$}
\State $p_i^{h_j} \gets p_i^{h_j}  +o_t$\label{algline:effu-update-p}
\State $n_i^{h_j} \gets n_i^{h_j} + 1$\label{algline:effu-update-n}
\EndFor\label{algline:effu-update-end} 
\For{$h_j \in  $ \textsc{Sort\_Desc}$\pa{\Him \smallsetminus \left\{ 1\right\} }$}\label{algline:effu-refresh-start}
\If{$n_i^{h_j} = h_j$} \label{algline:effu-refresh-condition}
\State $\hmueff \leftarrow p_i^{h_j}/h_j$ \Comment{Replace the current statistic $\hmueff$}\label{algline:effu-refresh-hmu}
\State{$p_i^{h_{j}} = p_i^{h_{j-1}} $} \label{algline:effu-refresh-p}\Comment{Refresh the pending statistics}
\State $n_i^{h_{j}} \gets n_i^{h_{j-1}} $\label{algline:effu-refresh-n}
\EndIf
\EndFor \label{algline:effu-refresh-end}
\Ensure $\left\{\left\{  \hmueff,\; p_i^{h_j}, \; n_i^{h_j} \right\}\right\}_{h_j \in \Him}$
\end{algorithmic}
\end{algorithm}

The core of \EFF is divided in four parts: \begin{enumerate}
    \item From Lines~\ref{algline:effu-new-condition} to~\ref{algline:effu-new-end}, we create new statistics at a logarithmic rate with respect to the growth of $N_i$;
    \item From Lines~\ref{algline:effu-update-first-p} to~\ref{algline:effu-update-first-hmu}, we update the statistics of window $h_1=1$;
    \item From Lines~\ref{algline:effu-update-start} to~\ref{algline:effu-update-end}, we update the other pending statistics and count;
    \item From Lines~\ref{algline:effu-refresh-start} to~\ref{algline:effu-refresh-end}, we eventually update $\hmueff$ and refresh the correspounding pending statistic and count.
\end{enumerate}
The remaining details are quite technical. Thus, we first give the high-level properties that are ensured by the recursive usage of \EFF. Then, we prove them by going through the algorithm line by line.

\begin{proposition}
\label{prop:effu}
 $\left\{\left\{  \hmueff,\; p_i^{h_j}, \; n_i^{h_j} \right\}\right\}_{h_j \in \Him}$, constructed recursively with {\EFF} with initial value $\left\{\left\{  \hmu_{i,\,\tteff}^1 : \texttt{None},\; p_i^{1} :0 , \; n_i^{1}:0 \right\}\right\}$ have the following properties :
 \begin{enumerate}[topsep=0pt]
  \item $\hmueff$ is the average of exactly $h_j$ consecutive samples among the $2h_j -1$ last ones. \label{list:effu-hmu}
  \item The delay between two updates of $\hmueff$ is in $\left\{\ceil{\frac{m-1}{m} h_j}, \dots, h_j -1\right\}$.\label{list:effu-delay}
  \item When $m = 2$, $h_j = 2^{j}$. Moreover, for $j\geq1$,  the $k$-th update $\hmueff$ happens at pull $ \pa{k+1} \cdot 2^{j-1}$, \ie every $2^{j-1}$ pulls (and at every rounds for $j=0$).\label{list:effu-m2}
  \item $\peff$ is the sum of the $\neff$ last samples. \label{list:effu-p}
  \item $\neff < h_j$ for $j\geq 1$. Also, $n_i^1 \leq 1$.\label{list:effu-n1}
  \item $\left\{ \neff \right\}_{h_j}$ is an non-decreasing sequence with respect to $h_j$ (or $j$).\label{list:effu-n2}
 \end{enumerate}
\end{proposition}
\begin{proof}

The three last properties are trivially true at the initialization. Thus, we show by induction that they remain true after updates.
\paragraph{Proof of \ref{list:effu-p}. } At Lines~\ref{algline:effu-new-p} and~\ref{algline:effu-new-n}, we create a new pending statistics and count by initializing them with other statistics and counts. Hence, because of the recursion hypothesis, all the pending statistics $\peff$ (including the created one) contains the sum of the $\neff$ \emph{before last} pulls. At Lines~\ref{algline:effu-update-first-p} and~\ref{algline:effu-update-first-n}, we update $p_i^1$ with the last sample and set $n_i^1$ to $1$. At Lines~\ref{algline:effu-update-p} and~\ref{algline:effu-update-n}, we add the last sample to $\peff$ (which was containing the before last samples) and increase the count by $1$. Hence, at the end of Line~\ref{algline:effu-update-n}, all the $\peff$ contain the sum of the last $\neff$ samples. Thus, refreshing $\peff$ and $\neff$ with $ p_i^{h_{j-1}}$ and $n_i^{h_{j-1}}$ keeps this property true (Lines~\ref{algline:effu-refresh-p} and~\ref{algline:effu-refresh-n}). 

\paragraph{Proof of \ref{list:effu-n1}.}
For $j=0$, $n_i^1$, which is equal to $0$ at the initialization, is set at $1$ at every update (Line~\ref{algline:effu-update-first-n}). Hence, we have $n_i^{h_0} \leq h_0=1$.
For $j\geq 1$, $n_i^{\ceil{m \cdot N_i}}$ is initialized at Line~\ref{algline:effu-new-n} with the value $n_i^{N_i} < N_i < \ceil{m \cdot N_i}$ by the induction hypothesis and because $m>1$.  Then, $\neff < h_j$ ($j\geq1$) is increased by one at each update at Line~\ref{algline:effu-update-n}. Hence, we now have $\neff \leq  h_j$ for all $j\in\Him$. However, for $j\geq 1$, if $\neff = h_j$ (Line~\ref{algline:effu-refresh-condition}), it is replaced by the precedent count $n_i^{h_{j-1}}\leq h_{j-1} < h_j$ (Line~\ref{algline:effu-update-n}). Thus, at the end of the update, we do have $\neff < h_j$ for $j\geq1$.

\paragraph{Proof of \ref{list:effu-n2}.}
At Line~\ref{algline:effu-new-n}, we create a new pending count corresponding to the largest $h_j$ and we initialize it with the precedent largest count. At Lines~\ref{algline:effu-update-first-n} and~\ref{algline:effu-update-n}, we set $n_i^1 =1$ and increase all the other $\neff$ by one. This operation preserves the non-decreasing property of the ordered set. Last, at Line~\ref{algline:effu-refresh-n}, we set few counts $\neff$ to the precedent value $n_i^{h_{j-1}}$- which also preserves the non-decreasing property of the ordered set. 

\paragraph{Proof of \ref{list:effu-hmu} and \ref{list:effu-delay}.}
Thanks to Property~\ref{list:effu-p}, we know that $\peff$ is the sum of the $\neff$ last sample. It is still true at the end of Line~\ref{algline:effu-update-n} (see the proof). Then, at Line~\ref{algline:effu-refresh-hmu}, and given the condition in Line~\ref{algline:effu-refresh-condition}, we set $\hmueff$ with the average of the last $h_j$ sample. Then, $\hmueff$ is not updated untill the condition at Line~\ref{algline:effu-refresh-condition} is fulfilled again. 

$\neff$ is refreshed with a quantity larger or equal to $1$ and smaller or equal to $h_{j-1}$ at Line~\ref{algline:effu-refresh-n}. Then, it is increased by one at each update. we know that $\hmueff$ will be updated at least every $h_j-1$, and at most every $h_j -h_{j-1}$ round. Hence, considering the worst possible delay we can conclude : $\hmueff$ is the average of exactly $h_j$ consecutive samples among the $2h_j -1$ last ones. Last, considering that $h_{j-1}\leq h_j /m$, we conclude that the minimal delay is larger or equal to $\frac{m-1}{m}h_j$.

\paragraph{Proof of \ref{list:effu-m2}.}
When $m=2$, it is easy to find by induction that,
 \[
 h_{j+1} = \ceil{m\cdot h_j} = 2h_j = 2^{j+1}.
 \]
For $j=0$, $\hmu_{i,\,\tteff}^1$ is updated at every update at Line~\ref{algline:effu-update-first-hmu}.
By induction on $j\geq 1$, $\hmueff$ is initialized (Line~\ref{algline:effu-refresh-hmu}) for the first time after $h_j= 2^{j} = 4 \cdot 2^{j-2}$ pulls. Therefore, it is also an updating pull for $\hmu_{i,\,\tteff}^{h_{j-1}}$ (by the induction hypothesis) and $n_j$ is set with $n_{j-1} = 2^{j-1}$ at Line~\ref{algline:effu-refresh-n}. Notice that we sort $\Him$ in the decreasing order at Line~\ref{algline:effu-refresh-start}, hence $n_j$ is updated with $n_{j-1}$ before it is itself updated with $n_{j-2}$.  Hence, $\hmueff$ is updated again in $h_j - 2^{j-1} = 2^{j-1}$ pulls, \ie after $6 \cdot 2^{j-2}$ pulls of arm $i$. Again, $n_j$ is set with $n_{j-1} = 2^{j-1}$ (because it is an updating pull for $\hmu_{i,\,\tteff}^{h_{j-1}}$). By induction, we see that the $k$-th update happens at pull $ \pa{k+1} \cdot 2^{j-1}$, \ie every $2^{j-1}$ pulls.

\end{proof}

\begin{remark}
At Line~\ref{algline:effu-refresh-n}, we refresh $\neff$ with $n_i^{h_{j-1}}$ which is often larger than $1$. Indeed, we could refresh $\peff$ and $\neff$ at $0$. Yet, in order to reduce the delay in the update, we use the variable available in the memory which contains the sums of $h$ last sample, with the largest $h< h_j$. According to Properties~\ref{list:effu-p},~\ref{list:effu-n1} and~\ref{list:effu-n2}, this quantity is $p_i^{h_{j-1}}$. 

Notice that we also sort $\Him$ in the decreasing order at Line~\ref{algline:effu-refresh-start} to minimize the delay: if there is two consecutive updates of $\hmueff$ and $\hmu_{i,\, \tteff}^{h_{j+1}}$ at the same run of {\EFF}, doing a backward loop guarantees to refresh $n_i^{h_{j+1}}$ with a larger value than with a forward loop. 
\end{remark}

\subsection{{\EFFFEWA} and {\EFFRAW}}
{\EFFFEWA} ($\piEF$) and {\EFFRAW} ($\piER$) are the two efficient versions of our initial algorithms. With an hyper-parameter $m>1$, they use \EFF instead of \UPDATE (Lines~\ref{algline:raw-update1} and~\ref{algline:raw-update2} in \RUCB). Therefore, they use $\left\{\hmueff\right\}_{i,h_j\in \Him}$ instead of $\left\{\hmu_i^h\right\}_{i,h \leq \Nitmone}$. More precisely, in \RUCB, we only change the $h\leq N_i$ by $h_j \in \Him$ and $\hmu_i^h$ by $\hmueff$ in the index computation at Line~\ref{algline:raw-pull}. We can perform similar changes to adapt \FEWA in \EFFFEWA.

\begin{proposition}
\EFFFEWA and \EFFRAW tuned with hyper-parameter $m$ have a $\cO\pa{K\log_m\pa{t}}$ worst-case time and space complexity at round $t$.
\end{proposition}
\begin{proof}
The total number of statistics for each arm $i$ at round $t$ is bounded by $\cO\pa{\log_m\pa{t}}$. Indeed, 
\[t \geq \Nitmone \geq h_j \geq m^{j-1} \implies j \leq 1 + \log_m\pa{t}.\]
Moreover, in \EFF we use 3 numbers for each $\left\{\hmueff\right\}_j$. Hence, the space complexity scales with \[ \sum_{i\in \arms} | \Him| = \sum_{i\in \arms} \cO\pa{\log_m\pa{t}} = \cO\pa{K\log_m\pa{t}}.\]
The time complexity of $\EFF$ scales with the number of statistics in arm $i_t$, \ie at most $\cO\pa{\log_m\pa{t}}$. The indexes computation of \EFFRAW  find the minimum of $K$ sets with cardinality $\cO\pa{\log_m\pa{t}}$, while finding the maximum among these indexes is a $\cO\pa{K}$ operation.  Thus, the worst-case time complexity is $\cO\pa{K\log_m\pa{t}}$. \EFFFEWA uses at most $\cO\pa{\log_m\pa{t}}$ times the procedure \FILTER  whose inner complexity scales with $|\arms_h| \leq K$. Therefore, in the worst case, the time complexity of \EFFFEWA at round $t$ is bounded by $\cO\pa{K\log_m\pa{t}}$.
\end{proof}

\subsection{Analysis}

The analysis of \RUCB and \FEWA only uses Proposition~\ref{prop:prb_favorable_event} and Lemma~\ref{lem:core-full}. We will derive analogous results for \EFFRAW and \EFFFEWA. The upper-bounds will directly follow with no additional effort.
\paragraph{A favorable event for efficiently updated adaptive windows}
\begin{proposition}
\label{prop:prb_favorable_event_eff}
For any round $t$ and confidence $\delta_{t} \triangleq 2t^{-\alpha}$, let 
\begin{equation*}
\!\HPeff\! \triangleq\! \Big\{ \forall i\!\in\!\arms,\ \forall n \!\leq\! t\!-\!1 ,\ \forall h_j \in \Him(n), \big| \hmueff(t, \pi) - \bmueff(t,\pi) \big| \!\leq\! c(h_j, \delta_{t}) \!\Big\}
\end{equation*}
 be the event under which the estimates at round $t$  are all accurate up to $c(h,\delta_{t}) \triangleq \sqrt{2 \subgaussian^2\log(2/\delta_t)/h}$. Then, for a policy $\pi$ which pulls each arms once at the beginning, and for all $t>K$,
\[
\PPempty\Big[\bar{\HPtwo}\Big] \leq 3Kt\delta_t= 6Kt^{1-\alpha}\,\cdot
\]
\end{proposition} 
\begin{remark}
The probability of the unfavorable event $\bar{\HPtwo}$ scales with $\cO\pa{t^{1-\alpha}}$ compared to $\cO\pa{t^{2-\alpha}}$ for $\bar{\HPevent}$ because the efficient algorithms construct less statistics. It means that our theory will hold for a wider range of $\alpha$. Yet, this benefits is only theoretical. The union bound in Proposition~\ref{prop:prb_favorable_event} is not tight because the different statistics share the same data: the confidence bounds are not independent at all. In practice, it leads to conservative tuning of the confidence bounds and one can decrease $\alpha$ to get better performance. 
\end{remark}

\begin{proof}
As in Proposition~\ref{prop:prb_favorable_event}, we have to count the number of statistics that are required to hold in the confidence region. Calling $u_j(t)$ the number of update of statistics $\hmueff$ after $t$ pulls, we have
\begin{align*}
    \PPempty\Big[\bar{\HPtwo}\Big] &\leq \sum_{i \in \arms} \sum_{j=0}^{\floor{\log_2\pa{t}}-1} u_j(t) \delta_t \\
    &\leq \sum_{i \in \arms} \pa{t - 1  + \sum_{j=1}^{\floor{\log_2\pa{t}}-1} \frac{t-1}{2^{j-1}}} \delta_t \\
    &\leq 3Kt\delta_t
\end{align*}
In the second inequality, we use Property~\ref{list:effu-m2} in Proposition~\ref{prop:effu}: statistics $\hmueff(n)$ is only updated every $2^{j-1}$ pulls for $j\geq 1$ (and every pull for $j=0$).
\end{proof}

\begin{lemma}
\label{lem:core-eff}
At round $t$ on favorable event $\HPtwo$, if arm~$i_{t}$ is selected by $\pi \in \left\{\piEF, \piER\right\}$ tuned with $m=2$, for any $h \leq \Nitmone$,  the average of its $h$ last pulls cannot deviate significantly from the best available arm at that round, i.e.,
\begin{equation*}
\bmu^{h}_{i_t}(t-1,\pi) \geq \max_{i \in \arms} \mu_{i}(t,\Nitmone)- \frac{C_\pi}{\sqrt{2\alpha}} c(h, \delta_t) \quad \text{with } 
\begin{cases}
C_{\piER} = \frac{4\sqrt{\alpha}}{\sqrt{2}-1}\\
C_{\piEF} = \frac{8\sqrt{\alpha}}{\sqrt{2}-1}
\end{cases}\cdot
\end{equation*}
\end{lemma}

\begin{proof}
Like for Lemma~\ref{lem:core-FEWA} (see its proof), our proof is done in a more general rotting framework that can be used in the next chapter. We denote by $\bar{\mu}^{hh'}_i(t-1,\pi)$ and $\hat{\mu}^{hh'}_i(t-1,\pi)$ the true mean and empirical average associated to the $h'-h$ samples between the $h$-th last one (included) and the $h'$-th last one (excluded). Let $j_h \in \NN^\star$ such that :
$2^{j_h} -1 \leq  h < 2^{j_h+1}$.
\begin{equation}
\label{eq:eff-decompo}
\bar{\mu}^{h}_{i_t}(t-1,\pi) \geq \bar{\mu}^{2^{j_h}-1}_{i_t}(t-1,\pi) = \sum_{j=0}^{j_h-1} \frac{2^j}{2^{j_h}-1} \bar{\mu}^{2^{j}2^{j+1}}_{i_t}(t-1,\pi).
\end{equation}
The inequality follows because the reward is decreasing and $h\geq 2^{j_h}-1$. Then, we decompose the average in a weighted sum of averages of geometrically expanding windows. Since the reward is decreasing we have that,
\begin{equation*}
\forall k \leq 2^j, \quad \bar{\mu}^{2^{j}2^{j+1}}_{i_t}(t-1,\pi) \geq \bar{\mu}^{k : k+2^{j}}_{i_t}(t-1,\pi).
\end{equation*}

$\hmuiteff$ contains $2^j$ samples among the $2^{j+1}-1$ last ones (see Proposition~\ref{prop:effu}). Setting $k\leq 2^j$ to the current delay of the statistics $\hmuiteff$ (see Point~\ref{list:effu-delay} in Proposition~\ref{prop:effu}), we can write,
\begin{equation}
\label{eq:hmueff-link}
\bar{\mu}^{2^{j}2^{j+1}}_{i_t}(t-1,\pi) \geq \bar{\mu}^{k : k+2^{j}}_{i_t}(t-1,\pi) = \bmuiteff\geq  \hmuiteff - c(2^j, \delta_t),
\end{equation}
where we use that we are on $\HPtwo$ at the last line. Therefore, gathering Equations~\ref{eq:eff-decompo} and~\ref{eq:hmueff-link}, 
\begin{equation}
\label{eq:eff-general}
\bar{\mu}^{h}_{i_t}(t-1,\pi) \geq \sum_{j=0}^{j_h-1} \frac{2^j}{2^{j_h}-1} \pa{\hmuiteff - c(2^j, \delta_t)}.
\end{equation}
Now, we will use the mechanics of the two algorithms. On the first hand, for $\EFFRAW$, we make the index appear in the inequality,
\begin{align}
 \bar{\mu}^{h}_{i_t}(t-1,\piER) &\geq \sum_{j=0}^{j_h-1} \frac{2^j}{2^{j_h}-1} \pa{\hmuiteff - c(2^j, \delta_t)} \nonumber\\
 &=\sum_{j=0}^{j_h-1} \frac{2^j}{2^{j_h}-1} \pa{\hmuiteff + c(2^j, \delta_t) - 2c(2^j, \delta_t)}\nonumber\\
 &\geq \min_{j \in \Hitwo} \pa{\hmuiteff + c(2^j, \delta_t)} - 2 \sum_{j=0}^{j_h-1} \frac{2^{j}}{2^{j_h}-1} c(2^j, \delta_t).
 \label{eq:effraw-index-appear}
 \end{align}
 Then, we can relate the left part of the sum to the best current value $\mu_{\ist}(t,\Nisttmone)$,
 \begin{equation}
\min_{j \in \Hitwo} \pa{\hmuiteff + c(2^j, \delta_t)} \geq \min_{j \in H_{\ist,2}} \pa{\hmu_{\ist,\, \tteff}^{h_j} + c(2^j, \delta_t)}\geq  \bmu_{\ist,\, \tteff}^{h_{\min}} \geq \mu_{\ist}(t,\Nisttmone).
 \label{eq:effraw-index-use}
 \end{equation}
where $h_{\min} \in \argmin_{h_j \in \Hitwo} \pa{\hmueff + c(h_j, \delta_t)} $.The first inequality follows because \EFFRAW selects the arm with the largest index. In particular, the index of $i_t$ is larger or equal to the index of $i^\star_t \in \argmax_{i\in \arms} \mu_i(t, N_{\ist,\,t})$. The second inequality holds on $\HPtwo$. The third inequality uses the decreasing of the reward. Putting Equations~\ref{eq:effraw-index-appear} and~\ref{eq:effraw-index-use}, we get,
\begin{equation}
\label{eq:effraw-result}
\bmu^{h}_{i_t}(t-1,\piER) \geq \mu_{\ist}(t,\Nisttmone)- 2 \sum_{j=0}^{j_h-1} \frac{2^{j}}{2^{j_h}-1} c(2^j, \delta_t).
\end{equation}
On the other hand, for \EFFFEWA, we know that the selected arm passes any filter of window $2^j \in \Hitwo$. Therefore, with $i_{\max} \in \argmax_{i \in \arms_{h_j}} \bmueff$, we can write,
\begin{flalign}
\qquad\hmuiteff &\geq \max_{i\in \arms_{h_j}} \hmueff -2c\pa{h_j, \delta_t} \nonumber && \text{Filtering rule}\\
\qquad&\geq \hmu_{i_{\max}, \, \tteff}^{h_j}  -2c\pa{h_j, \delta_t} \nonumber  && i_{\max} \in \arms_{h_j} \\
\qquad&\geq \bmu_{i_{\max}, \, \tteff}^{h_j} - 3c(h_j, \delta_t)\nonumber && \text{ on }\HPtwo \\
\qquad& = \max_{i\in \arms_{h_j}}  \bmueff -3c\pa{h_j, \delta_t}.
\label{eq:efffewa-3c}
\end{flalign}
We relate $\bmueff$ to the largest available value at round $t$,
\begin{equation}
\label{eq:efffewa-ist-relation}
\max_{i\in \arms_{h_j}} \bmueff \geq \max_{i\in \arms_{1}}\bmu_{i,\tteff}^1 =  \max_{i\in \arms}\bmu_{i,\tteff}^1 \geq \bmu_{\ist,\tteff}^1 \geq  \mu_{\ist}(t, \Nisttmone).
\end{equation}
The last inequality follows from the decreasing of the reward and the before last from the definition of the maximum operator. The first one uses a similar argument than in Lemma~\ref{lem:core-FEWA} : $\max_{i\in \arms_{h_j}} \bmueff$ increases with $h_j$.  Indeed, on $\HPtwo$, 
\[  i_j \triangleq \argmax_{i\in \arms_{h_j}} \bmueff \in \arms_{h_{j+1}},\]
 because it cannot be at more than two confidence bounds from the best empirical value during the filter $h_j$. Thus, we get, 
\[
\max_{i\in \arms_{h_j}} \bmueff = \bmu_{i_j,\tteff}^{h_j} \leq \bmu_{i_j,\tteff}^{h_{j+1}} \leq \max_{i\in \arms_{h_{j+1}}}\bmu_{i,\tteff}^{h_{j+1}}. 
\]
The first inequality follows because $\bmu_{i_j,\tteff}^{h_{j+1}}$ contains reward sample which are either in $\bmu_{i_j,\tteff}^{h_j}$ or are older than the ones in $\bmu_{i_j,\tteff}^{h_j}$. Indeed, when $m=2$, $\hmu_{i, \, \tteff}^{h_{j+1}}$ is updated synchronously with $\hmueff$ (see Property~\ref{list:effu-m2} in Proposition~\ref{prop:effu}). Hence,  at each update of $\hmu_{i, \, \tteff}^{h_{j+1}}$, it contains all the samples of $\hmueff$ and the $2^j$ precedent ones. Thus, because the reward is decreasing, we have $\bmu_{i_j,\tteff}^{h_{j+1}} \geq \bmu_{i_j,\tteff}^{h_{j}} $.  The second inequality uses that $i_j \in \arms_{h_{j+1}}$. Gathering Equations~\ref{eq:eff-general}, \ref{eq:efffewa-3c} and~\ref{eq:efffewa-ist-relation}, we get 
\begin{equation}
\label{eq:efffewa-result}
\bmu^{h}_{i_t}(t-1,\piEF) \geq  \mu_{\ist}(t,\Nisttmone)- 4 \sum_{j=0}^{j_h-1} \frac{2^{j}}{2^{j_h}-1} c(2^j, \delta_t).
\end{equation}
With few lines of algebra, we reduce the sum,
\begin{flalign*}
\sum_{j=0}^{j_h-1} \frac{2^{j}}{2^{j_h}-1} c(2^j, \delta_t) &= \sum_{j=0}^{j_h-1} \frac{\sqrt{2}^{j}}{2^{j_h}\!-\!1} c(1, \delta_t) && c(2^j, \delta_t) = \frac{c(1,\delta_t)}{\sqrt{2^j}} \\
& = \frac{\sqrt{2}^{j_h} -1 }{\pa{\sqrt{2}-1}\pa{2^{j_h}-1}}c(1, \delta_t) && \sum_{n=0}^N q^n = \frac{q^{N+1}-1}{q-1}\\
& = \frac{1}{\pa{\sqrt{2}-1}\pa{\sqrt{2}^{j_h}+1}} c(1, \delta_t) && a^2\!-\!1 \!=\!  \pa{a\!-\!1}\!\pa{a\!+\!1}\\
& \leq  \frac{\sqrt{2}}{\pa{\sqrt{2}-1}\sqrt{2^{j_h+1}}} c(1, \delta_t) &&\sqrt{2^{j_h}} +1 \geq  \frac{\sqrt{2^{j_h+1}} }{\sqrt{2}} \\
& \leq \frac{\sqrt{2}}{\pa{\sqrt{2}-1}\sqrt{h}} c(1, \delta_t) &&  h \leq 2^{j_h+1}  \\
& = \frac{\sqrt{2}}{\sqrt{2}-1} c(h, \delta_t). &&  \frac{c(1,\delta_t)}{\sqrt{h}} = c(h, \delta_t)
\end{flalign*}
Plugging this last equation in Equations~\ref{eq:effraw-result} and~\ref{eq:efffewa-result} leads to the final result,
\[
\bmu^{h}_{i_t}(t,\pi) \geq \max_{i \in \arms} \mu_{i}(t,\Nitmone)- \frac{C_\pi}{\sqrt{2\alpha}} c(h, \delta_t) \quad \text{with } 
\begin{cases}
C_{\piER} = \frac{4\sqrt{\alpha}}{\sqrt{2}-1}\\
C_{\piEF} = \frac{8\sqrt{\alpha}}{\sqrt{2}-1}
\end{cases}\cdot
\]
\end{proof}

\begin{remark}
\textbf{Can we adapt our theory for {$m\neq2$}?} 
For \EFFFEWA, we used at Equation~\ref{eq:efffewa-ist-relation} that $\hmueff$ is synchronously updated with the precedent statistics which is a specific characteristic for $m=2$. 
For \EFFRAW, the proof could work using a grid $\left\{2h_j, \dots,  2h_{j+1}-1\right\}$ to decompose the means (at Eq.~\ref{eq:eff-decompo}). Yet, the computation is much messier, mainly because of the ceil operator in $h_{j+1} = \ceil{m\cdot h_j}$. The constant ratio compared to \RAWUCB 's guarantee one could get with this technique would be no better than $\sqrt{m}\frac{m-1}{\sqrt{m}-1} = \sqrt{m}\pa{\sqrt{m}+1}$. When $m\rightarrow 1$, this constant does not go to one: it is disappointing because we know that \EFFRAW is equivalent to \RAWUCB for $m \leq 1 + \frac{1}{T}$.  
\end{remark}

Finally, we give a synthetic claim of Lemmas~\ref{lem:core-RAWUCB}, \ref{lem:core-FEWA} and~\ref{lem:core-eff}.
\begin{lemma}
\label{lem:core-full}
At round $t$ on favorable event $\HPevent$ (respectively, $\HPtwo$), if arm~$i_{t}$ is selected by $\pi \in \left\{\piF, \piR\right\}$ (respectively, $\pi \in \left\{\piEF, \piER\right\} $ tuned with $m=2$), for any $h \leq \Nitmone$,  the average of its $h$ last pulls cannot deviate significantly from the best available arm at that round, i.e.,
\begin{equation*}
\bmu^{h}_{i_t}(t,\pi) \geq \max_{i \in \arms} \mu_{i}(t)- \frac{C_\pi}{\sqrt{2\alpha}} c(h, \delta_t) \quad \text{with } 
\begin{cases}
C_{\piR} = 2\sqrt{2\alpha} \text{ and } C_{\piER} = \frac{4\sqrt{\alpha}}{\sqrt{2}-1}\\
C_{\piF} = 4\sqrt{2\alpha} \text{ and }C_{\piEF} = \frac{8\sqrt{\alpha}}{\sqrt{2}-1}
\end{cases}\cdot
\end{equation*}
\end{lemma}

\section{Analysis for the restless setting}
\label{app:restless}
\subsection*{Lower bounds}
The two lower bounds follow the same analysis. We build a set of rotting piece-wise stationary problems with an evenly spaced set of $\Upsilon -1$ breakpoints. The adversary can choose the distance between arms $\Delta=\frac{1}{4} \sqrt{\frac{\sigma^{2} K \Upsilon}{2 T}}$ at the maximum such that the best arm is barely identifiable between two breakpoints. Hence, at each break-point, each arm's value decreases by $\Delta$ or $2\Delta$. Even if the set of breakpoints would be known, the learner does not know which arm is the best on each stationary part. Hence, in the worst case, she suffers at least the sum of the minimax regret of $\Upsilon$ stationary bandits problems with horizon $\frac{T}{\Upsilon}$, \textit{i.e.}  $\cO \pa{\sqrt{K\Upsilon T}}$. In the piece-wise stationary setting, we can simply identify $\Upsilon = \Upsilon_T$. In the variation budget setting, the adversary has a constraint over $\Upsilon \Delta = \frac{1}{4} \sqrt{\frac{\sigma^{2} K \Upsilon^3}{2 T}}=  \cO\pa{V_T}$. Hence, when the budget is limited, the adversary can choose up to $\Upsilon = \cO\pa{ T^{1/3}}$ breakpoints such that the sub-optimal arms are "sufficiently" far from the best one (\textit{i.e} at $\Delta$). This dependence on $T$ leads to the increased regret rate of $\cO\pa{T^{2/3}}$.
\begin{lemma}\label{lem:lb}
Let $\Upsilon \in \left\{1,\dots, T\right\}$ and $\left\{\tau_k \triangleq \ceil{\frac{T}{ \Upsilon}} \text{ if } k \leq T \bmod{\Upsilon} \text{ else } \floor{\frac{T}{ \Upsilon}}\right\}_{k\leq \Upsilon}$. We call $t_k = \sum_{k'=1}^k \tau_{k'}$ and $t_0 = 0$.  Consider a family of piece-wise stationary bandits indexed by a vector $i^\star\in (\{0\}\cup \possibleArms)^{\Upsilon}$ as follows: arm $i$ is a Gaussian distribution $\mathcal{N}\pa{\mu_i(t), \sigma}$ such that 
\[
\forall k \in \left\{0 , \dots, \Upsilon -1 \right\}, \ \forall t \in \left\{t_{k-1}+1,\dots, t_{k}\right\}, \ 
\mu_i(t) = 
\begin{cases}
-k \Delta \text{ if } i = i^\star_k\\
-(k+1)\Delta  \text{ else.}
\end{cases}
\]
We denote by $\EEempty_{i^\star}$ the expectation under the problem indexed by $i^\star$. Then, if $\Delta = \frac{1}{4}\sqrt{\frac{\sigma^2K\Upsilon}{2T}}$, for any policy $\pi$ :
\[
 \exists i^\star\in (\{0\}\cup \possibleArms)^{\Upsilon}, \  \EEempty_{i^\star}\big[R_T(\pi)\big]  \geq  \frac{\sqrt{\sigma^2KT\Upsilon}}{32}\cdot
\]
\end{lemma}
\begin{proof}
Note that when $i^\star_k = 0$ then all the arms share the same means. We also define the vector $i^\star_{-k}$ equals to $i^\star$ with the coordinate $k$ empty and for $i\in\possibleArms$ the vector $(i^\star_{-k},i)$ as the vector where we fill the empty coordinate with $i$.  We fix a policy $\pi$ and we will lower bound its average regret on the bandits problem indexed by $i^\star \in \possibleArms^\Upsilon$ 
\begin{align*}
    \frac{1}{K^{\Upsilon}} \sum_{i^\star\in \possibleArms^{\Upsilon}}  \EEempty_{i^\star}\big[R_T(\pi)\big] &= \frac{1}{K^{\Upsilon}} \sum_{i^\star\in \possibleArms^{\Upsilon}} \sum_{k=1}^{\Upsilon}\Delta \EEempty_{i^\star}[\tau_k - N_{i^\star_k}^k] \\
    &=\Delta \left(T - \frac{1}{K^{\Upsilon}} \sum_{i^\star\in \possibleArms^{\Upsilon}} \sum_{k=1}^{\Upsilon} \EEempty_{i^\star}[N_{i^\star_k}^k]\right),
\end{align*}
where $N_i^k$ is the number of pulls of arm $i$ during epoch $k$. Thus we need to upper bound the following quantity
\[
\frac{1}{K^{\Upsilon}} \sum_{i^\star\in \possibleArms^{\Upsilon}} \sum_{k=1}^{\Upsilon} \EEempty_{i^\star}[N_{i^\star_k}^k] = \sum_{k=1}^{\Upsilon} \frac{1}{K^{\Upsilon-1}} \sum_{i^\star_{-k}\in \possibleArms^{\Upsilon-1}}\frac{1}{K} \sum_{i=1}^K\EEempty_{(i^\star_{-k},i)}[N_{i}^k]\,.
\]
Using the contraction of the entropy for the bounded random variable $N_{i}^k/\tau_k$ then the Pinsker inequality (see \citealp{garivier2018explore}) we get
\[
2\left(\frac{1}{\tau_k K} \sum_{i=1}^K\EEempty_{(i^\star_{-k},i)}[N_{i}^k] -\frac{1}{\tau_k K} \sum_{i=1}^K\EEempty_{(i^\star_{-k},0)}[N_{i}^k] \right)^2 \leq \frac{1}{K} \sum_{i=1}^K \EEempty_{(i^\star_{-k},0)}[N_{i}^k] \frac{\Delta^2}{2\sigma^2}\CommaBin
\]
since problems $(i^\star_{-k},i)$ and $(i^\star_{-k},0)$ differ only by a gap $\Delta$ on the arm $i$ during epoch $k$. Thanks to the fact that  $\sum_i N_i^k \leq \tau_k$ we get 
\[
\frac{1}{K} \sum_{i=1}^K\EEempty_{(i^\star_{-k},i)}[N_{i}^k] \leq \frac{\tau_k}{K} + \frac{\Delta}{2\sigma \sqrt{K}}\tau_k^{3/2}\,.
\]
Putting all together we have for $K\geq 2$
\begin{align*}
    \frac{1}{K^{\Upsilon}} \sum_{i^\star\in \possibleArms^{\Upsilon}}  \EEempty_{i^\star}\big[R_T(\pi)\big]  \geq \left(\frac{T}{2} -  \sum_{k=1}^{\Upsilon} \frac{\tau_k^{3/2} \Delta}{2\sigma \sqrt{K}}\right)\Delta\,.
\end{align*}
We have $\tau_k= \floor{\frac{T}{\Upsilon}}$ or $\tau_k= \ceil{\frac{T}{\Upsilon}}$  such that $\sum_{k=1}^{\Upsilon} \tau_k=T$. Hence, we have that $\tau_k \leq 2T/\Upsilon$ which leads to 
\[
 \frac{1}{K^{\Upsilon}} \sum_{i^\star\in \possibleArms^{\Upsilon}}  \EEempty_{i^\star}\big[R_T(\pi)\big]  \geq  \left(\frac{1}{2}T - \frac{\sqrt{2}T^{3/2}\Delta}{\sigma \sqrt{K\Upsilon}}\right)\Delta\,.
\]

Choosing $\Delta = \frac{1}{4}\sqrt{\frac{\sigma^2K\Upsilon}{2T}}$, we get 
\[
 \frac{1}{K^{\Upsilon}} \sum_{i^\star\in \possibleArms^{\Upsilon}}  \EEempty_{i^\star}\big[R_T(\pi)\big]  \geq  \frac{1}{4}\sqrt{\frac{\sigma^2K\Upsilon}{2T}}\left(\frac{1}{4}T\right) \geq \frac{\sqrt{\sigma^2KT\Upsilon}}{32}\cdot
\]
We can conclude by noticing that the average expected regret across the problem set  is lesser or equal to the maximum across the same problem set.
\end{proof}
\restapiecewiselb*
\begin{proof}
This result directly follows from Lemma~\ref{lem:lb} by choosing $\Upsilon = \Upsilon_T$. Indeed, the set of problems $\left\{i^\star \in \left(\left\{0\right\} \cup \possibleArms\right)^{\Upsilon_T} \right\}$ satisfy Assumption~\ref{assum:piecewise} as soon as $\Upsilon_T\Delta \leq V$, \textit{i.e.} $\Upsilon_T \leq \pa{\frac{32V^2T }{K\sigma^2}}^{1/3}$.
\end{proof}

\restavariationlb*
\begin{proof}
We want to use Lemma~\ref{lem:lb} but we need to make the set of problems $\left\{i^\star \in \left(\left\{0\right\} \cup \possibleArms\right)^{\Upsilon_T} \right\}$ comply with Assumption~\ref{assum:variation}. First, the function are bounded by $-V_T$. Hence, we need : 
\begin{equation}
\label{eq:bounded_condition}
  \Upsilon \Delta \leq V_T.  
\end{equation}

Second the total variation is bounded according to Equation~\ref{eq:defbudget}. When $t$ is not a break-point, the variation is null. At each break-point, the maximal variation across the arm is $2\Delta$. For $\Upsilon-1$ break-point, we have that

\begin{equation}
\label{eq:totalvar_condition}
  2\Delta \pa{\Upsilon-1}  \leq V_T.  
\end{equation}

Since $ 2\Delta \pa{\Upsilon-1} \leq \frac{\sigma}{2}\sqrt{\frac{K}{2T}}\Upsilon^{3/2} $, we choose 
\begin{equation}
\label{eq:set_upsilon}
\Upsilon = \min\pa{\max\pa{\floor{ 2\left(\frac{V_T^{2}T}{K\sigma^{2}}\right)^{1/3}},1},T}.
\end{equation}

By construction, \ref{eq:set_upsilon} satisfies \ref{eq:totalvar_condition}. Moreover, when $\Upsilon >1$, \ref{eq:totalvar_condition} is more restrictive than \ref{eq:bounded_condition}. For $\Upsilon = 1$, we simply assume $\Delta \leq V_T$, \textit{i.e.} $V_{T} \geq \sigma \sqrt{\frac{K}{8 T}}$.

Plugging \ref{eq:set_upsilon} in Lemma~\ref{lem:lb} allows us to conclude 
\[
    \mathbb{E}\left[R_T(\pi)\right] \geq \frac{1}{16\sqrt{2}} V_T^{1/3}\sigma^{2/3}K^{1/3}T^{2/3}.
\]
\end{proof}

\subsection*{Upper bounds}
\begin{lemma}[Bound on unfavorable events. Decomposition in unspecified batches. Bound on the first pull of each arm in each batch] 
\label{lem:FP}
Let an integer $\Upsilon \in\left\{ 1, \dots,T\right\}$.\\
Let $\mu_i : \NN^\star \rightarrow \left[0, -V\right]$, the $K$ decreasing reward functions.\\ 
Let $\left\{t_k\in\left\{ 1, \dots,T\right\} \right.\allowbreak\left. |\, t_k > t_{k-1}\right\}_{k\in \left\{ 1, \dots,\Upsilon-1\right\}}$ a set of $\Upsilon - 1$ distinct rounds delimiting $\Upsilon$ batches. We set $t_0=0$ and $t_\Upsilon = T$. \\
We call $h_{i}^{k} \triangleq \sum_{t=t_k +1}^{t_{k+1}} \mathds{1}\left(i_{t} = i\right)$ the number of pulls of arm $i$ in batch $k$ and $t_i^k(h)$ the time at which arm $i$ is pulled for the $h$-th time since $t_k + 1$. We also call $\arms_k \triangleq  \left\{ i \in \arms | h_i^k \geq 1\right\}$ the set of pulled arms in batch $k$. 

Then, $\pi \in \left\{\piR, \piF \right\}$ run with $\alpha \geq 4$, or $\pi \in \left\{\piER, \piEF \right\}$ run with $m=2$ and $\alpha \geq 3$, suffers an expected regret of
\begin{align*}
\EE{R_T(\pi)} \leq &  \, \EE{\sum_{k=0}^{\Upsilon-1} \sum_{i\in\arms_k}\sum_{t=t_k +1 }^{t_{k+1}}\sum_{h=2}^{h^k_{i}}\mathds{1}\pa{ t = t_i^k(h) \land \HPevent} \Big(\mu_{\star}(t) - \mu_{i}(t)\Big)} \\
&+   C_\pi \sigma \Upsilon K\sqrt{\log{T}} + 6KV.
\end{align*}
\end{lemma}
\begin{proof}
We start by separating the favorable events from the unfavorable events:
\begin{equation}
\label{eq:event_sep}
    R_T(\pi) = \underbrace{\sum_{t=1}^T \mathds{1}\pa{\HPevent} \pa{\mu_{\star}(t) - \mu_{i_t}(t)}}_{R_T(\pi | \HPevent)} + \underbrace{\sum_{t=1}^T \mathds{1}\big(\bar{\HPevent}\big) \pa{\mu_{\star}(t) - \mu_{i_t}(t)}}_{{R_T(\pi | \bar{\HPevent})}} \,,
\end{equation}
with $\mu_\star(t) \triangleq \max_{i\in\arms}\mu_i(t)$. For $\alpha \geq 4$, we can bound the cost of the unfavorable events thanks to Proposition~\ref{prop:prb_favorable_event},
\begin{equation}
\label{eq:bad_event}
    \EE {R_T(\pi | \bar{\HPevent})} \leq \sum_{t=1}^T \PP{\bar{\HPevent}} V  \leq \sum_{t=1}^T \frac{KV}{t^2} = \frac{KV\pi^2}{6} \leq 2KV.
\end{equation}

On the favorable events, given any ordered set of $\Upsilon -1$ breakpoints $\left\{t_k\right\}$, we divide the horizon in $\Upsilon$ batches $\left\{t_k+1, \dots, t_{k+1} \right\}_{k \leq \Upsilon-1}$, 
\[
R_T(\pi | \HPevent) \leq \sum_{k=0}^{\Upsilon-1} \sum_{t=t_{k} +1 }^{t_{k+1}} \mathds{1}\pa{\HPevent} \big(\mu_{\star}(t) - \mu_{i_t}(t)\big).
\]
We define $h_{i}^{k}$ the number of pulls of arm $i$ in batch $k$, \textit{i.e.}  $h_{i}^{k} = \sum_{t=t_k +1}^{t_{k+1}} \mathds{1}\left(i_{t} = i\right)$. We use $t_i^k(h)$ to designate the time at which arm $i$ is pulled for the $h$-th time since $t_k$.
\[
R_T(\pi | \HPevent) \leq \sum_{k=0}^{\Upsilon-1} \sum_{t=t_k +1 }^{t_{k+1}} \sum_{i\in\arms_k} \sum_{h=1}^{h^k_{i}} \mathds{1}\pa{t_i^k(h) = t \land \HPevent} \Big(\mu_{\star}(t) - \mu_{i}(t)\Big).
\]
We split the regret on the first pulls of each batch
\begin{align}
\label{eq:fp_op}
\begin{split}
    R_T(\pi | \HPevent) = & \underbrace{\sum_{k=0}^{\Upsilon-1}\sum_{t=t_{k} +1 }^{t_{k+1}} \sum_{i\in\arms_k} \mathds{1}\pa{t = t_i^k(1) \land \HPevent}\Big(\mu_{\star}(t) - \mu_{i}(t))\Big)}_{FP} \\ & +  \underbrace{\sum_{k=0}^{\Upsilon-1}\sum_{t=t_{k} +1 }^{t_{k+1}} \sum_{i\in\arms_k} \sum_{h=2}^{h^k_{i}}\mathds{1}\left( t = t_i^k(h) \land \HPevent \right)\Big( \mu_{\star}(t) - \mu_{i}(t)\Big)}_{OP} .
\end{split}
\end{align}

\paragraph{Analysis of the first pulls.}

We call $k_i^1$, the index of the batch at which arm $i$ is pulled for the first time.  We call $\arms_k^2 \triangleq \left\{ i \in \arms_k | k > k_i^1\right\}$, the set of arms pulled at least once during batch $k$ and at least once in a batch before $k$. We split the regret due to the very first pull each arm from the other first pulls in each batch,
\begin{align*}
FP  = &\sum_{k=0}^{\Upsilon-1}\sum_{i\in\arms_k}\sum_{t=t_{k} +1 }^{t_{k+1}}  \mathds{1}\pa{ t = t_i^k(1) \land \HPevent}\Big(\mu_{\star}(t) - \mu_{i}(t)\Big)\\
\leq& \sum_{i \in \arms}  \Big(0- \mu_i(t_i^{k_i^1}(1))\Big) +  \sum_{k=1}^{\Upsilon -1} \sum_{i\in \arms_k^2}\sum_{t=t_k +1 }^{t_{k+1}}  \mathds{1}\pa{t = t_i^k(1) \land \HPevent}\Big(\mu_{\star}(t) - \mu_{i}(t)\Big)\\
 =& \sum_{i \in \arms} \Big(0- \mu_i(t_i^{k_i^1}(1))\Big) \\
& + \sum_{k=1}^{\Upsilon -1} \sum_{i\in \arms_k^2} \sum_{t=t_k +1 }^{t_{k+1}}  \mathds{1}\pa{ t = t_i^k(1) \land \HPevent}\Big(\mu_{\star}(t) - \bar{\mu}^1_i(t, \pi) + \bar{\mu}^1_i(t,\pi) - \mu_{i}(t)\Big).
\end{align*}

The inequality is justified because $\mu_i(t) \leq 0$ for all $t$. In the last equation, we simply introduce $\bar{\mu}^1_i(t,\pi)$, the last pulled sample of arm $i$, which is well defined after the first pull of each arm.
According to Lemma~\ref{lem:core-full}, the first difference is bounded on the high-probability event $\HPevent$,
\begin{equation}
    \label{eq:fp_lemma1}
    \sum_{t=t_k +1 }^{t_{k+1}} \mathds{1}\pa{t = t_i^k(1) \land \HPevent}\pa{\mu_{\star}(t) - \bar{\mu}^1_i(t, \pi)} \leq \frac{C_\pi}{\sqrt{2\alpha}} c(1,2T^{-\alpha}) = C_{\pi} \sigma \sqrt{\log{T}}.
\end{equation}

We will show that we can telescope the second sum. First, we notice that we can collapse the sum on $t$ using $ \mathds{1}\pa{t = t_i^k(1)}$. Moreover, $\HPevent$ will not be needed: hence we can drop $\mathds{1}\pa{\HPevent} \leq 1 $.
\begin{equation}
\label{eq:fp_collapse}
 \sum_{t=t_k +1 }^{t_{k+1}} \mathds{1}\pa{ t = t_i^k(1) \land \HPevent}\pa{\bar{\mu}^1_i(t,\pi) - \mu_{i}(t)} \leq \bar{\mu}^1_i(t_i^k(1), \pi) - \mu_{i}(t_i^k(1)).
\end{equation}

For a given batch $k$ on which arm $i$ is pulled, the precedent reward sample has a mean $\bar{\mu}_i^h\pa{t_i^k\pa{1}, \pi}$. This sample is the last pull of the last batch $k'$ before $k$ on which arm $i$ is pulled. Hence, its mean is smaller than the mean of the first pull on this same batch $k'$ because the reward is decreasing. Hence, the sum can telescope
\begin{align}
\label{eq:fp_telescoping} 
\sum_{i \in \arms} \Big(0 -& \mu_i(t_i^{k_i^1}(1))\Big) + \sum_{k=1}^{\Upsilon -1}  \sum_{i\in \arms_k^2} \sum_{t=t_k +1 }^{t_{k+1}}\mathds{1}\pa{ t = t_i^k(1) \land \HPevent}\pa{ \bar{\mu}^1_i(t,\pi) - \mu_{i}(t)} \nonumber\\
& \leq \sum_{i \in \arms}\left\{ 0- \mu_i(t_i^{k_i^1}(1)) + \sum_{k=k_i^1 + 1}^{\Upsilon-1 } \mathds{1}\pa{h^k_{i} \geq 1  }\pa{\bar{\mu}^1_i(t_i^k(1), \pi) - \mu_{i}(t_i^k(1))} \right\} \nonumber\\
& \leq \sum_{i \in \arms} \Big(0-\mu_i(T)\Big) \leq KV\,.  
\end{align}
The first inequality uses the definition of $\arms_k^2$ along with Equation~\ref{eq:fp_collapse}. The second inequality follows from the telescoping argument presented above. The third inequality uses that $\mu_i(T) \geq -V$. Gathering Equation~\ref{eq:fp_lemma1} and  ~\ref{eq:fp_telescoping}, we can bound the term $FP$ (defined in Equation~\ref{eq:fp_op}) 
\begin{equation}
\label{eq:fp_bound}
FP \leq  KV + \sum_{k=1}^{\Upsilon - 1} \sum_{i\in \arms_k^2} C_\pi \sigma \sqrt{\log{T}} \leq KV + C_\pi \sigma \Upsilon K\sqrt{\log{T}} .    
\end{equation}

\paragraph{Conclusion.} From Equation~\ref{eq:event_sep}, we can bound the expected regret on the unfavorable events thanks to Equation~\ref{eq:bad_event}. On the favorable events, we can split the rounds in batches on which we isolate the first pull of each arm on each batch thanks to Equation~\ref{eq:fp_op}. Finally, we bound the regret due to these first pulls thanks to Equation~\ref{eq:fp_bound}, and for $\alpha \geq 4$,
\begin{align*}
\EE{R_T(\pi)} \leq &  \, \EE{\sum_{k=0}^{\Upsilon-1} \sum_{i\in\arms_k}\sum_{t=t_k +1 }^{t_{k+1}}\sum_{h=2}^{h^k_{i}}\mathds{1}\pa{ t = t_i^k(h) \land \HPevent} \Big( \mu_{\star}(t) - \mu_{i}(t)\Big)} \\
&+  C_\pi \sigma \Upsilon K \sqrt{\log{T}} + 3KV.
\end{align*}

For the efficient algorithms, we can use the same proof with $\HPtwo$ and get for $\alpha \geq 3$, 
\begin{align*}
\EE{R_T(\pi)} \leq &  \, \EE{\sum_{k=0}^{\Upsilon-1} \sum_{i\in\arms_k}\sum_{t=t_k +1 }^{t_{k+1}}\sum_{h=2}^{h^k_{i}}\mathds{1}\pa{ t = t_i^k(h) \land \HPevent} \Big(  \mu_{\star}(t) - \mu_{i}(t)\Big)} \\
&+  C_\pi \sigma \Upsilon K \sqrt{\log{T}} + 6KV.
\end{align*}

\end{proof}

\begin{lemma}[Analysis of the second pulls in each batch under the favorable events.]\label{lem:OP}
Let $\Delta_i^k \triangleq \mu_i(t_k+1) - \mu_i(t_{k+1})$, the decrement of arm $i$ in batch $k$. For any arm $i$ and any consecutive rounds $\left\{t_k+1, \dots , t_{k+1}\right\}$ such that $i$ is pulled $h_i^{k} \geq 1$ times, the regret due to the pulls after the first one can be bounded under the favorable events, 
\begin{multline*}
\sum_{t=t_{k} +1 }^{t_{k+1}} \sum_{h=2}^{h^k_{i}}\mathds{1}\left(t = t_i^k(h) \land \HPevent \right)\Big(  \mu_{\star}(t) - \mu_{i}(t)\Big) \leq  \pa{h_i^k-1}\Delta_i^k + \sum_{h=2}^{h^k_{i}}\mathds{1}\left(\HPt{t_i^k(h)}\right)\pa{  \mu_{\star}(t_i^k(h)) - \bar{\mu}_i^{h-1}(t_i^k(h),\pi)} .
\end{multline*}
\end{lemma}
\begin{proof}
We call $\Delta_{i}(t,t')\triangleq \mu_i(t) - \mu_i(t')$ the variation of arm $i$ between times $t$ and $t'$.
As a short notation, we refer to $\Delta_i^k \triangleq \Delta_{i}(t_k+1,t_{k+1})$ for the variation of arm $i$ in batch $k$.

\begin{equation}
\label{eq:batch_delta_var}
    \forall h\leq h_i^k, \quad \mu_i( t_i^k(h)) \geq \mu_i(t_{k+1}) =  \mu_i(t_k +1 ) - \Delta_i^k \geq \bar{\mu}_i^{h-1}( t_i^k(h), \pi) - \Delta_i^k\,.
\end{equation} 
The two inequalities are justified by the rewards decay. Indeed, any pull in batch $k$ has a higher reward than the value of arm $i$ at the end of the batch $t_{k+1}$. Moreover, the value at the beginning of the batch is higher that any average of $h$ value in this batch. The middle equality follows from the definition of $\Delta_i^k$.

Then, we plug Equation~\ref{eq:batch_delta_var} in the left hand side of our claim,
\begin{align*}
\sum_{t=t_{k} +1 }^{t_{k+1}}
\sum_{h=2}^{h^k_{i}}\mathds{1}&\left(t = t_i^k(h) \land \HPevent\right)\pa{  \mu_{\star}(t) - \mu_{i}(t)} \\
& = \sum_{h=2}^{h^k_{i}}\mathds{1}\left(\HPt{t_i^k(h)}\right)\pa{  \mu_{\star}(t_i^k(h)) - \mu_{i}(t_i^k(h))} \\
&\leq \sum_{h=2}^{h^k_{i}}\mathds{1}\left(\HPt{t_i^k(h)}\right)\pa{  \mu_{\star}(t_i^k(h)) - \bar{\mu}_i^{h-1}( t_i^k(h),\pi) + \Delta_i^k}\\
&\leq  \pa{h_i^k-1}\Delta_i^k + \sum_{h=2}^{h^k_{i}}\mathds{1}\left(\HPt{t_i^k(h)}\right)\pa{  \mu_{\star}(t_i^k(h)) - \bar{\mu}_i^{h-1}(t_i^k(h),\pi)} .
\end{align*}
The last inequality is justified by $\mathds{1}\left(\HPt{t_i^k(h)}\right)\leq 1$.
\end{proof}

\subsubsection*{Variation budget rotting bandits.}
\restabudgettheorem*
\begin{proof}
Let $\Upsilon \in \left\{1, \dots, T\right\}$ a number of evenly spaced batches that we will specify later. We define the length of these batches $\left\{\tau_{k} \triangleq\left\lceil\frac{T}{\Upsilon}\right\rceil \text { if } k \leq T \bmod \Upsilon \text { else }\left\lfloor\frac{T}{\Upsilon}\right\rfloor\right\}_{k \leq \Upsilon}$. Note that $\sum_{k=1}^{\Upsilon} \tau_k = T$. Let $t_k = \sum_{k'=0}^k \tau_{k'}$ the last round of each batch and $t_0 = 0$. On each of these batches, we apply Lemma~\ref{lem:OP} for the set of arms which have been pulled in this batch,
\begin{multline}
\label{eq:op_decomposition}
    \sum_{k=0}^{\Upsilon_T-1} \sum_{i\in\arms_k}\sum_{t=t_{k}+1}^{t_{k+1}} \sum_{h=2}^{h_{t}^{k}} \mathds{1}\left( t=t_{i}^{k}(h) \land \HPevent\right)\Big(\mu_{\star}(t)-\mu_{i}(t)\Big)
\leq \sum_{k=0}^{\Upsilon-1} \sum_{i\in\arms_k} \pa{h_i^k-1}\Delta_i^k\\
+ \sum_{k=0}^{\Upsilon-1} \sum_{i\in\arms_k}\sum_{h=2}^{h^k_{i}}\mathds{1}\left(\HPt{t_i^k(h)}\right)\pa{  \mu_{\star}(t_i^k(h)) - \bar{\mu}_i^{h-1}(t_i^k(h), \pi)}.
\end{multline}

The first sums can be handled using Assumption~\ref{assum:variation} and the evenly spaced property of $\tau_k$,
\begin{equation}
\label{eq:use_evenly_spaced}
\sum_{k=0}^{\Upsilon-1} \sum_{i\in\arms} \pa{h_i^k-1}\Delta_i^k \leq \sum_{k=0}^{\Upsilon-1} \max_{j\in \arms} \Delta_j^k\sum_{i\in\arms} \pa{h_i^k-1} = \sum_{k=0}^{\Upsilon-1} \max_{j\in \arms} \Delta_j^k \pa{\tau_k-K} \leq \frac{T}{\Upsilon}\sum_{k=0}^{\Upsilon-1} \max_{j\in \arms} \Delta_j^k.
\end{equation}
The first inequality is justified by definition of the maximum. The second equality states that the total number of pulls in batch $k$ is $\tau_k$. The third inequality uses that $\tau_k - K \leq \ceil{\frac{T}{\Upsilon}} -K \leq \ceil{\frac{T}{\Upsilon}} -K \leq \frac{T}{\Upsilon}$. Now, we need to relate $\max_{j\in \arms} \Delta_j^k$ and $V_T$,
\begin{equation}
\label{eq:use_assum_variation}
   \sum_{k=0}^{\Upsilon\!-\!1}\max_{j\in \arms} \Delta_j^k \!=\! \sum_{k=0}^{\Upsilon\!-\!1}\max_{j\in \arms} \!\sum_{t = t_k\!+\!1}^{t_{k\!+\!1}\!-\!1}\! \Delta_j\! \pa{t,t\!+\!1} \!\leq\! \sum_{k=0}^{\Upsilon\!-\!1} \sum_{t = t_k\!+\!1}^{t_{k\!+\!1}\!-\!1} \! \max_{j\in \arms} \Delta_j\!\pa{t,t\!+\!1} \!\leq\!  \sum_{t = 1}^{T}\max_{j\in \arms} \Delta_j\!\pa{t,t\!+\!1}\!\leq\! V_T .
\end{equation}
The first inequality is justified because the maximum of a sum is smaller than the sum of the maximums. In the second inequality, we add positive terms which are the maximum of the decay among the arms at the boundary between batches. The last inequality is justified by Assumption~\ref{assum:variation}. Therefore, we can bound the first sums using Equation~\ref{eq:use_evenly_spaced} and ~\ref{eq:use_assum_variation},
\begin{equation}
\label{eq:bound_sum1}
\sum_{k=0}^{\Upsilon-1} \sum_{i\in\arms} \pa{h_i^k-1}\Delta_i^k \leq \frac{V_T T }{\Upsilon}\cdot    
\end{equation}

The second sums can be bounded using Lemma~\ref{lem:core-full} on the high probability event $\HPt{t_i^k(h)}$ and Jensen's inequality,
\begin{align}
    \sum_{k=0}^{\Upsilon-1} \!\sum_{i\in\arms_k}\!\sum_{h=2}^{h^k_{i}}\!\mathds{1}\!\left(\!\HPt{t_i^k(h)}\!\right)\pa{\! \mu_{\star}(t_i^k(h)) \!- \!\bar{\mu}_i^{h-1}(t_i^k(h),\pi)\!} \!&\leq \sum_{k=0}^{\Upsilon-1} \sum_{i\in\arms_k} \sum_{h=2}^{h^k_{i}} \frac{C_\pi c\pa{\!h\!-\!1, 2T^{-\alpha}\!}}{\sqrt{2\alpha}} \nonumber\\
&= \sum_{k=0}^{\Upsilon-1} \sum_{i\in\arms_k} \sum_{h=2}^{h^k_{i}}C_\pi\sigma \sqrt{\frac{\log{T}}{h -1}}\nonumber\\
&\leq \sum_{k=0}^{\Upsilon-1} \sum_{i\in\arms_k} 2 C_\pi \sigma \sqrt{h_i^k \log{T}}\nonumber\\
&\leq  2 C_\pi \sigma \sqrt{\Upsilon K T\log{T}} .
\label{eq:bound_sum2}
\end{align}

We remark that the bound in Eq.~\ref{eq:bound_sum1} is decreasing with $\Upsilon$ and the bound in Eq.~\ref{eq:bound_sum2} is increasing with $\Upsilon$. We will choose $\Upsilon$ in order to minimize the sum of these two bounds (which will be our leading term). Therefore, we set,
\begin{equation}
    \label{eq:set_upsilon_variation}
    \Upsilon \triangleq \ceil{\pa{\frac{V_T^2 T}{C_\pi^2 \sigma^2 K\log{T}}}^{\nicefrac{1}{3}}}.
\end{equation}

We have that $\Upsilon\leq T$ when $V_T \leq  C_\pi \sigma T\sqrt{ K\log{T}}$. Moreover, we will use that $ \Upsilon \leq 2 \pa{\frac{V_T^2 T}{C_\pi^2 \sigma^2 K\log{T}}}^{\nicefrac{1}{3}} $ which is true when $V_T \geq \sqrt{\frac{C_\pi^2 \sigma^2 K\log{T}}{8T}}$. 

Finally, we use Lemma~\ref{lem:FP} where we replace the inner sums thanks to Equations~\ref{eq:op_decomposition}, \ref{eq:bound_sum1} and~\ref{eq:bound_sum2}. Then, we plug $\Upsilon$ set in \ref{eq:set_upsilon_variation} and conclude,
\begin{align*}
\EE{R_T\pa{\pi}} & \leq \frac{V_T T}{\Upsilon} +  2C_\pi\sigma \sqrt{\Upsilon K T\log{T}} +  C_\pi \sigma \Upsilon  K\sqrt{\log{T}} + 6 V_T K\\
&\leq  4\pa{C_\pi^2 \sigma^2 V_T K T^2\log{T}}^{\nicefrac{1}{3}} \!+ 2\Big(C_\pi \sigma V_T^2  K^2  T \sqrt{\log{T}}\Big)^{\nicefrac{1}{3}} \!+ 6 V_T K.
\end{align*}

When $V_T\leq  \sqrt{\frac{C_\pi^2 \sigma^2 K \log{T}}{8T}}$, the regret of any policy can be bounded , 
\begin{align*}
\mathbb{E}\left[R_T(\pi)\right] &\leq T V_T = V_T^{\nicefrac{1}{3}} T^{\nicefrac{2}{3}} V_T^{\nicefrac{2}{3}} T^{\nicefrac{1}{3}}\\
&\leq V_T^{\nicefrac{1}{3}} T^{\nicefrac{2}{3}} \left(\frac{ C_\pi^2 \sigma^2 K \log{T}}{8T}\right)^{\nicefrac{1}{3}} T^{\nicefrac{1}{3}}\\ 
&= \frac{1}{2} \pa{C_\pi^2\sigma^2 V_T K T^2 \log{T}}^{\nicefrac{1}{3}}\\
&\leq 4 \pa{C_\pi^2 \sigma^2 V_T K T^2 \log{T}}^{\nicefrac{1}{3}}.
\end{align*}

For completion, we also consider $V_T \geq  C_\pi \sigma T\sqrt{ K\log{T}}$. Yet, notice that in that case the leading term is $\cO\pa{KV_T}$. We start back from Lemma~\ref{lem:FP},
\begin{align*}
\EE{R_T(\pi)} \leq &  \, \EE{\sum_{k=0}^{\Upsilon-1} \sum_{i\in\arms_k}\sum_{t=t_k +1 }^{t_{k+1}}\sum_{h=2}^{h^k_{i}}\mathds{1}\pa{ t = t_i^k(h) \land \HPevent} \Big(\mu_{\star}(t) - \mu_{i}(t)\Big)} \\
&+   C_\pi \sigma \Upsilon K\sqrt{\log{T}} + 6KV_T.
\end{align*}
In fact, this result can be slightly improved at no cost, 
\begin{align*}
\EE{R_T(\pi)} \leq &  \, \EE{\sum_{k=0}^{\Upsilon-1} \sum_{i\in\arms_k}\sum_{t=t_k +1 }^{t_{k+1}}\sum_{h=2}^{h^k_{i}}\mathds{1}\pa{ t = t_i^k(h) \land \HPevent} \Big(\mu_{\star}(t) - \mu_{i}(t)\Big)} \\
&+   C_\pi \sigma \min\pa{\Upsilon K, T}\sqrt{\log{T}} + 6KV_T,
\end{align*}
because there are at most $\min\pa{\Upsilon K, T}$ first pulls (see the proof of Lemma~\ref{lem:FP}). Now, we choose $\Upsilon = T$. Hence, there is no second pulls and we have,
\begin{equation*}
\EE{R_T(\pi)} \leq   C_\pi \sigma T\sqrt{\log{T}} + 6KV_T,
\end{equation*} 

Now we use that $C_\pi \sigma T\sqrt{\log{T}} \leq \frac{V_T}{\sqrt{K}} \leq KV_T$,
\begin{align*}
\EE{R_T(\pi)} &\leq   \pa{C_\pi \sigma T\sqrt{\log{T}}}^{\nicefrac{2}{3}}\! \pa{C_\pi \sigma T\sqrt{\log{T}}}^{\nicefrac{1}{3}}\! + 6KV_T \\
& \leq  \pa{C_\pi^2 \sigma^2 V_T K T^2\log{T}}^{\nicefrac{1}{3}} \!+  6 KV_T\\
&\leq  4\pa{C_\pi^2 \sigma^2 V_T K T^2\log{T}}^{\nicefrac{1}{3}} \!+ 2\Big(C_\pi \sigma V_T^2  K^2  T \sqrt{\log{T}}\Big)^{\nicefrac{1}{3}} \!+ 6 K V_T.
\end{align*} 
\end{proof}

\subsubsection*{Piece-wise stationary rotting bandits.}
\sloppy
Let $\left\{t_k\right\}_{\left\{k \leq \Upsilon_T\right\}}$ be the set of breakpoints with $t_0=0$ and $t_{\Upsilon_T} = T$. For all $t \in \left\{t_k\! +\!1 , \dots , t_{k\!+\!1}\right\}$, $\mu_i(t) = \mu_i^k$. We denote $i^\star_k \in \argmax_{i\in \arms} \mu_i^k$ (one of) the best arm(s) in batch $k$, and $\mu_{\star}^k \triangleq \max_{i\in \arms} \mu_i^k$, the corresponding best value. We also call $\Delta_{i,k} \triangleq \mu_{\star}^k - \mu_i^k$ the gap between arm $i$ and optimal arm in batch $k$.

\begin{lemma}
\label{lem:OP-piecewise}
For an arm $i$ and a stationary batch $k$, we call $h_{i,\xi}^k \triangleq \max\left(h \leq h_i^k | \HPt{t_i^k(h)} \right)$ the last pull of arm $i$ in batch $k$ under the favorable events (possibly 0). If $h_{i,\xi}^k \geq 1$, the regret due to the second pulls on the favorable events is bounded by,
\[
\sum_{t=t_{k}+1}^{t_{k+1}} \sum_{h=2}^{h_{i}^{k}} \!\mathds{1}\!\pa{\! t=t_{i}^{k}(h) \land \HPevent \!}\!\Big(\!\mu_{\star}(t)\!-\!\mu_{i}(t)\!\Big) \leq\pa{\!h_{i, \xi}^{k}\!-\!1 \!} \Delta_{i, k} \leq C_\pi\sigma \sqrt{\pa{\!h_{i,\xi}^k\!-\!1\!}\log{T}}.
\]
\end{lemma}
\begin{proof}
We apply Lemma~\ref{lem:OP} on each stationary batch. Hence, $\Delta_i^k =0$ and we can write,
\begin{equation*}
\!\sum_{t=t_{k} +1 }^{t_{k+1}}\! \sum_{h=2}^{h^k_{i}} \! \mathds{1}\left(\!t = t_i^k(h) \land \HPevent \!\right)\Big(\! \mu_{\star}(t) - \mu_{i}(t)\!\Big) \leq   \sum_{h=2}^{h^k_{i}} \!\mathds{1}\left(\!\HPt{t_i^k(h)}\!\right)\pa{\!  \mu_{\star}(t_i^k(h)) - \bar{\mu}_i^{h-1}(t_i^k(h),\pi)\!}\! .
\end{equation*}

We notice that $\mu_{\star}(t_i^k(h)) = \mu_{\star}^{(k)}$. We call $h_{i,\xi}^k \triangleq \max\left(h \leq h_i^k\, \,| \HPt{t_i^k(h)} \right)$. Hence,
\begin{align*}
\sum_{h=2}^{h^k_{i}}\!\mathds{1}\left(\HPt{t_i^k(h)}\right)\!\pa{ \! \mu_{\star}(t_i^k(h)) \!-\! \bar{\mu}_i^{h\!-\!1}(t_i^k(h),\pi)\!}  & \!=\! \sum_{h=2}^{h^k_{i,\xi}}\!\mathds{1}\left(\HPt{t_i^k(h)}\right)\!\pa{  \mu^k_{\star} \!-\! \bar{\mu}_i^{h\!-\!1}(t_i^k(h), \pi)} \\
 &\leq \sum_{h=2}^{h^k_{i, \xi}} \mu_{\star}^{k} - \bar{\mu}_i^{h-1}(t_i^k(h), \pi)\\
 &= \sum_{h=2}^{h^k_{i, \xi}} \mu_{\star}^{k} - \mu_i^k\\
 & = \pa{h_{i,\xi}^k - 1 }\Delta_{i,k}\, .  
\end{align*}

The first equality follows from $\forall h > h_{i,\xi}^k, \, \mathds{1}\left(\HPt{t_i^k(h)}\right) =0$ by definition of $h_{i,\xi}^k$. The first inequality follows by dropping $\mathds{1}\left(\HPt{t_i^k(h)}\right) \leq 1$. The second equality uses that the function is stationary in batch $k$ : $\forall h \leq h_{i,\xi}^k, \bar{\mu}_i^{h-1}(t_i^k(h), \pi) = \mu_{i}^k.$ The last equality follows by definition of $\Delta_{i,k}$ (which does not depend on the summand index $h$).

Then, we apply Lemma~\ref{lem:core-full} at time $t_i^k\pa{h_{i,\xi}^k}$. By definition of $h_{i,\xi}^k$, $\mathds{1}\!\left(\!\HPt{t_i^k(h_{i,\xi}^k)\!}\right) = 1$.
\begin{align*}
 \pa{h_{i,\xi}^k - 1 }\Delta_{i,k}  \leq \frac{C_\pi}{\sqrt{2\alpha}}\pa{h_{i,\xi}^k - 1 }c(h_{i,\xi}^k\!-\!1, 2T^{-\alpha}) = C_\pi \sigma \sqrt{\pa{h_{i,\xi}^k-1}\log{T}}.
\end{align*} 
\end{proof}

\restapiecewisetheorem*
\begin{proof}

We apply Lemma~\ref{lem:OP-piecewise},
\[\sum_{k=0}^{\Upsilon_T-1} \! \sum_{i \in \mathcal{K}_{k}} \!\sum_{t=t_{k}+1}^{t_{k+1}} \! \sum_{h=2}^{h_{i}^{k}} \!\mathds{1}\!\left( \!t\!=\!t_{i}^{k}(h)\!\land\!\HPevent\!\right)\Big(\mu_{\star}(t)\!-\!\mu_{i}(t)\Big) \leq  \sum_{k=0}^{\Upsilon_T-1} \!\sum_{i\in\arms_k}\!  C_\pi \sigma \sqrt{h_{i,\xi}^k\log{T}} .\]

We notice that $ \sum_{k=0}^{\Upsilon_T -1}\sum_{i\in\arms_k} h_{i,\xi}^k\leq T$. Hence, thanks to Jensen's inequality, 
\[
\sum_{k=0}^{\Upsilon_T-1} \sum_{i\in\arms_k} C_\pi\sigma \sqrt{h_{i,\xi}^k\log{T}} \leq  C_\pi \sigma \sqrt{ \Upsilon_T K T\log{T}}.
\]

We use Lemma~\ref{lem:FP} with the last equation and conclude,
\[
\EE{R_T(\pi)} \leq C_\pi \sigma \sqrt{\log{T}} \pa{ \sqrt{\Upsilon_T KT} + \Upsilon_T K} + 6KV.
\]
\end{proof}

\restapiecewisetheorempd*
\begin{proof}
Let $\arms_k \triangleq \left\{ i \in \arms | \Delta_{i,k} > 0\right\}$, the set of sub-optimal arms in batch $k$.
We apply Lemma~\ref{lem:OP-piecewise} to bound the number of wrong pull (under the favorable events) of arm $i\in \arms_k$ during batch $k$,
\begin{align*}
     \Delta_{i,k} \pa{h^k_{i, \xi} -1} & \leq C_\pi \sigma \sqrt{\pa{h_{i,\xi}^k-1}\log{T}} \implies h_{i,\xi}^k \leq 1 + \frac{C_\pi^2 \sigma^2\log{T}}{\Delta_{i,k}^2}\, \cdot
\end{align*}

Then, we apply Lemma~\ref{lem:OP-piecewise} again to bound the regret due to second pulls of any sub-optimal arm $i\notin \argmax_{i \in \arms} \mu_i^k$ in any batch $k$,
\begin{align*}
OP\pa{i,k} &\triangleq\! \sum_{t=t_{k}+1}^{t_{k+1}} \sum_{h=2}^{h_{i}^{k}} \mathds{1}\!\left(\! t\!=\!t_{i}^{k}(h) \land \HPevent \!\right)\left(\mu_{\star}(t)\!-\!\mu_{i}(t)\right) \\
&\leq C_\pi \sigma \sqrt{\pa{h_{i,\xi}^k \!-\! 1}\log{T}}\\
 &\leq \frac{C_\pi^2\sigma^2\log{T}}{\Delta_{i,k}}\cdot
 \end{align*}

We apply Lemma~\ref{lem:FP} on the set of $\Upsilon_T -1$ breakpoints and we conclude thanks to the precedent equation,
\begin{align*}
\EE{R_T(\pi)} & \leq \EE{\sum_{k=0}^{\Upsilon_T-1} \sum_{i\in\arms_k}OP\pa{i,k} }+ C_\pi \sigma \Upsilon_T K\sqrt{\log{T}} + 6KV  \\
&\leq \sum_{k=0}^{\Upsilon_T-1} \sum_{i\in\arms} \frac{C_\pi^2 \sigma^2\log{T}}{\Delta_{i,k}} +  C_\pi \sigma \Upsilon_T K \sqrt{ \log{T}} + 6KV\,.
\end{align*}
\end{proof}

\section{Rested rotting setting}
\label{app:rested-theory}
\subsubsection*{Sketch of the proof}
In Lemma~\ref{lem:regret-decompo}, we split the regret decomposition according to whether the overpulls has been done on the favorable event $\HPevent$ or not. 

In Lemma~\ref{lem:rested-B}, we show that the part of the expected regret due to pulls under $\bar{\HPevent}$ is bounded by a constant with respect to $T$ for $\alpha > 4$. Indeed, while we have only trivial bounds on the quality of the pulls on these events, we can control their probabilities thanks to Proposition~\ref{prop:prb_favorable_event}.

In Lemma~\ref{lem:rested-A}, we show that for $\hiT$ overpulls of arm $i$, we suffer no more than $\tcO\pa{\sqrt{\hiT}}$ on the favorable event. Indeed, thanks to Lemma~\ref{lem:core-full}, we know that the cost of the $h$ before last pulls is bounded by $h \cdot c(h, \delta_t) = \tcO\pa{\sqrt{h}}$.

The proof of Proposition~\ref{prop:rested-PI} follows by noticing that $\sum_{i \in \arms} \hiT \leq T$ which leads to the $\tcO\pa{\sqrt{KT}}$ rate. Indeed, thanks to the concavity of the $\sqrt{\cdot}$ and to Jensen's inequality, we find that the worst allocation is $\hiT = \frac{T}{K}$.

In Lemma~\ref{lem:UB-OP-PD}, we construct a problem-dependent bound of $\hiT$ which extends the notion of gap for rotting bandits using Lemma~\ref{lem:core-full}.

The proof of Proposition~\ref{prop:rested-PD} follows by plugging this bound in the result of Lemma~\ref{lem:rested-A}.
\subsubsection*{Full proof}
\label{ss:rested-proof}
Let $t_i^\pi(n)$ the function such that $t_i^\pi(n) = t$ when policy $\pi$ selects arm $i$ at time $t$ for the $n$-th time. We call $\mu_T^+(\pi) \triangleq \max_{i \in \arms} \mu_i\left(\NiT\right)$, \textit{i.e.} the largest available reward for $\pi$ at round T+1.  
\begin{lemma}
\label{lem:regret-decompo}
 Let $\hiT \triangleq | \NiT - \NiT^{\star}|$. For any policy $\pi$, the regret at round T is no bigger than
\begin{equation*}
R_T(\pi) \leq \sum_{i \in \overpullSet} \sum_{h=0}^{\hiT-1}\left[\xi^\alpha_{t_i^\pi(\NiT^\star + h)} \right]\left(\mu^+_T(\pi) - \mu_i(\NiT^{\star} + h ) \right) + \sum_{t=1}^T \Big[\bar{\HPevent}\Big]Lt.
\end{equation*}
We refer to the the first sum above as to $A_\pi$ and to the second sum as to $B$.
\end{lemma}
\begin{proof}
In the rested rotting setting, we can conveniently write the regret as
\begin{align}
\!\regret(\pi) &= \sum_{i\in\possibleArms}\left( \sum_{n=0}^{N_{i,T}^\star-1}  \reward_{i}(n)  - \sum_{n=0}^{N_{i,T}^\pi-1}  \mu_{i}(n) \right) \nonumber\\ 
& = \sum_{i \in \underpullSet}\sum_{n=N_{i, T}^{\pi}}^{N_{i, T}^{\star}-1} \mu_i(n) - \sum_{i \in \overpullSet} \sum_{n=N_{i, T}^{\star}}^{N_{i, T}^{\pi}-1} \mu_i(n),\label{eq:regret2}
\end{align}
where we define $\underpullSet \triangleq \left\{ \arm \in \possibleArms | N_{i, T}^{\star} > N_{i, T}^{\pi} \right\}$ and likewise $\overpullSet \triangleq \left\{ i\in \possibleArms | N_{i, T}^{\star} < N_{i, T}^{\pi}\right\}$ as the sets of arms that are respectively under-pulled and over-pulled by~$\pi$ with respect to the optimal policy.
We consider the regret at round $T$. 

We upper-bound all the rewards in the first double sum - the underpulls - by their maximum $\reward^+_T(\pi) \triangleq \max_{i\in\possibleArms} \reward_i(N_{i,T}^\pi)$. Indeed, for any overpulls $\mu_i(n_i) $ (with  $n_i \geq N_{i,T}^\pi$), we have that
\[
\mu_i(n_i) \leq \mu_i(N_{i,T}^\pi) \leq \mu^+_T(\pi)  \triangleq \max_{i\in\possibleArms} \reward_i(N_{i,T}^\pi),
\]
where the first inequality follows by the non-increasing property of $\mu_i$s; and the second by the definition of the maximum operator. Second, we notice that there are as many underpulls than overpulls (terms of the second double sum) because there both policies $ \pi$ and $\pi^\star$ pull $T$ arms. Notice that this does \emph{not} mean that for each arm $i$, the number of overpulls equals to the number of underpulls, which cannot happen anyway since an arm cannot be simultaneously underpulled and overpulled. Therefore, we keep only the second double sum,
\begin{equation}
\label{eq:regret-first-bound}
\regret(\pi) \leq \sum_{i\in \overpullSet}   \sum_{n=N_{i,T}^\star}^{N_{i,T}^\pi-1} \pa{\mu^+_T(\pi) - \mu_i(n)}.
\end{equation}

Then, we need to separate overpulls that are done under $\HPevent$ and under $\bar{\HPevent}$. We introduce $t_i^{\pi}(n)$, the round at which $\pi$ pulls arm $i$ for the $n$-th time. We now make the round at which each overpull occurs  explicit,
\begin{align*}
\regret(\pi) & \leq \sum_{i\in \overpullSet}   \sum_{h=0}^{\hiT-1} \sum_{t=1}^T \left[ t_i^{\pi}\pa{\NiT^{\star} + h} = t \right]  \pa{\mu^+_T(\pi) - \mu_i(\NiT^{\star} + h)}\\
& \leq \underbrace{\sum_{i\in \overpullSet}   \sum_{h=0}^{\hiT-1} \sum_{t=1}^T \left[ t_i^{\pi}\pa{\NiT^{\star} + h} = t \land \HPevent \right] \pa{\mu^+_T(\pi) - \mu_i(\NiT^{\star} + h)}}_{A_\pi}\\
&+ \underbrace{\sum_{i\in \overpullSet}\sum_{h=0}^{\hiT-1} \sum_{t=1}^T \left[ t_i^{\pi}\pa{\NiT^{\star} + h} = t \land \bar{\HPevent} \right]\pa{\mu^+_T(\pi) - \mu_i(\NiT^{\star} + h)}}_B.
\end{align*}
For the analysis of the pulls done under $\HPevent$ we do not need to know at which round it was done. Therefore, 
\[
A_\pi \leq \sum_{i\in \overpullSet}   \sum_{h=0}^{\hiT-1}  \left[ \xi^\alpha_{t(\Nit^\star + h)} \right] \pa{\mu^+_T(\pi) - \mu_i(\NiT^{\star} + h)}.
\]
For \FEWA or \RUCB, it is not easy to directly guarantee the low probability of overpulls (the second sum). Thus, we upper-bound the regret of each overpull at round $t$ under $\bar{\HPevent}$ by its maximum value $Lt$. While this is done to ease \myAlgorithm analysis, this is valid for any policy $\pi$. Then, noticing that we can have at most 1 overpull per round $t$, i.e., $\sum_{i\in \overpullSet}\sum_{h=0}^{\hiT-1}\left[ t_i^{\pi}\pa{\NiT^{\star} + h} = t  \right] \leq 1$, we get
\[
B \leq  \sum_{t=1}^T \Big[\bar{\HPevent}\Big] Lt\pa{\sum_{i\in \overpullSet}\sum_{h=0}^{\hiT-1}\left[ t_i^{\pi}\pa{\NiT^{\star} + h} = t  \right]} \leq  \sum_{t=1}^T \Big[\bar{\HPevent}\Big] Lt.
\]
Therefore, we conclude that
\[
\regret(\pi) \leq \underbrace{\sum_{i\in \overpullSet} \sum_{h=0}^{\hiT-1} \left[ \xi^\alpha_{t_i^\pi(\Nit^\star + h)} \right]\pa{\mu^+_T(\pi) - \mu_i(\NiT^{\star} + h)}}_{A_{\pi}} + \underbrace{\sum_{t=1}^T \Big[\bar{\HPevent}\Big] Lt}_B. 
\]
\end{proof}
\begin{lemma}
\label{lem:rested-B}
Let $\zeta(x) = \sum_n n^{-x}$. Thus, with $\delta_t = 2t^{-\alpha}$ and $\alpha > 4$, we can use Proposition~\ref{prop:prb_favorable_event} and get
\[
\EE B \triangleq \sum_{t=1}^T p\pa{\bar{\HPevent}}Lt \leq \sum_{t=1}^T KLt^{3-\alpha}\leq KL \zeta(\alpha -3)\, .
\]
In particular, for $\alpha \geq 5$, we have :
\[
\EE B \leq KL\zeta(2) \leq 2KL < 5KL\, .
\]
Thanks to Proposition~\ref{prop:prb_favorable_event_eff}, we can prove a similar bound on  $B_2 \triangleq \sum_{t=1}^T \Big[\bar{\HPtwo}\Big] Lt$, 
\[\EE{B_2} \leq 6KL \zeta(\alpha -2),\;\; \text{\ie, for $\alpha \geq 4$,} \;\; \EE{B_2} \leq 5KL.\]

\end{lemma}

\begin{lemma}
\label{lem:rested-A}
We define $\hiT^\xi \triangleq \max\left\{ h \leq \hiT | \ \xi^\alpha_{t_i^{\pi}(\Nit^\star + h)}\right\}$, the largest number of overpulls of arm $i$ pulled under $\HPevent$ at round $t = t_i^{\pi}(\Nit^\star + \hiT^\xi) \leq T$. We also define $\overpullSet_\xi \triangleq \left\{ i \in \overpullSet | \  \hiT^\xi \geq 1 \right\}.$ For policy $\pi \in \left\{ \piR, \piF\right\}$ with parameter $\alpha$, $A_{\pi}$ defined in Lemma~\ref{lem:regret-decompo} is upper-bounded by
\begin{align*}
A_{\pi} &\triangleq  \sum_{i\in \overpullSet}   \sum_{h=0}^{\hiT-1}  \left[ \xi^\alpha_{t_i^{\pi}(\NiT^\star + h) }\right] \pa{\mu^+_T(\pi) - \mu_i(\NiT^{\star} + h)} \\
& \leq \sum_{i\in \overpullSet_\xi} \pa{C_\pi \sigma \sqrt{\left(\hiT^\xi -1\right)\log\pa{T}} +C_\pi \sigma\sqrt{\log{\pa{T}}}+  L}.
\end{align*}
\end{lemma}
\begin{proof}

First, we define $\hiT^\xi \triangleq \max\left\{ h \leq \hiT | \ \xi^\alpha_{t_i^{\pi}(\Nit^\star + h)}\right\}$, the largest number of overpulls of arm $i$ pulled at round $t_i \triangleq t_i^{\pi}(\Nit^\star + \hiT^\xi) \leq T$ under $\HPevent$. Now, we upper-bound $A_{\pi}$ by including all the overpulls of arm $i$ until the $\hiT^\xi$-th overpull, even the ones under $\bar{\HPevent}$,
\begin{align*}
A_{\pi} &\triangleq  \sum_{i\in \overpullSet}   \sum_{h=0}^{\hiT-1}  \left[ \xi^\alpha_{t_i^{\pi}(\Nit^\star + h) }\right] \pa{\mu^+_T(\pi) - \mu_i(\NiT^{\star} + h)} \\
&
\leq \sum_{i\in \overpullSet_\xi}   \sum_{h=0}^{\hiT^\xi}  \pa{\mu^+_T(\pi) - \mu_i(\NiT^{\star} + h)},
\end{align*}
where $\overpullSet_\xi \triangleq \left\{ i \in \overpullSet | \  \hiT^\xi \geq 1 \right\}.$ We can therefore split the second sum of~$\hiT^\xi$ term above  into two parts. The first part corresponds to the first $\hiT^\xi-1$ (possibly zero) terms (overpulling differences) and the second part to the last  $(\hiT^\xi-1)$-th one. Recalling that at round $t_i$, arm $i$ was selected under $\xi^\alpha_{t_i}$, we apply
Lemma~\ref{lem:core-full} to bound the regret caused by previous overpulls of $i$ (possibly none),
\begin{align}
A_{\pi} &\leq  \sum_{i \in \overpullSet_\xi}   \mu^+_T(\pi) - \mu_i\pa{N_{i, T}^\star + \hiT^\xi  -1} + \frac{C_\pi}{\sqrt{2\alpha}}\pa{\hiT^\xi - 1}c\pa{\hiT^\xi-1, \delta_{t_i }} \label{eq:cor1-use1}\\
&\leq \sum_{i \in \overpullSet_\xi}   \mu^+_T(\pi) - \mu_i\pa{N_{i, T}^\star + \hiT^\xi  -1} + \frac{C_\pi}{\sqrt{2\alpha}}\pa{\hiT^\xi - 1}c\pa{\hiT^\xi-1, \delta_{T}}\\
&\leq \sum_{i \in \overpullSet_\xi}   \mu^+_T(\pi) - \mu_i\pa{N_{i, T}^\star + \hiT^\xi  -1} + C_\pi \sigma\sqrt{\pa{\hiT^\xi - 1}\log{\pa{T}}}\CommaBin
\label{eq:lasttimepdq}
\end{align}
 The second inequality is obtained because $\delta_t$ is decreasing and $c(.,\delta)$ is decreasing as well. The last inequality is the definition of confidence interval in Proposition~\ref{prop:prb_favorable_event}. 
 If  $\NiT^{\star} = 0$ and $\hiT^\xi = 1$ then
\[ \mu^+_T(\pi) - \mu_i(\NiT^{\star} + \hiT^\xi - 1) =  \mu^+_T(\pi) - \mu_i(0) \leq L,\] 
since $\mu^+_T (\pi) \leq \max_{j\in\arms}\mu_j(0)$ and  $\max_{j\in\arms}\mu_j(0) - \mu_i(0) \leq L$ by the definition of $L$ (Eq.~\ref{eq:LT}).
Otherwise, we can decompose 
\begin{align*}
\mu^+_T(\pi) - \mu_i(\NiT^{\star} + \hiT^\xi - 1) 
= &\underbrace{\mu^+_T(\pi) - \mu_i(\NiT^{\star} + \hiT^\xi-2)}_{A_1} \\&+ 
\underbrace{\mu_i(\NiT^{\star} + \hiT^\xi-2) -  \mu_i(\NiT^{\star} + \hiT^\xi - 1)}_{A_2}.
\end{align*}
For term $A_1$, since this $\hiT^\xi$-th overpull is done under $\xi^\alpha_{t_i}$, by Lemma~\ref{lem:core-full} we have that
\[
A_1 = \mu^+_T(\pi) - \bmu_i^1(\NiT^{\star} + \hiT^\xi-1) \leq 1c(1, \delta_{t_i}) \leq 2c(1,\delta_{T}) \leq C_\pi \sigma \sqrt{\log\left(T\right)} .
\] 
The second difference, 
$A_2 = \mu_i(\NiT^{\star} + \hiT^\xi-2) -  \mu_i(\NiT^{\star} + \hiT^\xi - 1 )$  
cannot exceed $L$ by definition (Eq.~\ref{eq:LT}), the maximum decay in one round is bounded.
Therefore, we further upper-bound Equation~\ref{eq:lasttimepdq} as
\begin{align}
A_{\pi} \leq \sum_{i\in \overpullSet_\xi} \pa{C_\pi \sigma \sqrt{\left(\hiT^\xi -1\right)\log\pa{T}} + C_\pi \sigma\sqrt{\log{\pa{T}}}+  L}.
\label{HPevent0}
\end{align}
\end{proof}

\label{proof1}
\restaalgoindepub*
\label{proof2} 

\begin{proof}
In Lemma~\ref{lem:regret-decompo}, we split the regret in two parts. The first one $B$ corresponds to the regret due to unfavorable events $\bar{\HPevent}$. We do not derive any guarantee of our algorithms on these events but their probabilities can be controlled thanks to parameter $\alpha$. Hence, for $\alpha > 4$, we show in Lemma~\ref{lem:rested-B} that the part of the expected regret due to unfavorable events can be bounded by a constant w.r.t. $T$. Yet, we choose $\alpha \geq 5$ to have a small constant.

The second one $A_\pi$ corresponds to the regret due to favorable events $\HPevent$ which can be bounded for our two algorithms (\FEWA and \RUCB) thanks to Lemma~\ref{lem:rested-A}. In order to get a problem-independent upper bound, we need to replace $\hiT^\xi$ by a problem-independent quantity. Starting from Lemma~\ref{lem:rested-A},

\begin{equation*}
A_{\pi} \leq \sum_{i\in \overpullSet_\xi} \pa{C_\pi \sigma \sqrt{\left(\hiT^\xi -1\right)\log\pa{T}} +C_\pi \sigma\sqrt{\log{\pa{T}}}+  L}.
\end{equation*}

Since $\overpullSet_\xi \subseteq \overpullSet$, we can upper-bound the number of terms in the above sum by  $K$.
Next, the total number of overpulls $\sum_{i\in\overpullSet} \hiT$ cannot exceed $T$. 
As square-root function is concave we can use Jensen's inequality. 
Moreover, we can deduce that the worst allocation of overpulls is the uniform one, i.e., $\hiT = T/K,$
\begin{align}
A_{\pi} &\leq K(C_\pi \sigma\sqrt{\log(T)} + L) + C_\pi \sigma\sqrt{\log(T)} \sum_{i\in \overpullSet} \sqrt{(\hiT - 1)}\nonumber\\ 
&\leq K (C_\pi \sigma\sqrt{\log(T)} + L) + C_\pi \sigma\sqrt{KT\log(T)}.
\label{eq:Abound-PI}
\end{align}

Therefore, using Lemma~\ref{lem:regret-decompo} together with Equations~\ref{eq:Abound-PI} and Lemma~\ref{lem:rested-B}, we bound the total expected regret as
\begin{equation}
\mathbb{E}[\regret(\pi)] \leq C_\pi \sigma\sqrt{\log\pa{T}}\pa{\sqrt{KT} +K} + 6KL\cdot
\end{equation}
\end{proof}

\begin{lemma}\label{lem:UB-OP-PD}
We define the smallest reward gathered by the optimal policy $\mu^-_T$ and the gap of the h first overpulls of arm $i$ with respect to that value $\Delta_{i,h}$.
\begin{align*}
&\mu^-_T\triangleq \min_{i \in \arms^\star} \mu_i \pa{\NiT^\star -1} \text{ with }\arms^\star \triangleq\left\{ i \in \arms | \NiT^\star \geq 1 \right\}, \\
&\Delta_{i,h} \triangleq \mu^-_T- \bar{\mu}_i^h\left( \Nit^\star+h \right).
\end{align*}
$\hiT^\xi$ defined in Lemma~\ref{lem:regret-decompo} is upper-bounded by a problem-dependent quantity,
\begin{equation*}
\hiT^\xi \leq   \hiT^+  \triangleq \max \left\{ h \leq T \big| \ h \leq  1 + \frac{C_\pi^2 \sigma^2 \log \pa{T}}{\Delta_{i,h-1}^2} \right\}  \leq  1 + \frac{C_\pi^2 \sigma^2 \log \pa{T}}{\Delta_{i,\hiT^+-1}^2}\cdot
\end{equation*}
\end{lemma}
\begin{proof}

We want to bound $\hiT^\xi $ with a problem dependent quantity $\hiT^+$. We remind the reader that for arm $i$ at round $T$, the $\hiT^\xi$-th overpull is pulled under $\xi^\alpha_{t_i}$ at round $t_i$. Therefore, Lemma~\ref{lem:core-full} applies and we have
\begin{align*}
\bmu_i^{\hiT^\xi  - 1} \left( \NiT^\star + \hiT^\xi   - 1 \right) &\geq \mu_T^+(\pi) - \frac{C_\pi}{\sqrt{2\alpha}} c\pa{\hiT^\xi   - 1, \delta_{t_i}}
\\& \geq \mu_T^+(\pi) - \frac{C_\pi}{\sqrt{2\alpha}} c\pa{\hiT^\xi   - 1, \delta_T}
\\& \geq \mu_T^+(\pi) - C_\pi \sigma \sqrt{\frac{\log\pa{T}}{\hiT^\xi-1}}\CommaBin
\end{align*}
Hence, we have that 
\begin{equation}
\label{eq:hiTxi-bound}
\hiT^\xi \leq 1 + \frac{C_\pi^2 \sigma^2\log\pa{T}}{\pa{\mu_T^+(\pi)- \bmu_i^{\hiT^\xi  - 1} \left( \NiT^\star + \hiT^\xi   - 1 \right) }^2 }\cdot
\end{equation}
Yet, this upper-bound still depends on random quantities such as $\mu_T^+(\pi)$ or $\hiT^\xi$ on the denominator. 
Consider the smallest value collected by the optimal policy, 
\[
\mu^-_T\triangleq \min_{i \in \arms^\star} \mu_i \pa{\NiT^\star -1}\text{ with } \arms^\star \triangleq\left\{ i \in \arms | \NiT^\star \geq 1 \right\}.
\]
It is the $T$-th largest value among the $KT$ possible ones. Since $\bmu_i^{\hiT^\xi  - 1} \left( \NiT^\star + \hiT^\xi  - 1 \right)$ is an average of overpulls value, which are all smaller or equal to $\mu^-_T$, we have
\[\mu^-_T\geq \bmu_i^{\hiT^\xi  - 1} \left( \NiT^\star + \hiT^\xi   - 1 \right).\] 
Moreover, $\mu_T^- > \mu_T^+(\pi)$ implies that the regret is 0. Indeed, in that case $\mu_T^+(\pi)$ - the pull with the largest value among the remaining values at the end of the game for $\pi$ - is \emph{strictly smaller} than $\mu_T^-$ - the $T$-th largest reward sample.  Therefore, $\pi$ has collected the $T$ largest value and has zero regret. Hence, we focus on the case $\mu^-_T\leq \mu^+_T(\pi)$, for which the regret may not be zero.  In that case, we can upper-bound the RHS term Equation~\ref{eq:hiTxi-bound} by replacing the random quantity $\mu_T^+(\pi)$ by the smaller quantity $\mu^-_T$. Hence, 
\[
\hiT^\xi \leq 1 + \frac{C_\pi^2 \sigma^2\log\pa{T}}{\pa{\mu_T^+(\pi)- \bmu_i^{\hiT^\xi  - 1} \left( \NiT^\star + \hiT^\xi   - 1 \right) }^2 }\ \leq 1 + \frac{C_\pi^2 \sigma^2\log\pa{T}}{\Delta_{i,\hiT^\xi  - 1}^2}\CommaBin
\] 
with $\Delta_{i,h} \triangleq \mu^-_T- \bar{\mu}_i^h\left( \Nit^\star+h \right)$, the difference between the lowest mean value of the arm pulled by $\pi^\star$ and the average of the $h$ first overpulls of arm~$i$. Yet, this self-bounding property of $\hiT^\xi $ is not a proper problem-dependent upper bound. We will consider the largest $h$ which satisfies this self-bounding property, 
\begin{equation*}
 \hiT^+  \triangleq \max \left\{ h \leq T \big| \ h \leq  1 + \frac{C_\pi^2 \sigma^2 \log \pa{T}}{\Delta_{i,h-1}^2} \right\}\cdot
\end{equation*}
We have that,
\begin{equation*}
\hiT^\xi \leq  \hiT^+  \leq  1 + \frac{C_\pi^2 \sigma^2 \log \pa{T}}{\Delta_{i,\hiT^+-1}^2}\cdot
\end{equation*}
\end{proof}
\restaalgoub*
\begin{proof}
We use Lemmas~\ref{lem:rested-A} and Lemma~\ref{lem:UB-OP-PD} to bound $A_\pi$ (see Lemma~\ref{lem:regret-decompo}). Indeed, since the square-root function is increasing, we can upper-bound the result in Lemma~\ref{lem:rested-A} by replacing $\hiT^\xi$ by its upper bound in Lemma~\ref{lem:UB-OP-PD}
\begin{align*}
A_{\pi} &\leq \sum_{i\in \overpullSet_\xi} \pa{C_\pi \sigma\sqrt{\log(T)} \left( 1 + \sqrt{\hiT^+ - 1}\right) + L}\\
& \leq \sum_{i\in \overpullSet_\xi} \pa{C_\pi \sigma\sqrt{\log(T)} \left( 1 + \frac{C_\pi \sigma\sqrt{\log(T)}}{\Delta_{i,\hiT^+-1}}\right) + L}. 
\end{align*}
Notice that the quantity $\overpullSet_\xi \subset \arms$. Therefore, we have 
\begin{equation}
\label{eq:Abound-PD}
A_{\pi} \leq \sum_{i\in \arms} \pa{\frac{C_\pi^2\sigma^2\log\pa{T}}{\Delta_{i,\hiT^+-1}} + C_\pi \sigma \sqrt{\log\pa{T}} +L }. 
\end{equation}
Using Lemmas~\ref{lem:regret-decompo}, \ref{lem:rested-B}, and Equation~\ref{eq:Abound-PD} we get
\begin{align*}
\EE{\regret(\pi)} &=\EE{A_{\pi}} + \EE B 
\\&
\leq \sum_{i\in \arms} \pa{\frac{C_\pi^2\sigma^2\log\pa{T}}{\Delta_{i,\hiT^+-1}} + C_\pi \sigma \sqrt{\log\pa{T}} +L } + 5KL \\
&\leq \sum_{i\in \arms} \pa{\frac{C_\pi^2\sigma^2\log\pa{T}}{\Delta_{i,\hiT^+-1}} + C_\pi \sigma \sqrt{\log\pa{T}} +6L } \cdot
\end{align*}
\end{proof}
\section{Full experiments}
\label{app:experiments}
The code of all our experiments can be found on SMPyBandits \citep{SMPyBandits}, an open-source bandits package in Python. The goal of these experiments is to perform an exhaustive benchmark of non-stationary algorithms which might be able to perform well in both rested and restless rotting setups in an agnostic way (\textit{i.e.} with the same tuning). 

\paragraph{Algorithms and parameters.} We include \RAWUCB and \FEWA \citep{seznec2019rotting}, the only algorithms which got known regret bounds in both setups. We include two versions of each algorithm: with the theoretical tuning $\alpha =4$; and with the empirical tuning $\alpha_{\mathrm{R}} = 1.4$ and $\alpha_{\mathrm{F}} = 0.06$. These two values are selected by grid-search on the rested benchmark. This benchmark has 30 different problems (for different $L$) but this is not a problem as the best tuning of $\alpha$ is the same for all the considered problem. In the restless setting,  we replace \RAWUCB and \FEWA by their efficient versions because of the longer horizon. 

We also include \EXPS\citep{auer2002nonstochastic}, an algorithm which was designed for the very general adversarial bandits problem against switching experts. As explained in the introduction, tuned \EXPS reaches the minimax optimal rate in all the presented restless setup. Yet, it is unclear if it is able to learn in the rested rotting bandits problem. We use the theoretical tuning which requires the knowledge of $T$ and $V_T$.

We also include \GLRUCB \citep{besson2019generalized}. This algorithm has two parameters : a confidence level $\delta$ for its change-point detector and an active exploration rate $\alpha$. We set $\alpha$ to zero. Indeed, the active exploration of change-detection algorithms is only useful in the increasing case (as argued by \citet{cao2019nearly}). We tune $\delta$ by its theoretical value, which requires the knowledge of $T$. Last, we only restart the history of the changed arm as our setup do not assume that all the rewards change simultaneously. For fair comparison, we only use the subgaussian version of the algorithm. Indeed, KL-UCB indexes are expensive to compute. Instead, for all the confidence bound algorithms, we rather tune $\sigma^2 = 1$ in the rested benchmark and $\sigma^2 = 0.29$ in the restless benchmark (the variance of a binomial $\mathcal{B}\pa{10,0.03)}$.  

We do not include \SWA \citep{levine2017rotting} which was shown to be less consistent than \FEWA \citep{seznec2019rotting} on rested rotting bandits. We do not include \SWUCB and \DUCB as they were shown to be unable to learn in the rested setting  \citep{levine2017rotting, seznec2019rotting}. We also do not include \CUSUMUCB \citep{liu2018change-detection} and \MUCB \citep{cao2019nearly}, as 1) they were shown to under-perform against \GLRUCB \citep{besson2019generalized}; and 2) their change-detector is harder to tune.

\subsection{Simulated benchmark for rested bandits}
\label{app:rested-sim}
\paragraph{Setup.} We use the two-arm benchmark of \cite{seznec2019rotting}. Arms are gaussians with fix variance $\sigma = 1$ and rested rotting mean. The first arm has a constant mean $0$ while the second arm abruptly switches from $+\tfrac{L}{2}$ to $-\tfrac{L}{2}$ at $t = \frac{T}{4} = 2500$. Several values of $L$ are investigated between $10^{-3}$ and $10$.
\begin{figure*}[t]
\centering
\includegraphics[clip, width= 0.325\textwidth]{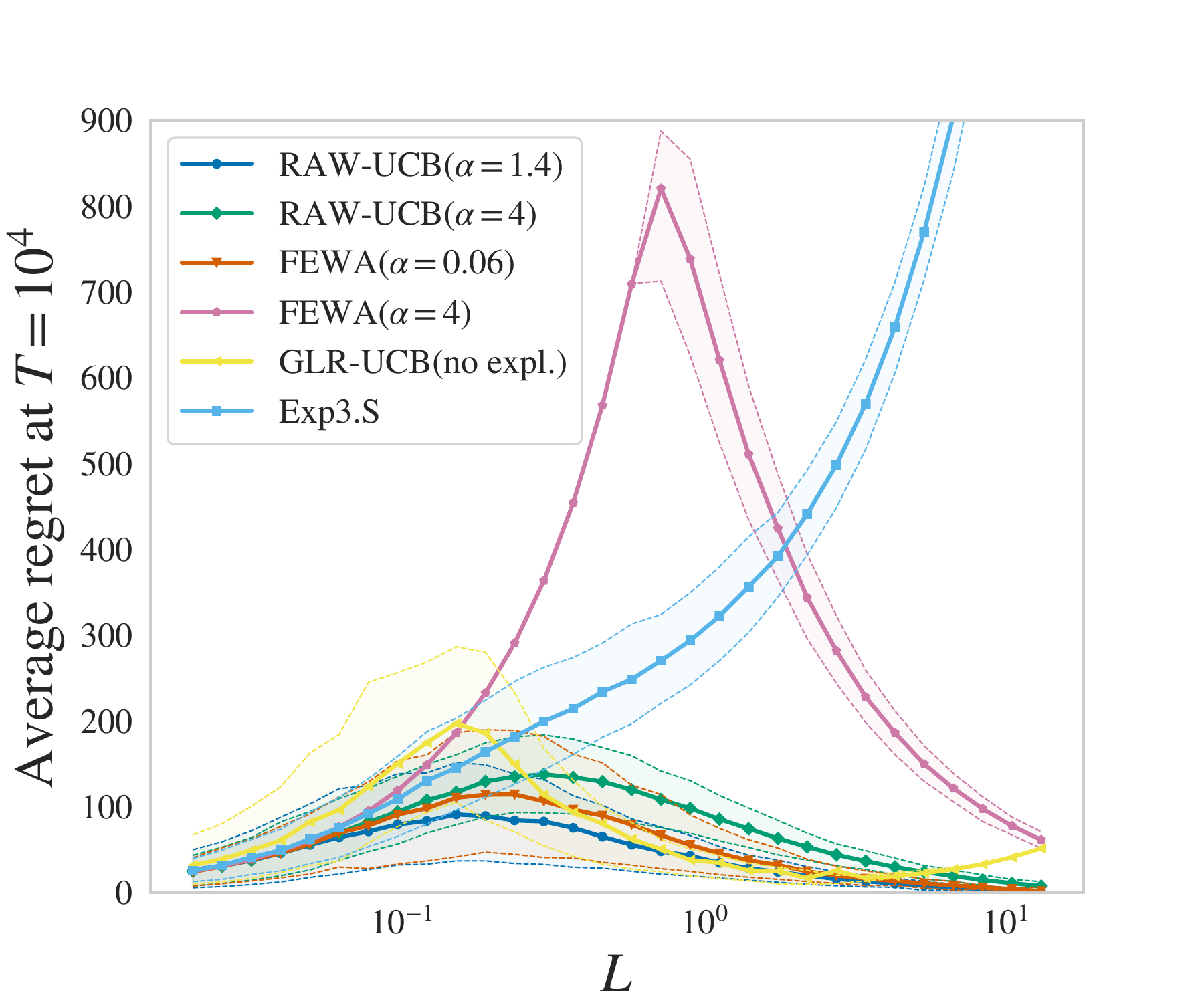}
\includegraphics[clip, width= 0.325\textwidth]{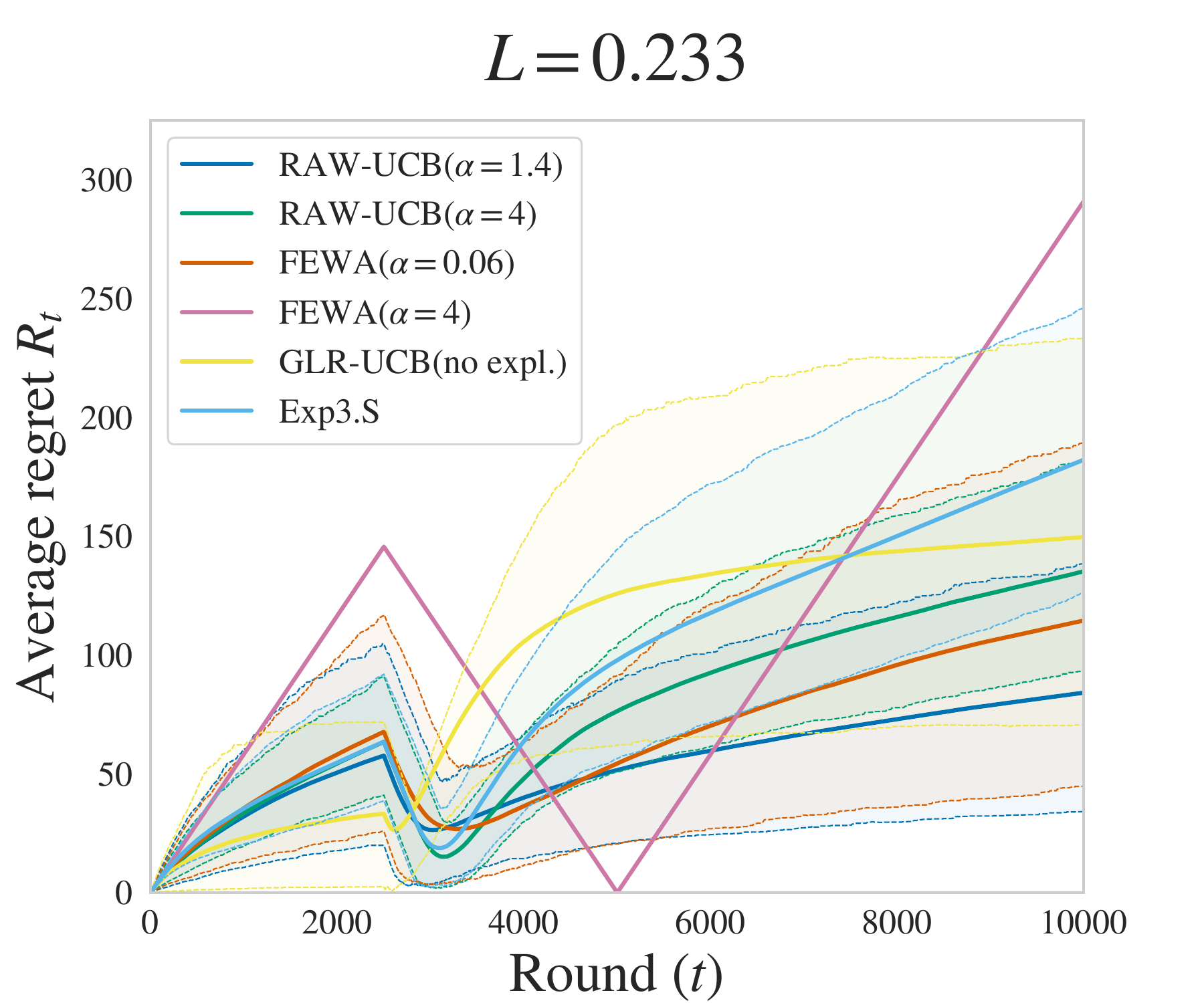}
\includegraphics[clip, width= 0.325\textwidth]{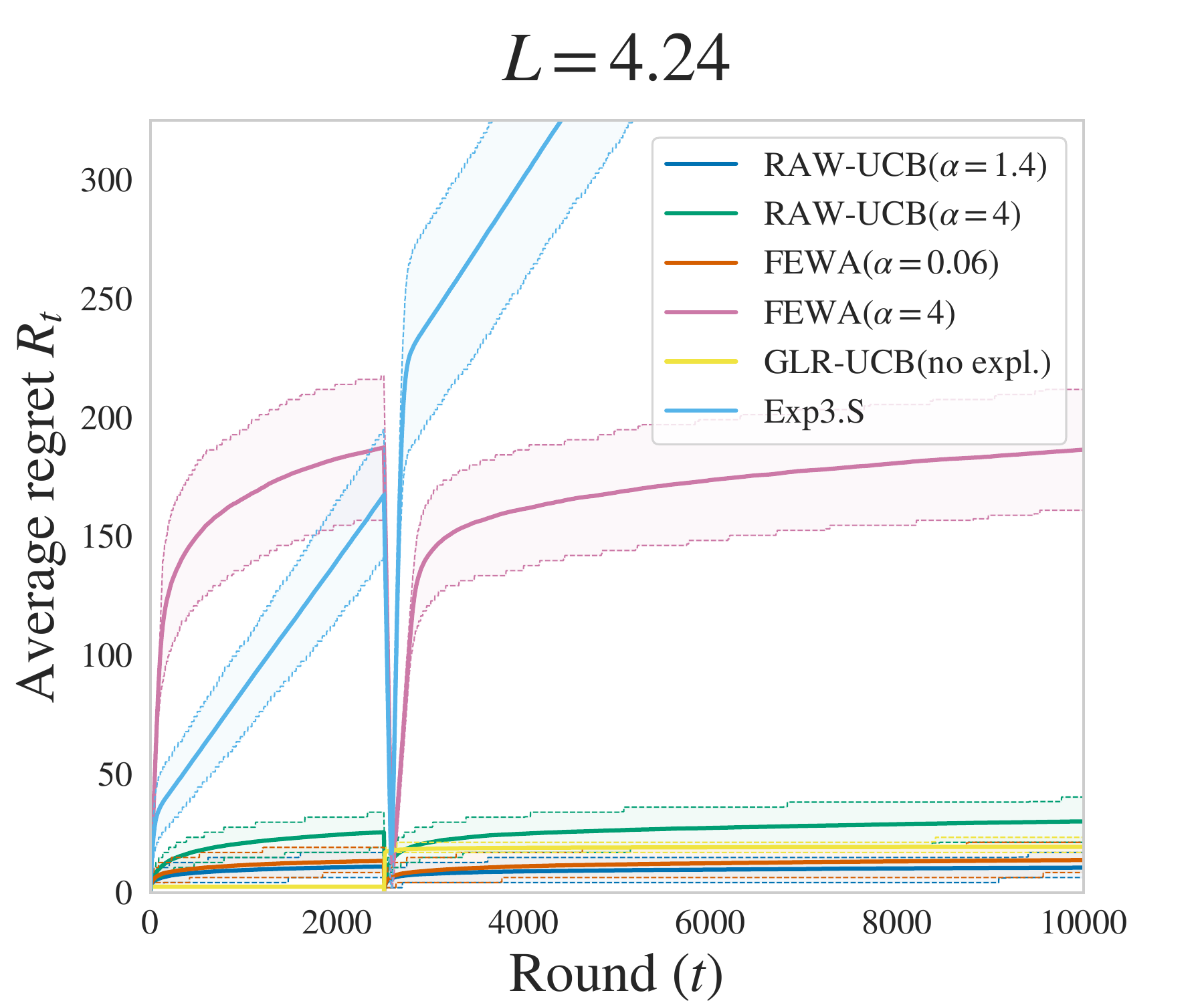}
\caption{\textbf{Left:} Regret at the end of the game for different values of $L$. \textbf{Middle, Right:} Regret across time for two values of $L$. Average over 2000 runs. We highlight the $\left[10\%, 90\%\right]$ confidence region.}
\label{Fig-rotting-bench}
\end{figure*}
\paragraph{Result : \RAWUCB vs \FEWA. } We compare \RAWUCB and \FEWA both for theoretical value $\alpha = 4$ and tuned values $\alpha_{\mathrm{R}} = 1.5$ and $\alpha_F = 0.06$ (selected by grid-search). For theoretical tuning, we see in Figure~\ref{Fig-rotting-bench} (left), that \RAWUCB outperforms \FEWA on all sizes of decays by a factor $\sim 4$ which is predicted by our theory. Indeed, there is also a factor 4 between the two problem-dependent upper-bounds. Surprisingly, for empirical tuning, the average performances of the two algorithms are much closer. We also notice that there is a larger variance in \FEWA's result compared to \RAWUCB. This is not surprising because we had to drastically reduce the confidence bounds to make \FEWA practical. It means that empirical \FEWA filters arms based only on a handful of samples. This bet leads to either very good runs or very bad runs. Last, Figure~\ref{Fig-rotting-bench} (middle, right) shows that \RAWUCB outperforms \FEWA at any time $T$, both on easy and difficult problems. 

Overall, our experiment suggests that \RAWUCB has better expected and high-probability performance than \FEWA on rested problems. Moreover, our analysis reduces the gap between theory and practice. Indeed, \RAWUCB practical confidence bounds are reduced compared to theoretical value by a factor $\sqrt{4/1.4} = 1.7$ while \FEWA's are reduced by a factor $\sqrt{4/0.06} = 8.2$. Note that the empirical tuning $\delta_T = T^{-1.4}$ is very close to asymptotic optimal tuning of \UCB :  $\delta_T = T^{-1} \log{\pa{T}}^{-2} \sim T^{-1.48}$ for $T= 10^4$. It suggests that \RAWUCB might not need to use larger confidence bands than \UCB for stationary bandits. 

\paragraph{Result : Restless algorithms.} \EXPS shows reasonable performance for small $L$ and very bad performance for large $L$. Indeed, \EXPS suffers from the fact that it pulls any arm with a probability at least $\sqrt{T}^{-1}$. When the cost of a single mistake is big (large $L$), it increases the regret. When the distance between arm is small (small $L$), all the consistent policies do $\sim T$ number of mistakes. Hence, the $\sqrt{T}^{-1}$ exploration rate is not a problem here. Combined with the observation of \citet{levine2017rotting} and \citet{seznec2019rotting}, we can conclude that passive forgetting and active random exploration leads to linear regret rate in rested rotting bandits.

This is why we cancel the random exploration for \GLRUCB. \GLRUCB use an active forgetting mechanism based on change-point detection. While it has been designed for restless bandits, it gives surprisingly good result on this rested benchmark. For $L < 10^{-1}$, \GLRUCB shows worse regret than \FEWA ($\alpha = 4$) which is equivalent to round-robin for small $L$. This is because the change is too small to be detected. Hence, the algorithm uses biased sample to compute the suboptimal arm's UCB after the break point. In this region, there is also large regret deviation. For $L\in \left[ 0.1, 4\right]$, \GLRUCB performs very well. Indeed, it can detect the change-point and then run the optimal \klucb subroutine for the remaining rounds (on which reward is stationary). For $L > 4$, \GLRUCB have worse performance than $\RAWUCB$. Indeed,  when an arm gets abruptly worse, 1) we detect the change-point; 2) we restart the arm's history which triggers additional exploratory pulls. This two steps mechanism require more pulls at the break-point than \RAWUCB. However, we can see on Figure~\ref{Fig-rotting-bench} (right) that after the first few pulls \GLRUCB pulls the sub-optimal arm at an optimal $\log{T}$ rate. This benchmark reveals that \GLRUCB with local restarts and no random exploration may be able to learn in the rested rotting setting, in particular when the decay is limited compared to the noise.

\subsection{Real world Yahoo! experiment}
\label{app:restless-exp}
\label{app:experiments-appendix}

\begin{figure*}[ht!]
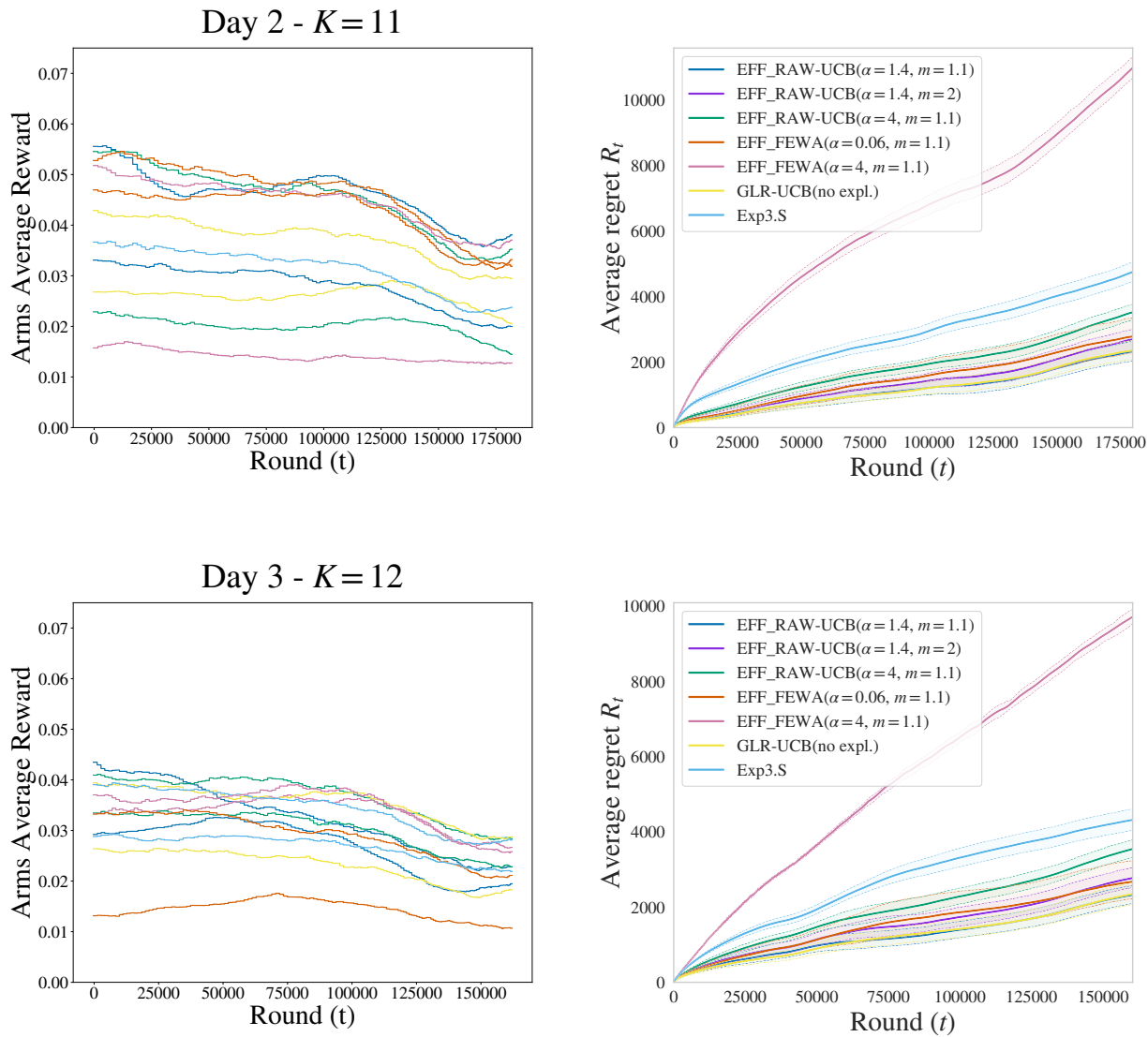

\caption{\textbf{Left:} reward functions on from the Yahoo! dataset \\ \textbf{Right:} average regret of policies over 500 runs}
\includegraphics[clip, width= 0.49\textwidth]{reward_plot/reward_plot_day2.pdf}
\includegraphics[clip, width= 0.49\textwidth]{result/DAY2.pdf}
\end{figure*}
\begin{figure*}
\includegraphics[clip, width= 0.49\textwidth]{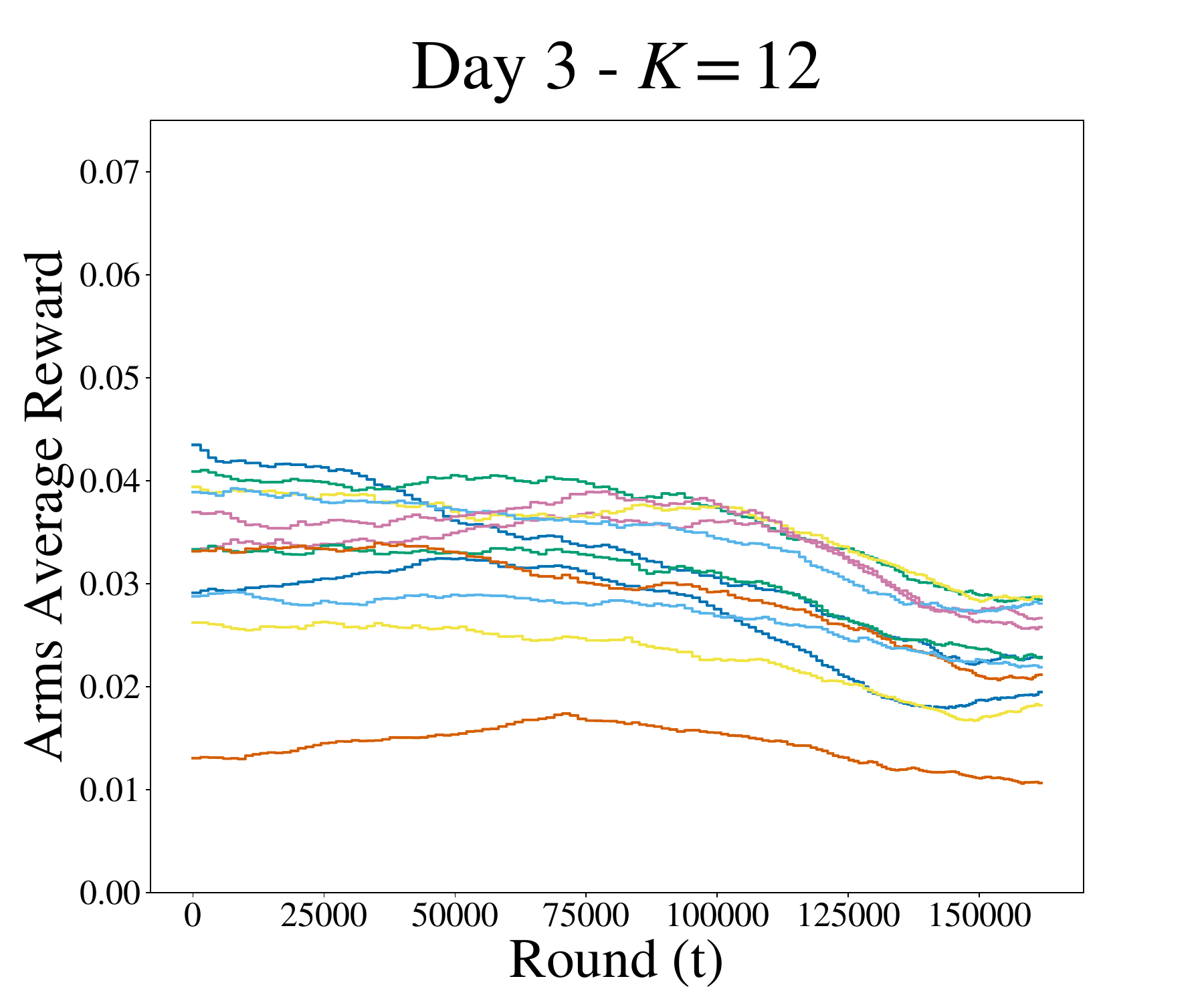}
\includegraphics[clip, width= 0.49\textwidth]{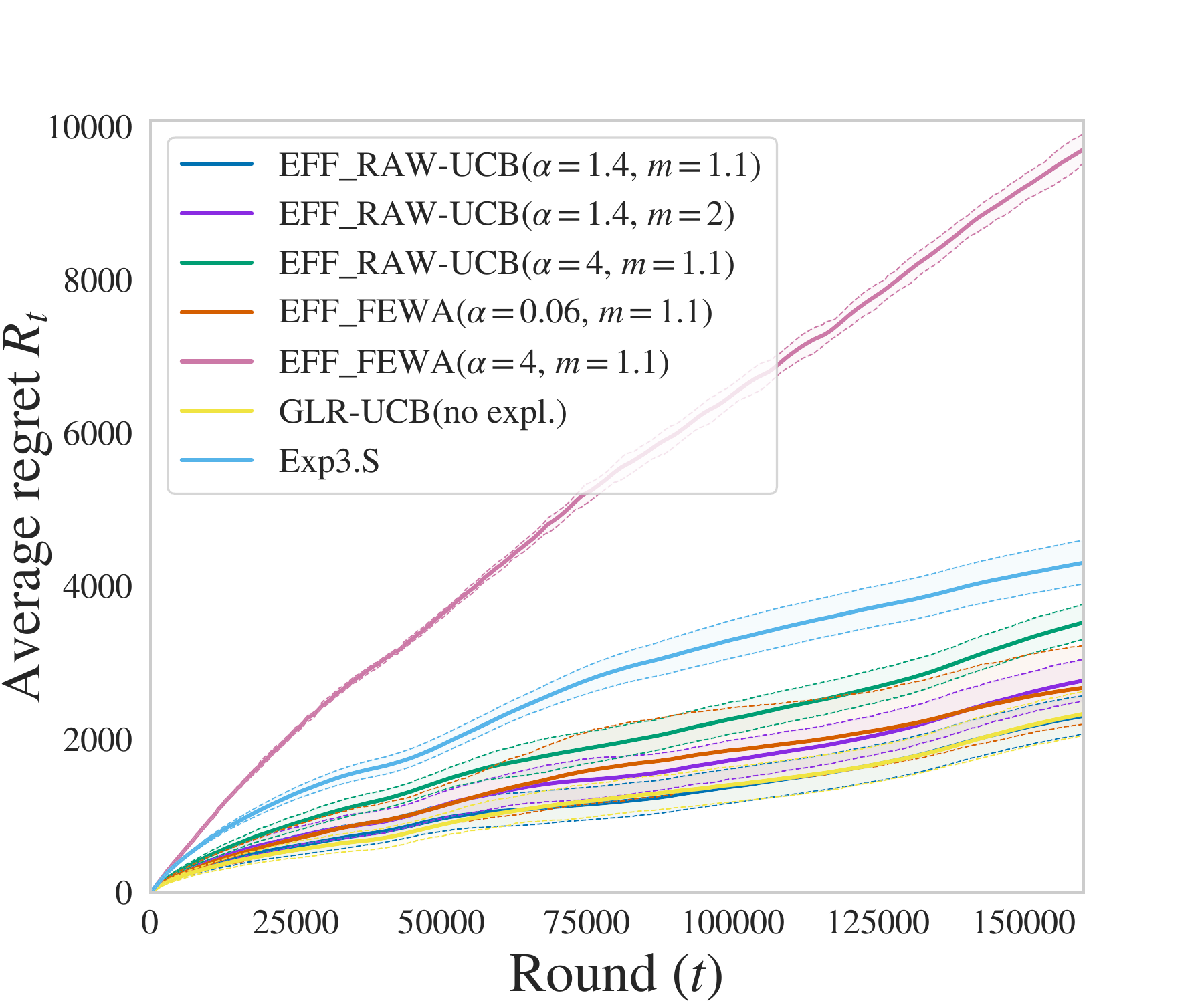}
\end{figure*}

\begin{figure*}
\includegraphics[clip, width= 0.49\textwidth]{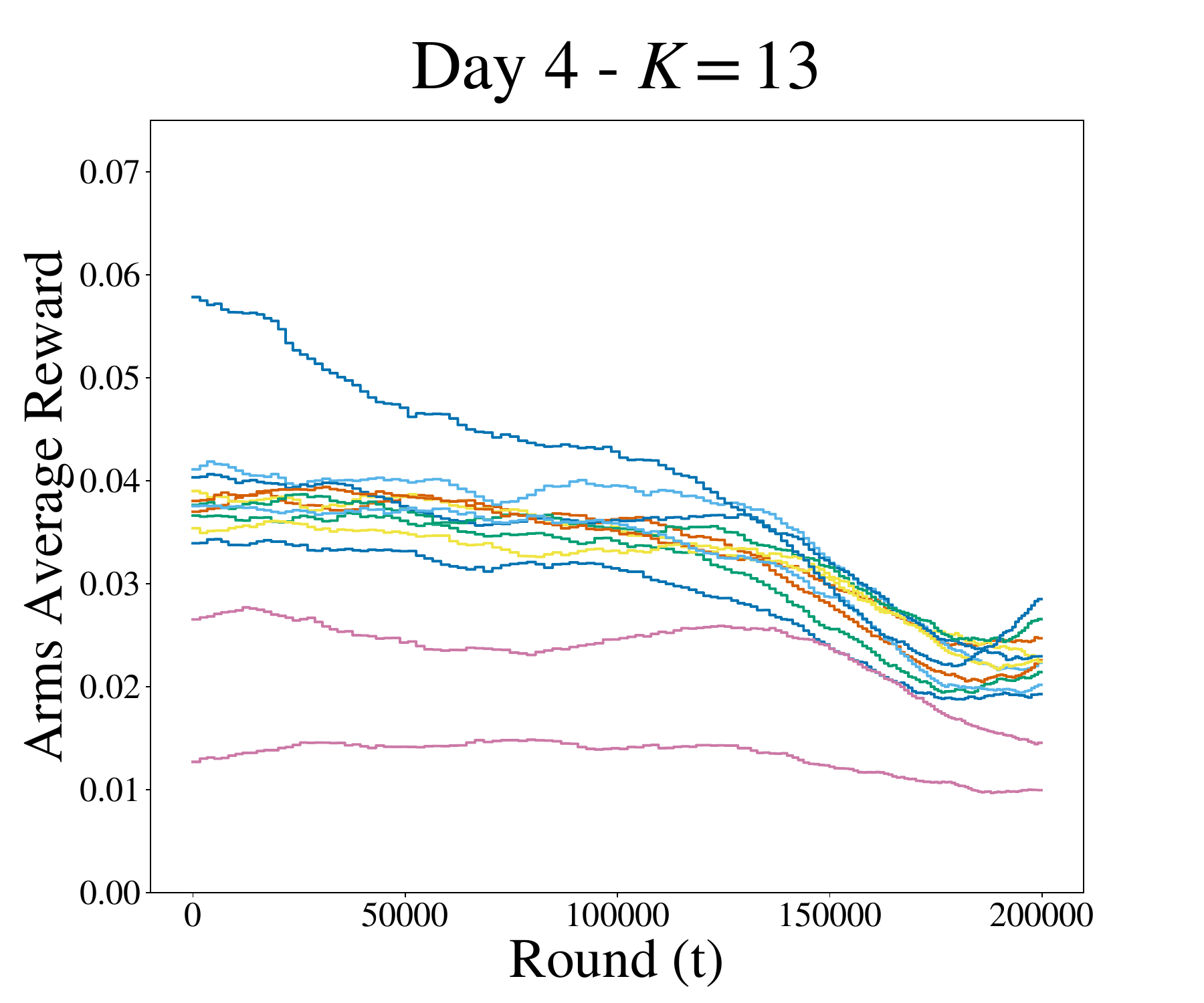}
\includegraphics[clip, width= 0.49\textwidth]{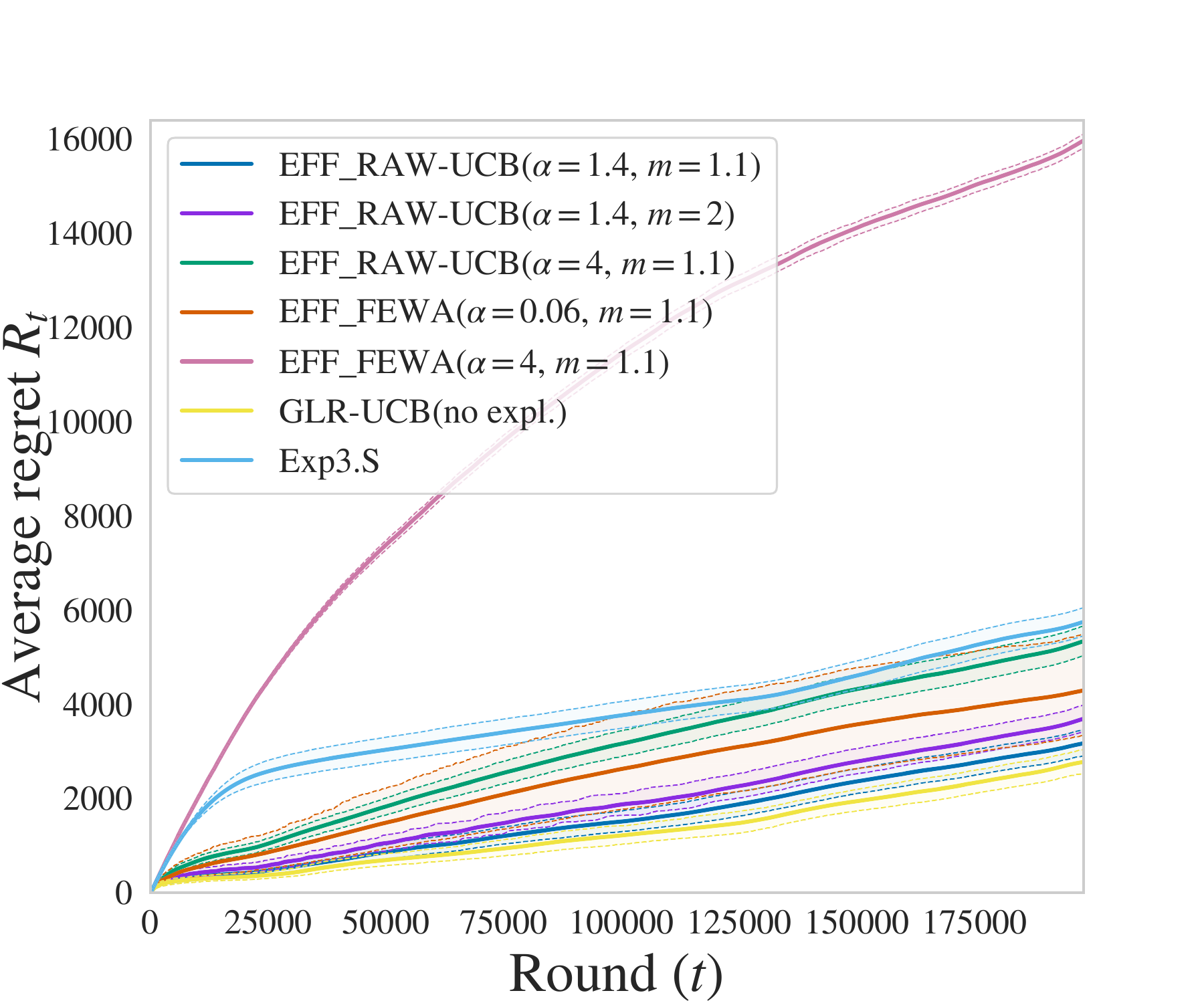}
\end{figure*}

\begin{figure*}
\includegraphics[clip, width= 0.49\textwidth]{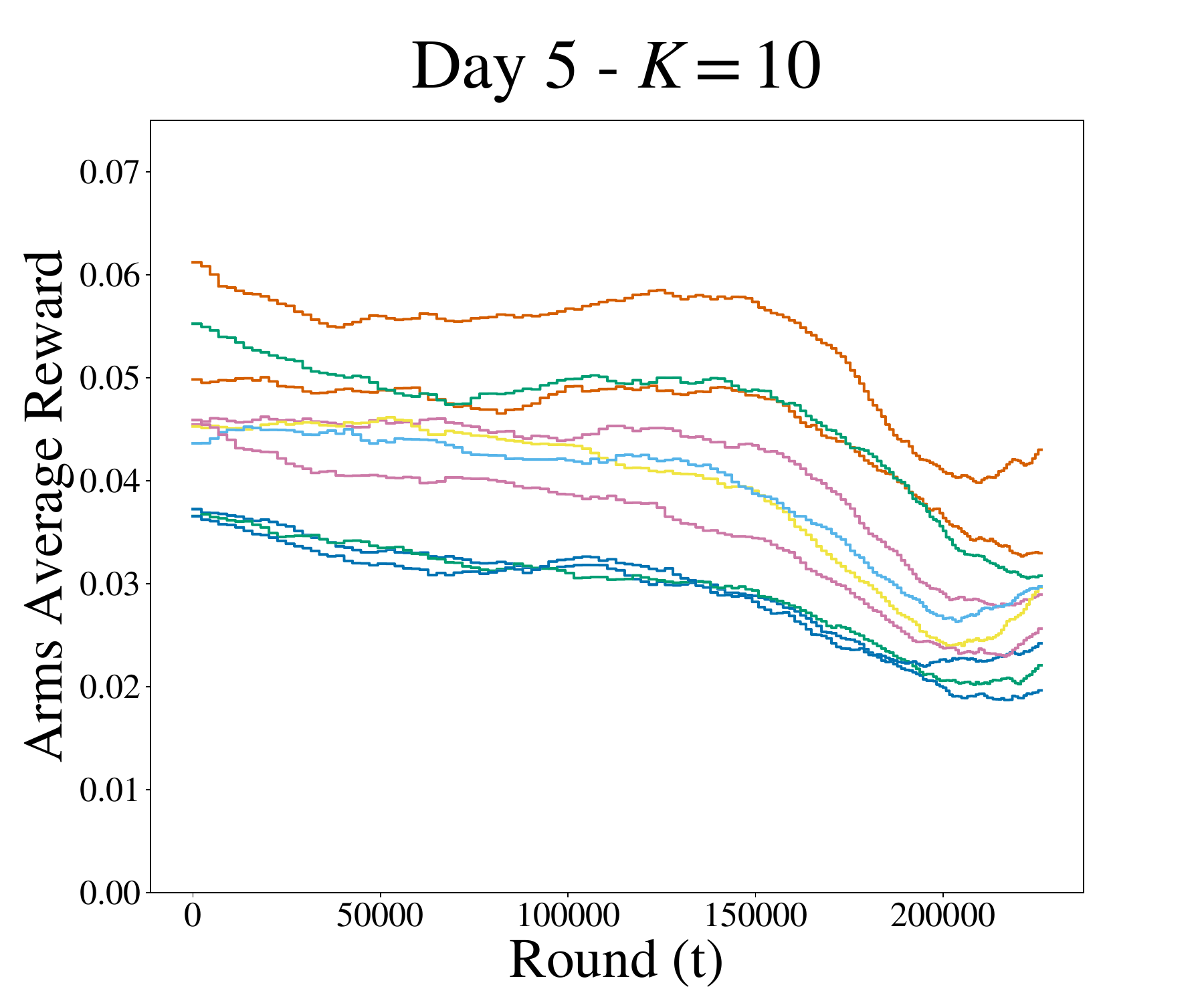}
\includegraphics[clip, width= 0.49\textwidth]{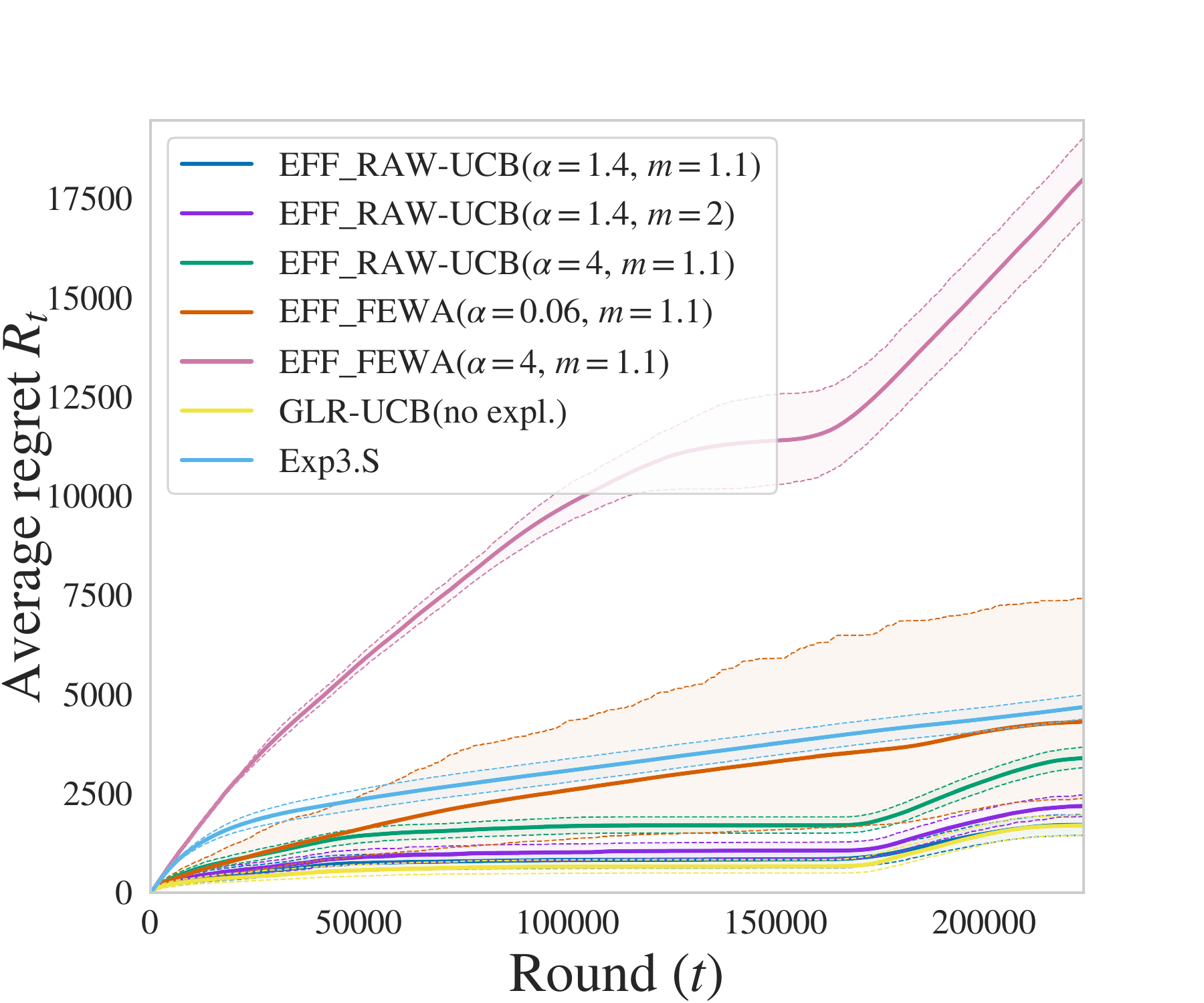}
\end{figure*}

\begin{figure*}
\includegraphics[clip, width= 0.49\textwidth]{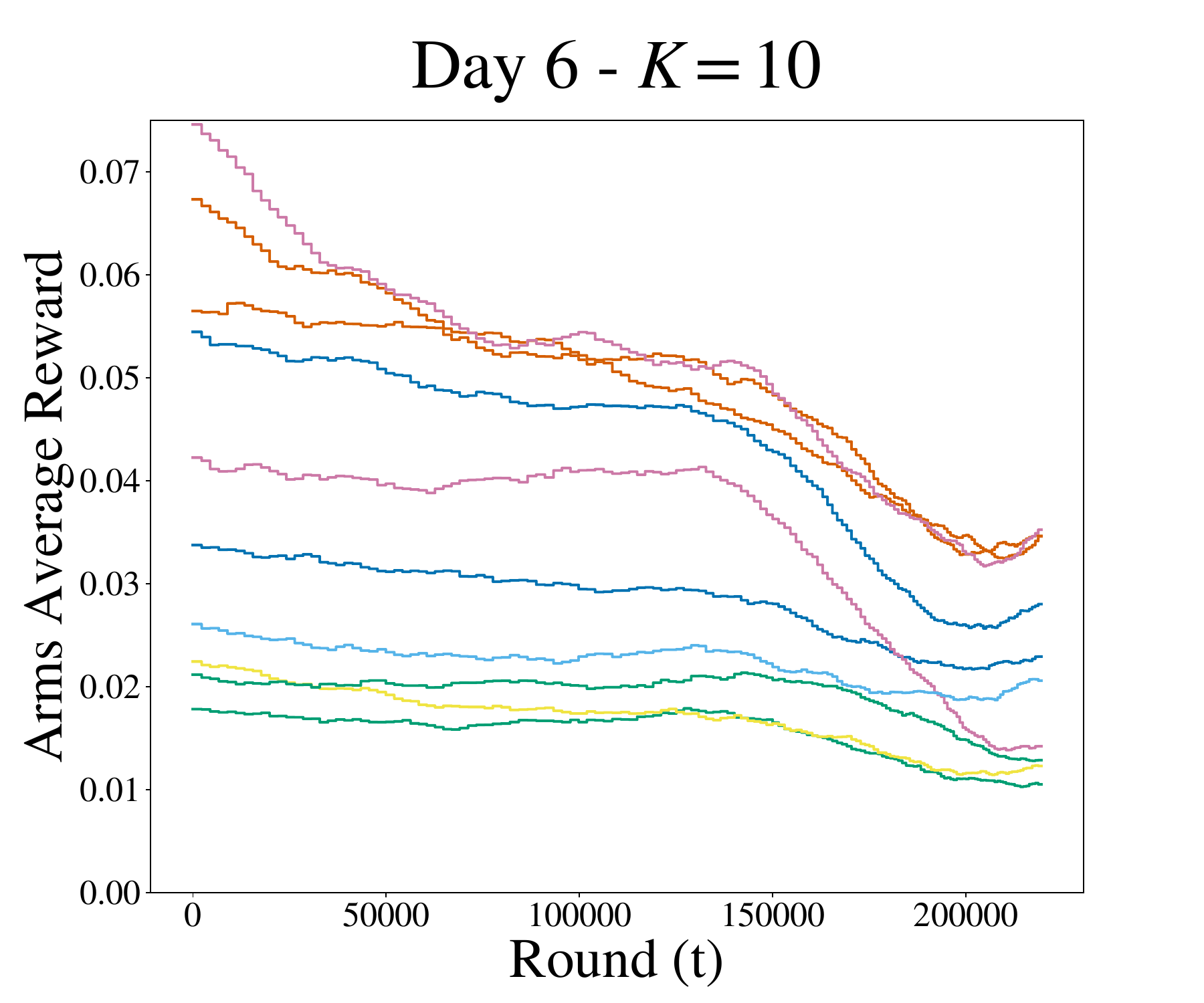}
\includegraphics[clip, width= 0.49\textwidth]{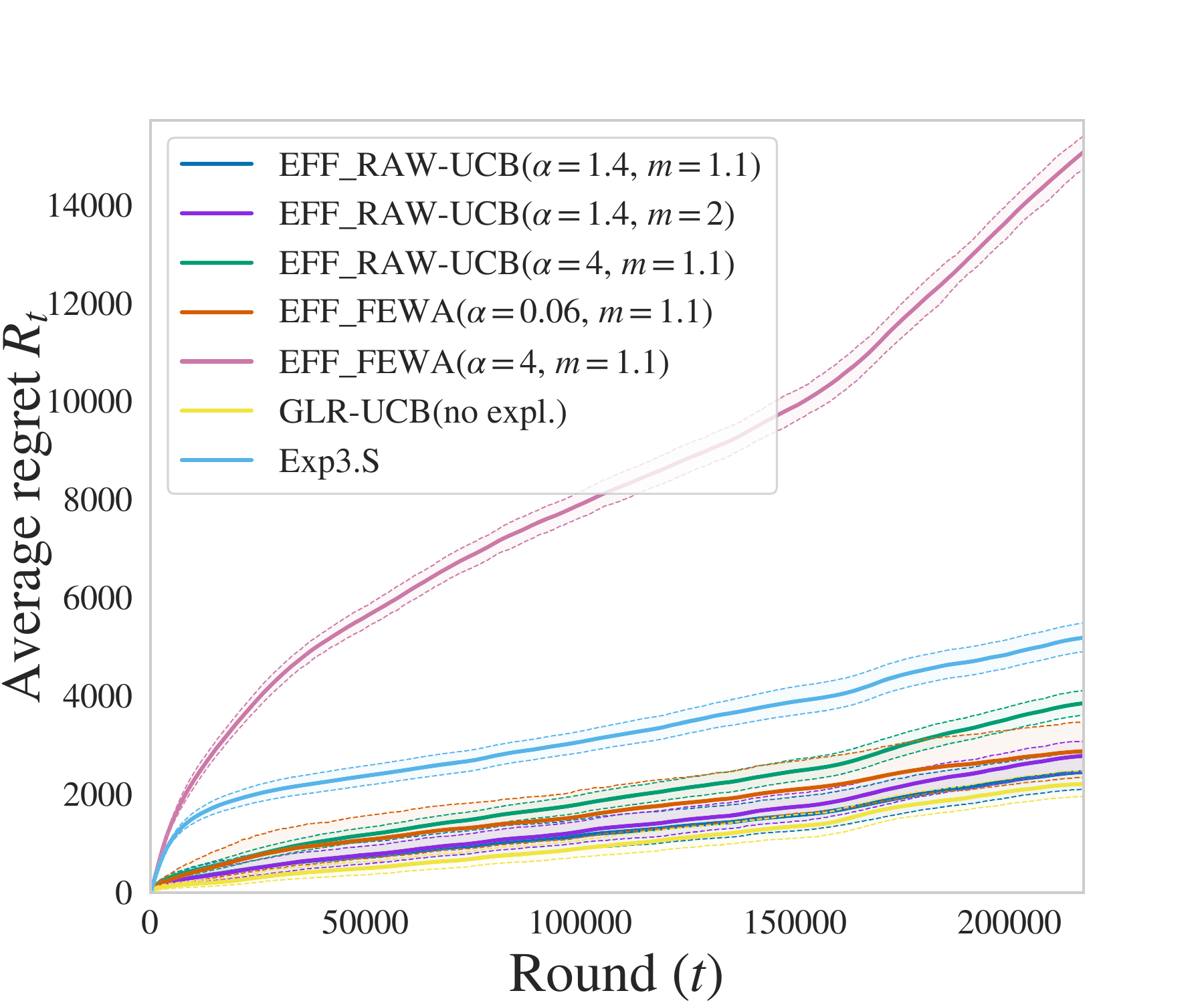}
\end{figure*}

\begin{figure*}
\includegraphics[clip, width= 0.49\textwidth]{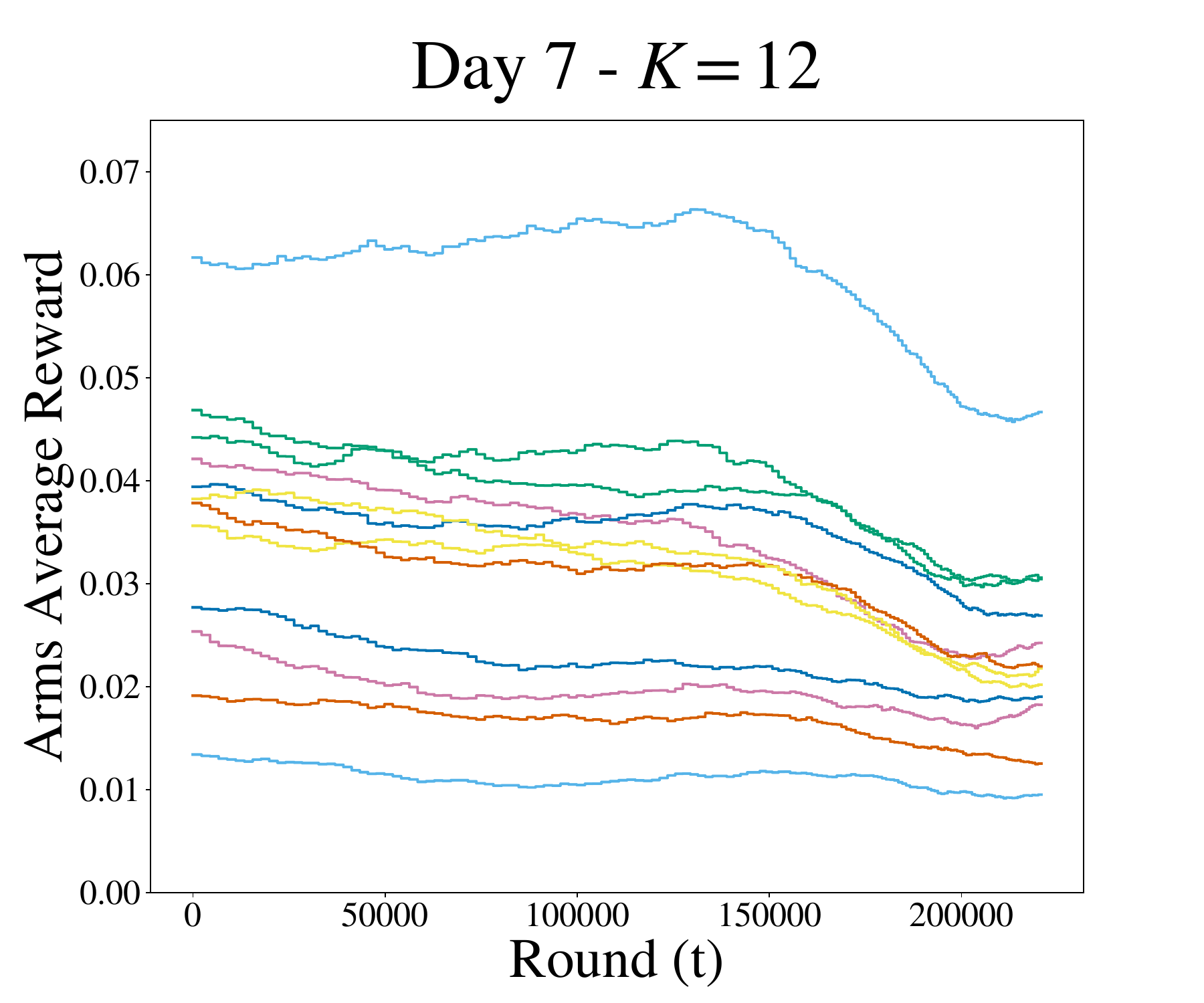}
\includegraphics[clip, width= 0.49\textwidth]{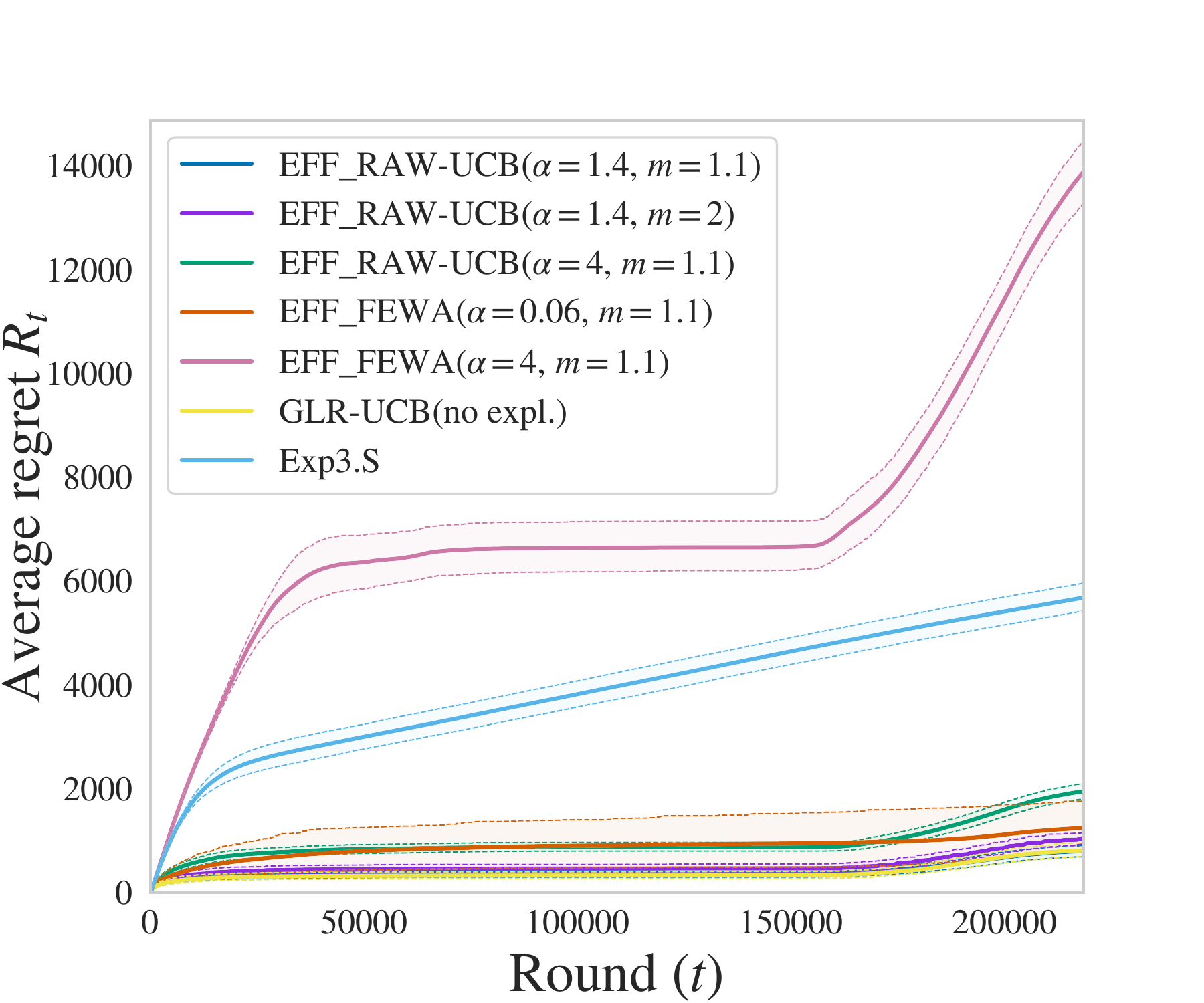}
\end{figure*}

\begin{figure*}
\includegraphics[clip, width= 0.49\textwidth]{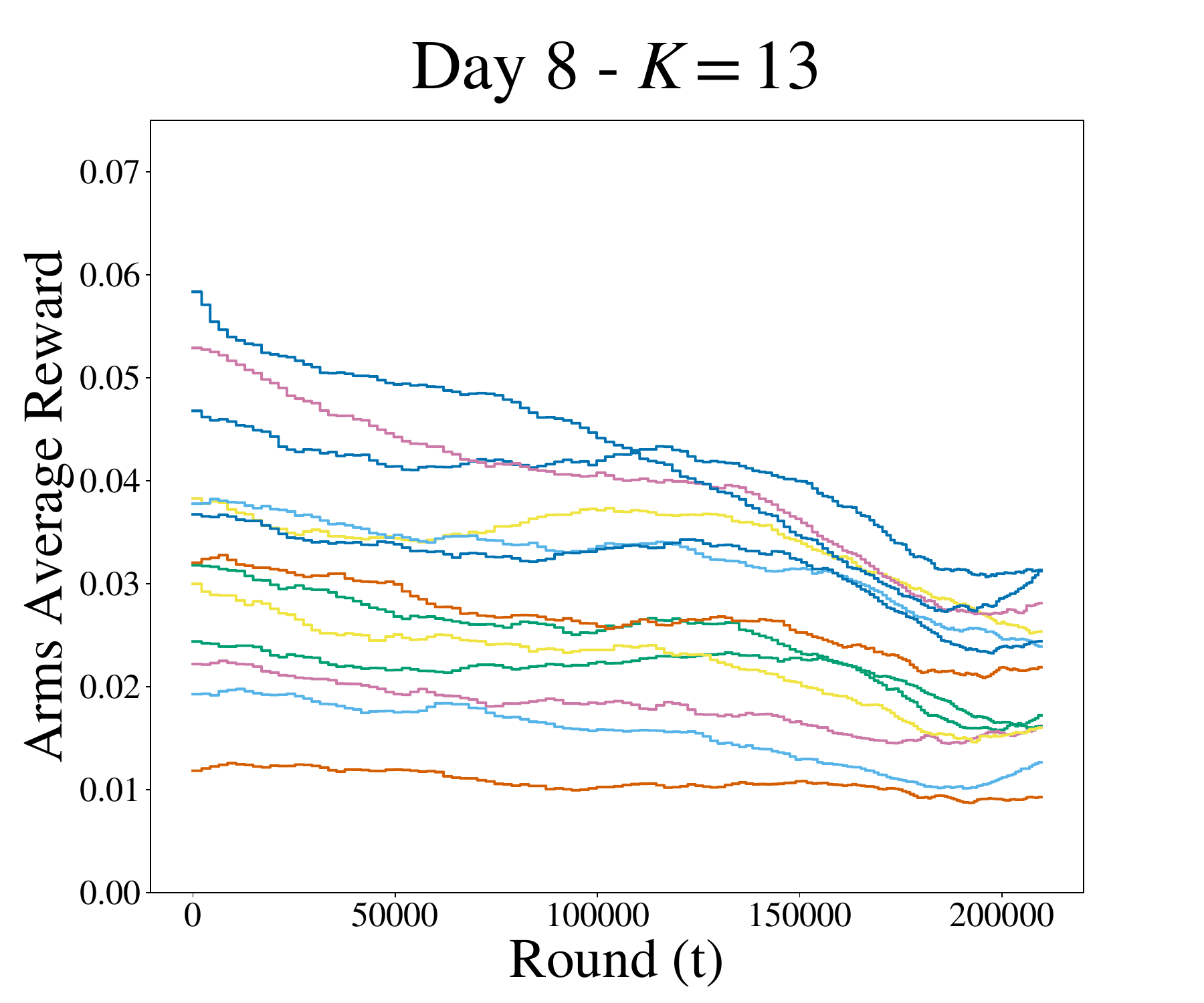}
\includegraphics[clip, width= 0.49\textwidth]{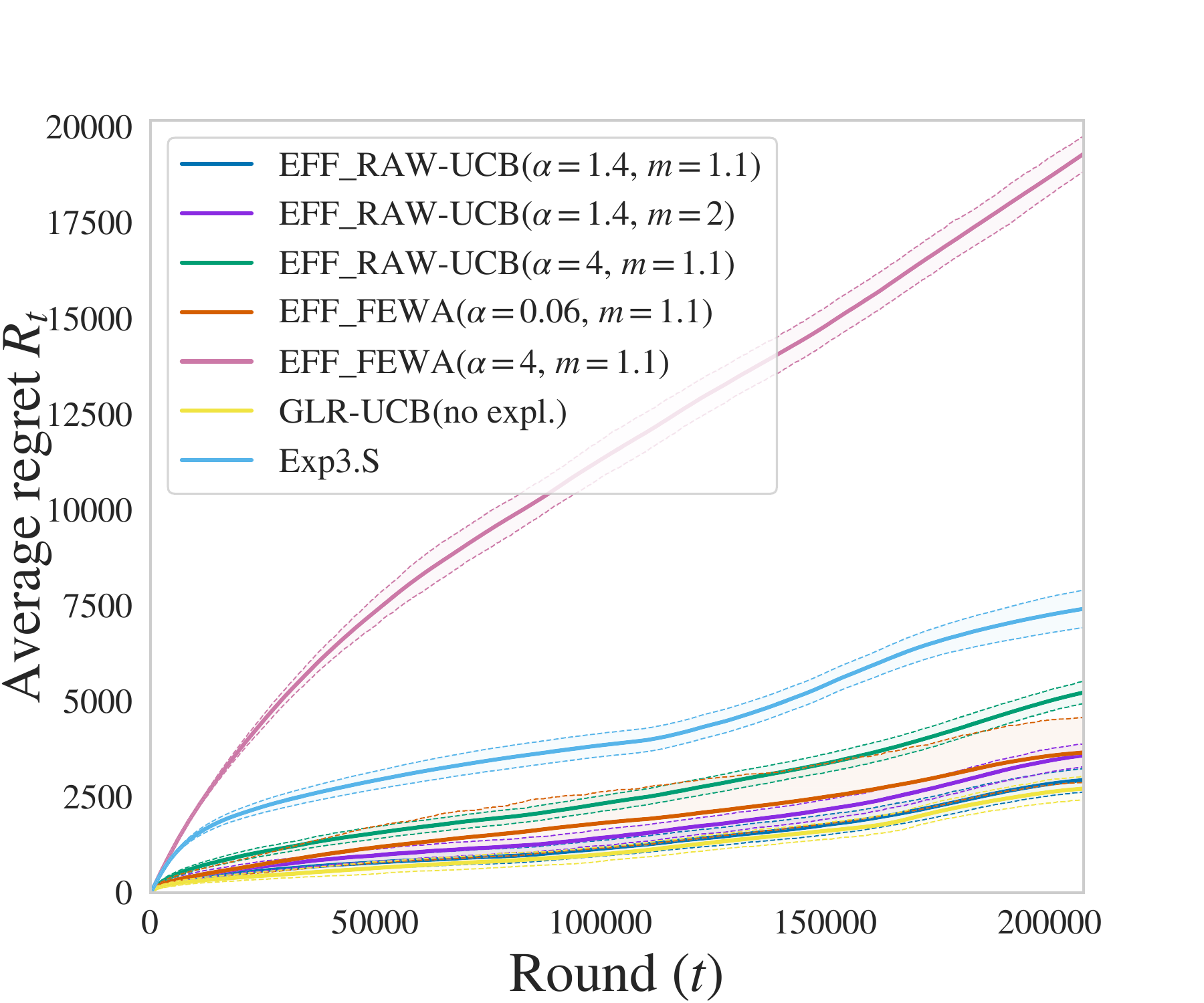}
\end{figure*}

\begin{figure*}
\includegraphics[clip, width= 0.49\textwidth]{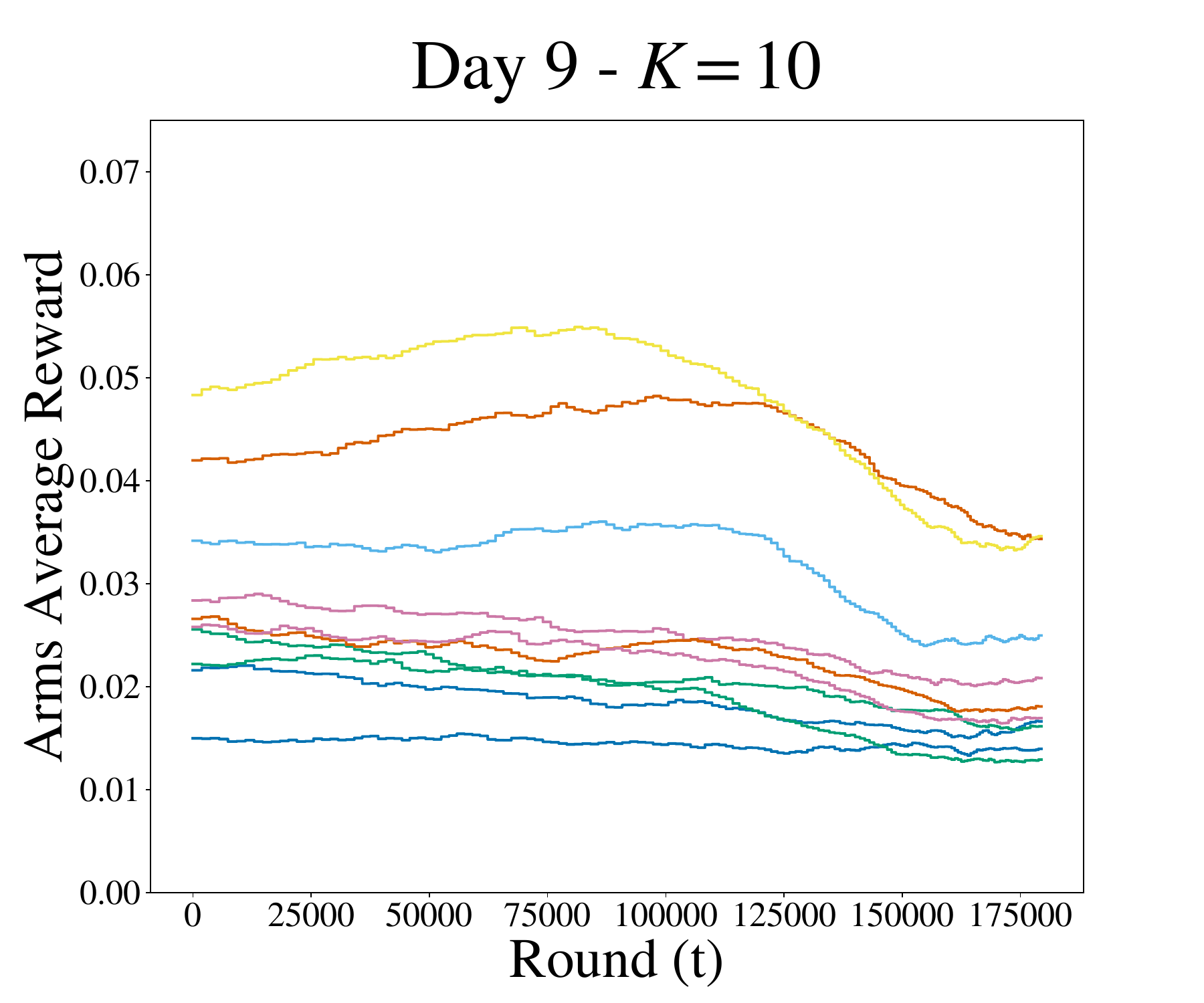}
\includegraphics[clip, width= 0.49\textwidth]{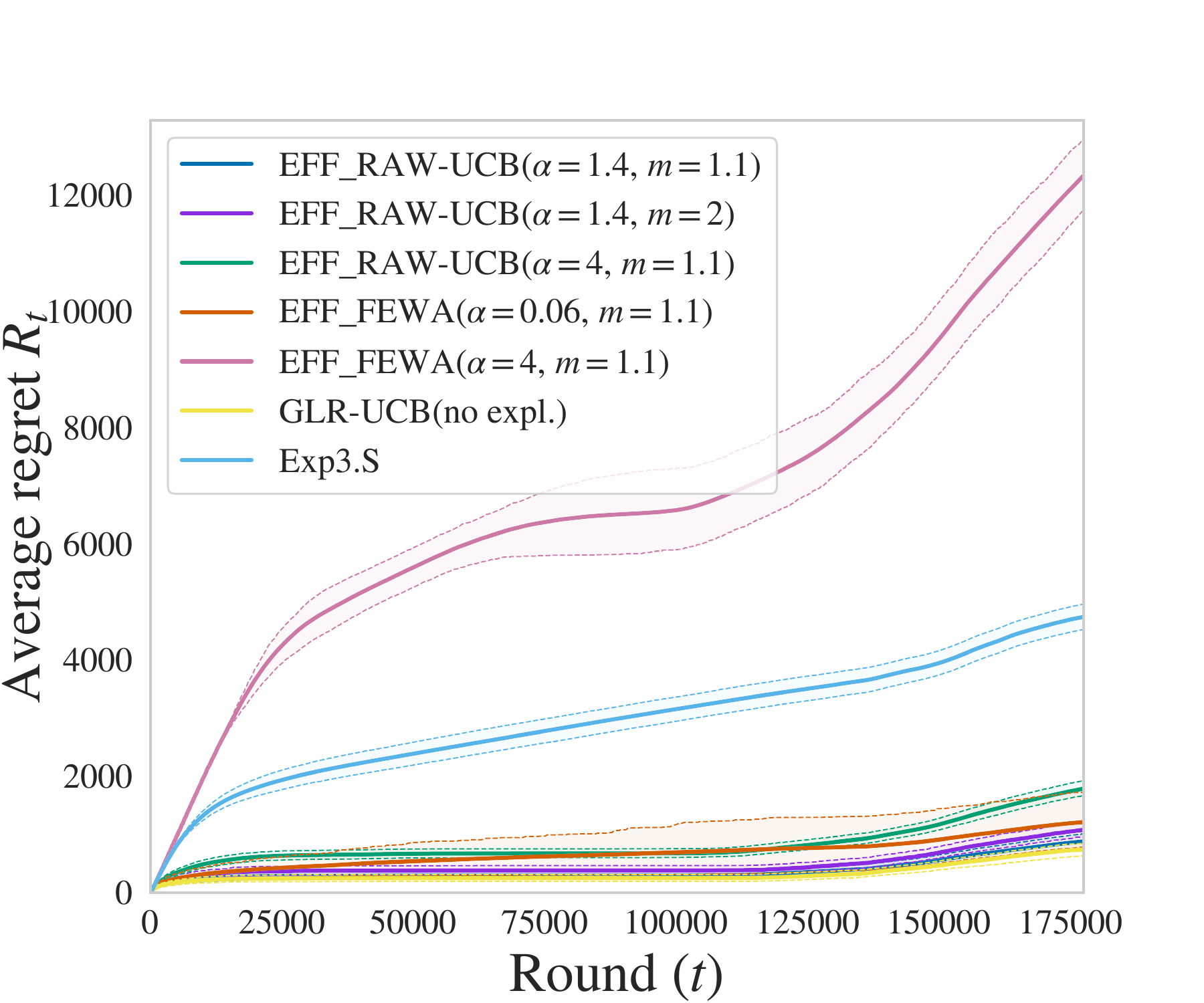}
\end{figure*}


\begin{table*}[ht!]
\def\arraystretch{1.2}
\begin{tabular}{|c|c|c|c|c|c|c|c|}
  \hline
  \textbf{Day}&\EFFRAW & \EFFRAW & \EFFFEWA  & \EFFFEWA & \EFFFEWA & \EXPS & \GLRUCB \\
   (T)&$(\alpha\!=\!1.4, m\!=\!1.1)$& $(\alpha\!=\!1.4, m\!=\!2)$ & $(\alpha\!=\!4, m\!=\!1.1)$ & $(\alpha\!=\!0.06)$ & $(\alpha\!=\!4)$ &  &\\
  \hline
\textbf{2}&67&35&65&143&337&56&560\\
\textbf{3}&66&33&65&175&308&53&613\\
\textbf{4}&90&43&90&223&391&67&683\\
\textbf{5}&86&47&88&159&473&77&2421\\
\textbf{6}&91&46&91&183&487&75&707\\
\textbf{7}&74&41&74&115&380&69&1529\\
\textbf{8}&88&44&89&193&428&71&957\\
\textbf{9}&64&34&63&116&341&55&971\\
  \hline
\end{tabular}
\vspace{0.2em}
  \caption{Average computational time in seconds for each algorithm in each experiment.}
\end{table*}

\end{document}